\documentclass[lettersize,journal]{IEEEtran}

\usepackage{array}
\usepackage[caption=false,font=normalsize,labelfont=sf,textfont=sf]{subfig}
\usepackage{stfloats}
\usepackage{verbatim}
\usepackage{graphicx}
\usepackage{cite}

\usepackage{graphicx}%
\usepackage{multirow}%
\usepackage{amsmath,amssymb,amsfonts}%
\usepackage{amsthm}%
\usepackage{mathrsfs}%
\usepackage{xcolor}%
\usepackage[table,xcdraw]{xcolor}
\usepackage{textcomp}%
\usepackage{manyfoot}%
\usepackage{booktabs}%
\usepackage{algorithm}%
\usepackage{algpseudocode}%

\makeatletter
\renewcommand\fs@ruled{%
  \def\@fs@cfont{\bfseries}%
  \let\@fs@capt\floatc@ruled
  \def\@fs@pre{{\color{black}\hrule height.8pt depth0pt}\kern2pt}%
  \def\@fs@post{\kern2pt{\color{black}\hrule}\relax}%
  \def\@fs@mid{\kern2pt{\color{black}\hrule}\kern2pt}%
  \let\@fs@iftopcapt\iftrue
}
\renewcommand\floatc@ruled[2]{%
  {\color{black}\@fs@cfont #1} {\color{black}#2}\par
}
\restylefloat{algorithm}
\makeatother

\usepackage{listings}%
\usepackage[colorlinks=true,linkcolor=blue,citecolor=blue,urlcolor=blue]{hyperref}

\theoremstyle{thmstyleone}%

\theoremstyle{thmstyletwo}%

\theoremstyle{thmstylethree}%

\usepackage{subcaption}
\usepackage{url}
\usepackage{makecell}

\newcommand{\revision}[1]{\textcolor{black}{#1}} % red/black/magenta
\newcommand{\revisionx}[1]{\textcolor{black}{#1}} % red/black/magenta
\newtheorem{prop}{Proposition}
\newcommand{\tableref}[1]{\text{Table }\ref{#1}}
\newcommand{\figref}[1]{\text{Figure }\ref{#1}}
\def\methodname{UniE2F}

\begin{document} 

\title{\methodname: A Unified Diffusion Framework for Event-to-Frame Reconstruction with Video Foundation Models}

\author{Gang~Xu, Zhiyu~Zhu, and~Junhui~Hou,~\IEEEmembership{Senior Member,~IEEE}% <-this % stops a space
\thanks{Gang Xu and Zhiyu Zhu contributed equally to this work. \textit{(Corresponding author: Junhui Hou.)}}
\thanks{Gang Xu is with the Department of Computer Science, City University of Hong Kong, Hong Kong, China, and also with the Guangdong Laboratory of Artificial Intelligence and Digital Economy (SZ), Shenzhen, China (e-mail: xugang@gml.ac.cn).}
\thanks{Zhiyu Zhu is with the Department of Computer Science, City University of Hong Kong, Hong Kong, China, and also with the Department of Computer Science, City University of Hong Kong (Dongguan), Dongguan, China (e-mail: zhiyu.zhu@cityu-dg.edu.cn).}
\thanks{Junhui Hou is with the Department of Computer Science, City University of Hong Kong, Hong Kong, China (e-mail: jh.hou@cityu.edu.hk).}
\thanks{This work was supported in part by the National Natural Science Foundation of China under Grant 62422118, and in part by the Hong Kong Research Grants Council under Grant 11218121 and Grant N\_CityU1114/25.}
}

% % The paper headers
% \markboth{Journal of \LaTeX\ Class Files,~Vol.~14, No.~8, August~2021}%
% {Shell \MakeLowercase{\textit{et al.}}: A Sample Article Using IEEEtran.cls for IEEE Journals}

% \IEEEpubid{0000--0000/00\$00.00~\copyright~2021 IEEE}
% % Remember, if you use this you must call \IEEEpubidadjcol in the second
% % column for its text to clear the IEEEpubid mark.

\maketitle

\begin{abstract}
Event cameras excel at high-speed, low-power, and high-dynamic-range scene perception. However, as they fundamentally record only relative intensity changes rather than absolute intensity, the resulting data streams suffer from a significant loss of spatial information and static texture details. In this paper, we address this limitation by leveraging the generative prior of a pre-trained video diffusion model to reconstruct high-fidelity video frames from sparse event data. Specifically, we first establish a baseline model by directly applying event data as a condition to synthesize videos. Then, based on the physical correlation between the event stream and video frames, we further introduce the event-based inter-frame residual guidance to enhance the accuracy of video frame reconstruction. Furthermore, we extend our method to video frame interpolation and prediction in a zero-shot manner by modulating the reverse diffusion sampling process, thereby creating a unified event-to-frame reconstruction framework. \revision{Experimental results on real-world and synthetic datasets demonstrate that our method outperforms previous reconstruction approaches on most quantitative metrics and achieves superior qualitative results, while achieving competitive zero-shot performance compared with dedicated methods explicitly trained for interpolation and prediction tasks.} We also refer the reviewers to the \textbf{video demo} contained in the \textit{supplementary material} for video results. The code will be publicly available at \url{https://github.com/CS-GangXu/UniE2F}. 
\end{abstract}


\section{Introduction}

\begin{figure}[t]
\centering
\includegraphics[width=\linewidth]{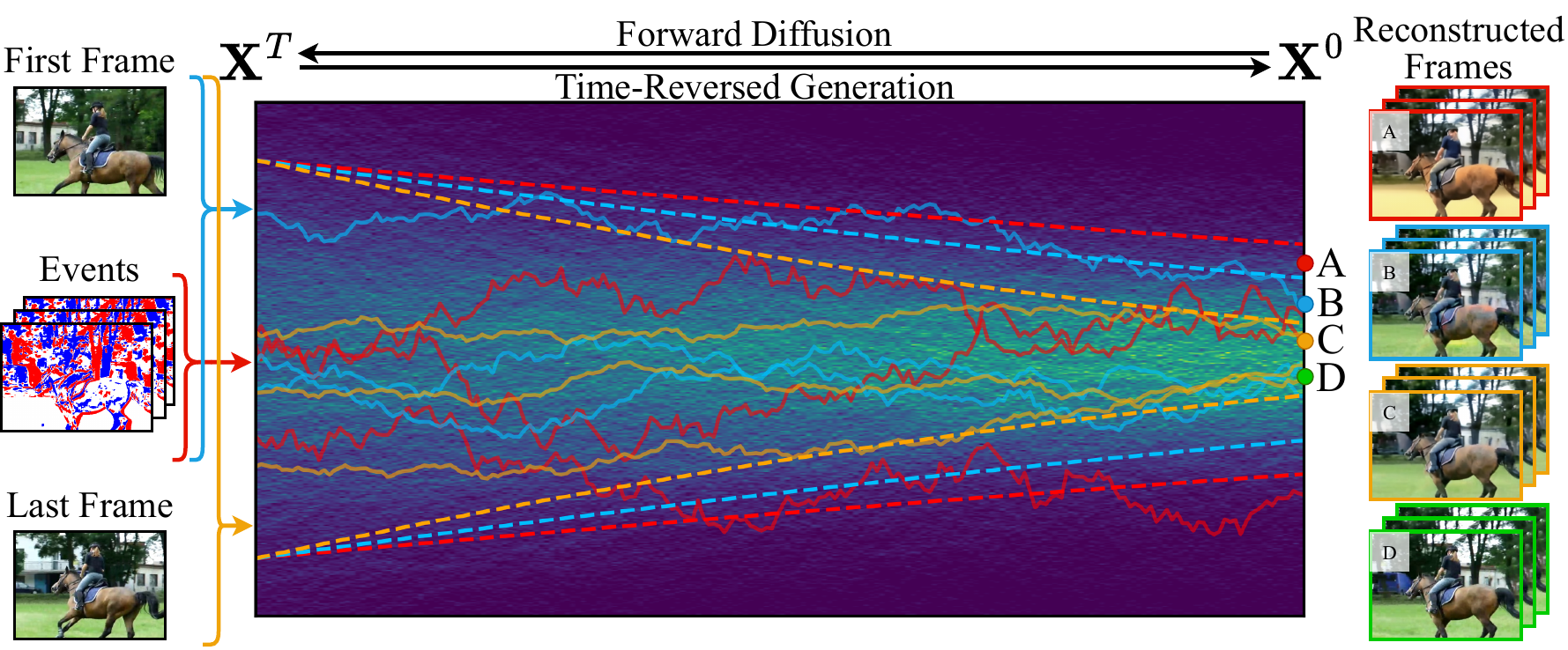}
% \caption{
% Illustration of the forward and backward diffusion processes for our \methodname~under the conditional event data. The right and left parts indicate the inputs and results of our algorithm, while in the central plot, the solid and dashed lines with the same color represent the reverse-time sampling SDE and ODE trajectories under the same setting, respectively. The proposed method can adapt to different types of event-assisted frame reconstruction tasks. (\textcolor{red}{\textbf{A}}) \textbf{event-based frame reconstruction}: with input of only event, to reconstruct RGB frame; (\textcolor{blue}{\textbf{B}}) \textbf{frame prediction}: with input of both event and the first frame to reconstruct the remaining frames; and (\textcolor{orange}{\textbf{C}}) \textbf{frame interpolation}: with input of event and the first and last frames to reconstruct the intermediate frames. (\textbf{D}) denotes the ground-truth frames corresponding to the conditional event data.
% }
\caption{
Illustration of the forward and backward diffusion processes of our \methodname~ under conditional event data. The left and right parts show the inputs and outputs, while in the central plot, solid and dashed lines of the same color denote the reverse-time SDE and ODE sampling trajectories under the same setting, respectively. The proposed method adapts to various event-assisted frame reconstruction tasks: (\textcolor{red}{\textbf{A}}) frame reconstruction, reconstructing RGB frames from events only; (\textcolor{blue}{\textbf{B}}) frame prediction, using events and the first frame to reconstruct subsequent frames; and (\textcolor{orange}{\textbf{C}}) frame interpolation, using events with the first and last frames to reconstruct intermediate frames. (\textbf{D}) denotes the ground-truth frames.
}

\label{fig:introduction}
\end{figure}

\IEEEPARstart{E}{vent} cameras, also known as dynamic vision sensors, measure high-frequency changes in pixel intensity as ``events'' and output a continuous flow encoding the time, location, and polarity of each change~\cite{brandli2014a, brandli2014real}. This asynchronous processing allows them to capture high dynamic range scenes (up to 140 dB) with exceptional temporal resolution (approx. 1 \textmu s) and low power consumption (5 mW) without motion blur~\cite{wan2022learning, zhang2023event}. Building on these advantages, event data has been widely adopted in object tracking~\cite{zhu2022learning, zhu2023cross}, semantic segmentation~\cite{chen2024segment}, and gaze estimation~\cite{zhu2025modeling}.
While these attributes make event cameras ideal for high frame rate video reconstruction~\cite{rebecq2019high, gallego2020event}, the data is inherently sparse as it only captures relative brightness changes. Consequently, this limited information content has led previous event-based video frame reconstruction approaches~\cite{rebecq2019high, stoffregen2020reducing, scheerlinck2020fast, cadena2021spade, weng2021event} to reconstruct images that differ \textit{significantly} from the richly detailed scenes observed in real-world environments.
Moreover, the current utility of event cameras extends beyond image reconstruction. Their microsecond-level resolution unlocks critical applications in temporal modeling, i.e., video frame interpolation (VFI)~\cite{dong2023video} and video frame prediction (VFP)~\cite{li2020video}. Where standard cameras struggle with high-speed motion, often suffering from blur and temporal gaps, event-based VFI~\cite{tulyakov2021time,stepan2022time} bridges these discontinuities by leveraging continuous event streams to synthesize intermediate frames. This enables applications ranging from smooth slow-motion smartphone photography to latency-critical tasks in autonomous navigation. Similarly, VFP~\cite{zhu2024video,wang2025event} utilizes event dynamics to forecast future states, compensating for hardware limitations in scientific observation. However, due to the limited model capability, previous work usually processes those tasks of reconstruction, VFI, and VFP have been treated as isolated tasks.

Recently, the diffusion model~\cite{ho2020denoising} has shown remarkable progress in image generation and image restoration~\cite{kawar2022denoising, chung2023diffusion, wang2023zeroshot, saharia2022image, saharia2022palette, gao2023implicit}, relying on a core mechanism of iterative noise addition and removal to generate realistic outcomes.
Through large-scale pre-training, the stable diffusion model~\cite{rombach2022high} has accumulated rich generative prior knowledge and a powerful capability to approximate diverse and complicated distributions.
With the aid of conditional information, such as text or image prompts, the stable diffusion model~\cite{rombach2022high} can achieve a more accurate generation process, thereby enhancing the fidelity of details and semantic consistency of the generated content. Furthermore, the stable video diffusion (SVD) model~\cite{blattmann2023stable}, benefiting from large-scale pre-training on video data, also demonstrates its capabilities to produce videos that are both realistic and visually pleasing based on text and image conditions.

In this paper, building on the powerful generative capacity of the SVD model, we propose the \textbf{Uni}fied Framework for \textbf{E}vent\textbf{2F}rame (\methodname), which leverages event data as the conditional input to guide the reconstruction process, bridging the gap between sparse event data and highly detailed real-world scenarios.
After fine-tuning the SVD model, we introduce the event-based inter-frame residual guidance, which exploits event data to effectively constrain the residual between consecutive reconstructed frames.
Specifically, during the reverse diffusion process, this mechanism first predicts the inter-frame residuals from event data.
Then, it iteratively refines the intermediate latent via the gradient descent algorithm based on the inter-frame residuals to improve the reconstruction accuracy.
In addition to reconstructing video frames purely from the event data, we further adapt the proposed method to event-based video frame interpolation~\cite{tulyakov2021time, tulyakov2022time, he2022timereplayer, wu2022video} and prediction~\cite{zhu2024video} in a \textit{zero-shot} manner.

Based on the powerful generative capacity of mapping sparse and asynchronous event data to continuous video frames, the prior information from the first and last frames for interpolation, or solely from the first frame for prediction, is utilized to guide the reverse diffusion process.
Specifically, we modulate the score function by incorporating deviation derived from the discrepancies between the estimated clean latent and the given reference latent, guiding the reverse sampling process to reconstruct intermediate or subsequent frames with enhanced temporal consistency and visual fidelity.
As shown in \figref{fig:introduction}, as more prior information is provided, our \methodname~can produce frames that more closely match the ground truth.
Through these designs, we construct a unified event-to-frame reconstruction framework that effectively handles diverse applications while reducing the need for task-specific models.
\revision{Experimental results on real-world and synthetic datasets demonstrate that our method outperforms previous reconstruction approaches on most quantitative metrics and achieves superior qualitative performance. Moreover, without any task-specific training for video frame interpolation or prediction, our method achieves competitive performance compared with dedicated methods explicitly trained for the corresponding tasks.}

In summary, the main contributions of this work are as follows.

\begin{itemize}
\item We propose a diffusion-based event-to-frame reconstruction framework that utilizes the event-based inter-frame residual guidance to align the physical correlations between those frames, thus boosting the performance.
\item \revisionx{We derive a conditional error-propagation bound showing how improved inter-frame residual consistency tightens the reconstruction-error upper bound.}
\item We extend the proposed method to video frame interpolation and prediction in a \textit{zero-shot} manner, which formulates the reverse sampling of the diffusion model to build a unified event-to-frame reconstruction framework.
\end{itemize}

The rest of this paper is organized as follows: In Section \ref{sec:related work}, we review related work on event-based video frame reconstruction, interpolation, and prediction, and advancements in diffusion models. Section \ref{sec:preliminary} introduces event representation and the diffusion model. In Section \ref{sec:proposed method}, we present the proposed method, including the event-conditioned fine-tuning of the diffusion model, inter-frame residual guidance, and adaptation to video frame interpolation and prediction. Section \ref{sec:experiment} presents detailed experimental settings, results, and comparisons with state-of-the-art methods. In Section \ref{sec:ablation study}, we conduct comprehensive ablation studies to verify the effectiveness of each component in our network. Finally, Section \ref{sec:conclusion} concludes the paper and outlines potential future research directions.

\section{Related Work}
\label{sec:related work}

\noindent
\textbf{Event-based Video Frame Reconstruction.}
Event cameras offer unique advantages, including high dynamic range, high temporal resolution, and low power consumption~\cite{brandli2014a, brandli2014real, gallego2020event, pan2019bringing, pan2020high, chen2022ecsnet}. 
Existing event-based video frame reconstruction approaches can be broadly divided into two categories: diffusion-model-based methods and non-diffusion-based methods.
Early non-diffusion methods relied on the intensity gradients provided by events and optical flow to reconstruct scene intensity~\cite{cook2011interacting, kim2008simultaneous, munda2018real, bardow2016simultaneous}.
Recently, deep learning has driven major advances: Rebecq et al.~\cite{rebecq2019high} pioneered the use of convolutional recurrent neural networks for high-quality reconstructions; 
FireNet~\cite{scheerlinck2020fast} achieves faster inference with a lightweight design; 
Stoffregen et al.~\cite{stoffregen2020reducing} improved model generalization via data augmentation. 
By utilizing the self-supervised framework~\cite{paredes2021back}, spatially-adaptive denormalization~\cite{cadena2021spade}, transformer models~\cite{weng2021event}, and context-guided hypernetworks, these approaches have shown remarkable performance.
For the diffusion-based methods, based on the coarse reconstruction result from ETNet~\cite{weng2021event}, \cite{liang2023event} and \cite{zhu2024temporal} utilized a diffusion model~\cite{ho2020denoising} to enhance the high-frequency components of objects in reconstructed images.
In addition, Chen et al.~\cite{chen2024lase} leveraged language guidance and pretrained diffusion models to achieve semantic-aware event-to-video reconstruction.
E2VIDiff \cite{liang2024e2vidiff} introduces diffusion models with event-guided sampling to reconstruct colorful and perceptually realistic videos from achromatic event streams.
CUBE \cite{zhao2024controllable} utilizes the edge information from events and combined it with textual descriptions to guide the diffusion network to synthesize videos in a zero-shot manner.
\revision{
Recently, Kim et al.~\cite{kim2025event} proposed an event-guided unified framework that jointly performs low-light video enhancement, motion deblurring, and frame interpolation by effectively integrating multi-temporal event and frame features.
Bai et al.~\cite{bai2026ae2vid} proposed AE2VID, which introduces aperture modulation to actively generate dense event signals and jointly exploits aperture-modulation-triggered and motion-triggered events for more robust video reconstruction, particularly improving static-region reconstruction and alleviating error accumulation.
}

In contrast to these methods, which are limited to reconstructing video frames from events, our \methodname, leveraging rich pre-trained generative priors and score function modulation, not only reconstructs video frames in real‑world scenes from event inputs, but also performs video frame interpolation and future‑frame prediction in a zero‑shot manner, without any additional training.

\vspace{0.5em}
\noindent
\textbf{Event-based Video Frame Interpolation.}
Due to the high temporal resolution and low latency of event data, it has been utilized to interpolate intermediate frames to achieve accurate and temporally consistent video frame reconstruction.
Time Lens \cite{tulyakov2021time} introduces a unified CNN framework that combines event-based motion estimation with both warping- and synthesis-based interpolation to robustly generate high-quality frames under non-linear motion, motion blur, and illumination changes.
Then, in Time Lens++~\cite{tulyakov21cvpr}, the motion spline estimator and multi-scale feature fusion module were proposed to achieve temporally consistent interpolation and reduce ghosting artifacts.
Liu et al.~\cite{liu2025timetracker} utilized the continuous trajectory guided motion estimation module to track the continuous motion trajectory of each divided patch.
TimeLens-XL \cite{ma2024timelens} decomposes large inter-frame displacements into multiple small-step motions and recursively estimates the optical flow at each step using events, enabling tracking of nonlinear motion.
Additionally, CBMNet~\cite{kim2023event} was proposed to leverage cross-modal asymmetric bidirectional motion fields from events and images to interpolate video frames without motion approximations.
Based on the pretrained diffusion model, Chen et al.~\cite{chen2025repurposing} trained a trainable copy on a real‑world and event‑based interpolation dataset in order to control the diffusion network to synthesize intermediate frames.
\revision{More recently, Liu et al.~\cite{liu2026epa} introduced EPA, a perceptually aligned event-based video frame interpolation framework that leverages degradation-insensitive semantic features and bidirectional event-guided alignment to improve perceptual quality and generalization, particularly when the input key frames are degraded. Moreover, Fu et al.~\cite{fu2026one} proposed a bidirectional warping framework that estimates a holistic motion representation in a single forward pass and queries bidirectional optical flows at arbitrary timestamps, enabling efficient and high-quality event-based video frame interpolation. Liu et al.~\cite{liu2026video} proposed DSER++, which reconstructs an event-based reference for feature-level alignment with key frames and introduces asymmetric selective refinement to better handle occlusions and non-rigid motion.}

Distinct from previous work, we propose a new paradigm that simply modulates the reverse diffusion sampling process to transfer reconstruction priors to interpolation and prediction in a zero-shot manner.
By fully leveraging event streams, this unified paradigm completes event-driven reconstruction, interpolation, and prediction without task-specific datasets or complex network designs.

\vspace{0.5em}
\noindent
\textbf{Diffusion Model.}
Denoising Diffusion Probabilistic Model (DDPM) ~\cite{ho2020denoising} has demonstrated remarkable success in image synthesis, with its capability evidenced by generating high-quality images from random noise.
The Latent Diffusion Model (LDM) \cite{rombach2022high} performs forward and reverse diffusion processes on the latent space, significantly enhancing efficiency in high-resolution image synthesis.
Following LDM, Rombach et al.\cite{rombach2022high} introduced DiffIR~\cite{xia2023diffir}, applying diffusion on a compact prior representation to achieve more efficient and stable image restoration.
To enable high-resolution video generation, \cite{blattmann2023align} developed video LDM by introducing a temporal dimension to the pre-trained image LDM~\cite{rombach2022high} and fine-tuning it on video data.
Subsequently, the SVD model is accomplished by pretraining video LDM~\cite{blattmann2023align} on well-curated datasets and pipelines, resulting in significant performance improvements for high-quality video generation.

\vspace{0.5em}
\noindent
\textbf{Event Stacking Approaches.}
To convert the asynchronous event stream into a tensor suitable for convolutional networks, a common practice is to partition it into temporal bins, accumulate events in each bin to form “event frames,” and stack these along the channel dimension into an event volume. Wang et al.~\cite{wang2019event} first systematically introduced time-based and event-count-based stacking strategies (SBT/SBE) for event-based HDR and high frame-rate video reconstruction, establishing the basic form of event stacking; Nam et al.~\cite{nam2022stereo} proposed mixed-density stacks in stereo event-based depth estimation to balance short-term details and long-term context; Teng et al.~\cite{teng2022nest} further modeled bidirectional event summations as a learnable Neural Event Stack for image enhancement, alleviating the sensitivity of hand-crafted stacks to noise.

\section{Preliminary}
\label{sec:preliminary}

% Specifically, consider a frame sequence \(\{I_{k}\}_{k=0}^{N_{I}-1}\) where each frame \(I_{k} \in \mathbb{R}^{3 \times H \times W}\) is associated with a timestamp \(s_{k} \in [0,T]\), and these timestamps are evenly distributed over the same duration of $T$ seconds.

\noindent
\textbf{Event Representation.}
Given the event stream $\{e_{i}\}_{i=0}^{N_{E}-1}$ containing $N_{E}$ events over a duration of $T$ seconds, each event data $\{e_{i}\}_{i=0}^{N_{E}-1}$ is encoded in the format of $e_{i}= (x_{i}, y_{i}, t_{i}, p_{i})$, where $x_{i}$, $y_{i}$, $t_{i}$ and $p_{i}$ denote the pixel positions, timestamp, and the polarity of intensity change.
Notably, $x_{i} \in \{0, ..., W-1\}$, $y_{i}  \in \{0, ...,H-1\}$, $t_{i} \in [0, T]$, $p_{i} \in \{+1, -1\}$ for all $i \in \{0, ..., N_{E}-1\}$, where $H$ and $W$ denote the height and width of the sensor array of an event camera. Due to the sparse and asynchronous characteristics of the event stream, a typical preprocessing approach in previous work~\cite{rebecq2019high, stoffregen2020reducing, weng2021event} is to accumulate the sequence of event data into a 2D image-like tensor representation, allowing its compatibility with frame-based reconstruction algorithms.
Specifically, let us $\mathbf{V}\in\mathbb{R}^{F\times3\times H\times W}$ denote the sequence of $F$ video frames to be reconstructed, and let $\{s_{f}\}_{f=0}^{F-1}$ be the corresponding timestamps, uniformly distributed over the same duration of $T$ seconds.
By setting $s_{-1}=0$, we distribute the continuous events between two adjacent frames into the group $\mathcal{G}_{f} = \{e_{i} \mid \mathbf{s}_{f-1} \leq t_{i} < \mathbf{s}_{f}\}$, where $\mathcal{G}_{f}$ denotes the $f$-th group of events that spans a duration of $\Delta T$ seconds.
Then, to leverage the rich generative prior of the pre-trained diffusion model developed on 3-channel RGB frames, we transform $F$ event groups into the sequence of 3-channel event representations also with the number of $F$, which is denoted as $\mathbf{E} \in \mathbb{R}^{F \times 3 \times H \times W}$. 
For each event representation $\mathbf{E}_{f}$ ($f \in \{0, ..., F-1\}$), we encode the event data $\mathcal{G}_{f}$ into three channels corresponding to (i) the sum of all events, (ii) the sum of positive events only, and (iii) the sum of negative events only.

% Formally, the above encoding procedure can be formulated as
% \begin{equation}
% \begin{aligned}
% \mathbf{E}_{f, c, y, x}=& \sum_{i} p_i \cdot \mathbf{1}_{\{y=y_{i}, x=x_{i}\}} \cdot (\mathbf{1}_{\{c=0\}}  + \mathbf{1}_{\{c=1, p_{i}=1\}} + \mathbf{1}_{\{c=2, p_{i}=-1\}}),
% \end{aligned}
% \end{equation}
% where $x \in \{0, ..., W-1\}$, $y \in \{0, ...,.H-1\}$ and $c \in\{0,1,2\}$ denote the pixel location and channel index, respectively.
% % 
% Here, $\mathbf{1}_{\{\cdot\}}$ is the indicator function that returns 1 if the specified condition is met and 0 otherwise.

\vspace{0.5em}
\noindent
\textbf{Diffusion Model.}
As a type of generative model, the diffusion model~\cite{ho2020denoising} first utilizes a forward diffusion process to transform data from complex high-dimensional distributions to simpler ones (typically Gaussian) and then applies the reverse diffusion process to reconstruct the original data distribution from the simplified one.
According to~\cite{song2021scorebased}, the forward stochastic differential equation (SDE) process of the latent diffusion model~\cite{rombach2022high} can be formulated as
\begin{equation}
d\mathbf{x} = f(t)\mathbf{x}dt + g(t)d\mathbf{w},
\label{forward}
\end{equation}
where $\mathbf{x} \in \mathbb{R}^{n}$ is the noised latent state, $\mathbf{w} \in \mathbb{R}^{n}$ represents a standard Wiener process,  $t$ indicates the diffusion timestamp, $f(t)$ and $g(t)$ yield the drift and diffusion coefficients, which indicate the variation of data and noise components during the diffusion process.
Therefore, the reverse process of the latent diffusion model~\cite{rombach2022high} can be formulated via the ordinary differential equation (ODE) solution~\cite{wu2024motion}:

\begin{equation}
d\mathbf{x} = \left[f(t)\mathbf{x} - \frac{1}{2}\revision{g^2(t)}\nabla_{\mathbf{x}} \log q_t(\mathbf{x})\right] dt,
\end{equation}

where $\nabla_{\mathbf{x}} \log q_t(\mathbf{x})$ is usually approximated through training a score model $\mathcal{S}_\theta(\mathbf{x}, t)$ with parameter $\theta$.
Taking into account the special case of the variance exploding (VE) diffusion process~\cite{song2021scorebased} of the SVD model, the standard reverse sampling process can be simplified as
\begin{equation}
\mathbf{x}^{t-1}  = \revision{\mathcal{R}(\mathbf{x}^t,\mathbf{u}^t)} = \mathbf{x}^t - \frac{\mathbf{x}^t - \mathbf{u}^t}{\sigma_{t}} (\sigma_{t} - \sigma_{t-1}),
\label{x_t-1}
\end{equation}
\revision{where $\mathbf{u}^t = \mathcal{X}_{\theta}(\mathbf{x}^t,t)$ represents the estimate of the clean latent which is obtained from the noised latent $\mathbf{x}^t \sim \mathcal{N}(\mathbf{x}^0, \sigma_{t}^2\mathbf{I})$ through the U-Net and Elucidated Diffusion Model (EDM)-style preconditioning at step $t$, and $\sigma_{t}^2$ is the variance of the Gaussian noise in the diffusion process.}

\begin{figure*}[t]
\centering
\includegraphics[width=\linewidth]{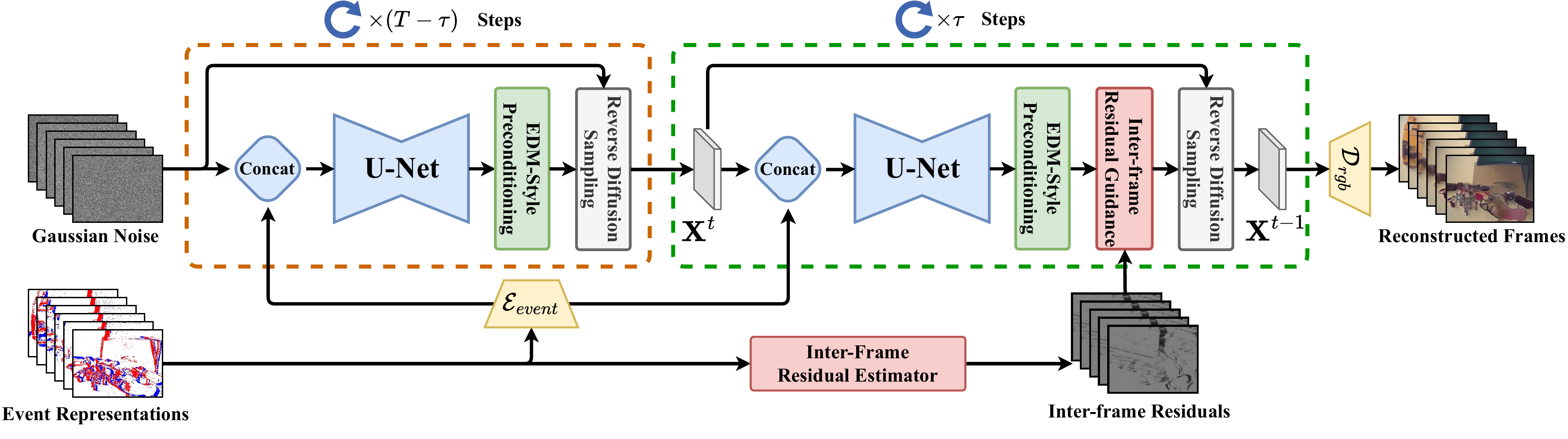}
\vspace{-20pt}
\caption{
\revision{Overview of the proposed \methodname~framework during inference.
The event representations are encoded by an event encoder and concatenated with the noisy latent $\mathbf{X}^{t}$ along the channel dimension to condition the U-Net at each step $t$.
During the first $T-\tau$ reverse diffusion steps, the model follows the standard reverse diffusion process.
During the last $\tau$ steps, the proposed event-based inter-frame residual guidance is introduced to refine the estimated clean latent before reverse sampling.
Meanwhile, the event representations are fed into the inter-frame residual estimator to predict the residuals used for guidance.}
}
\label{fig:method}
\vspace{-15pt}
\end{figure*}

\begin{figure}[t]
\centering
\includegraphics[width=0.8\linewidth]{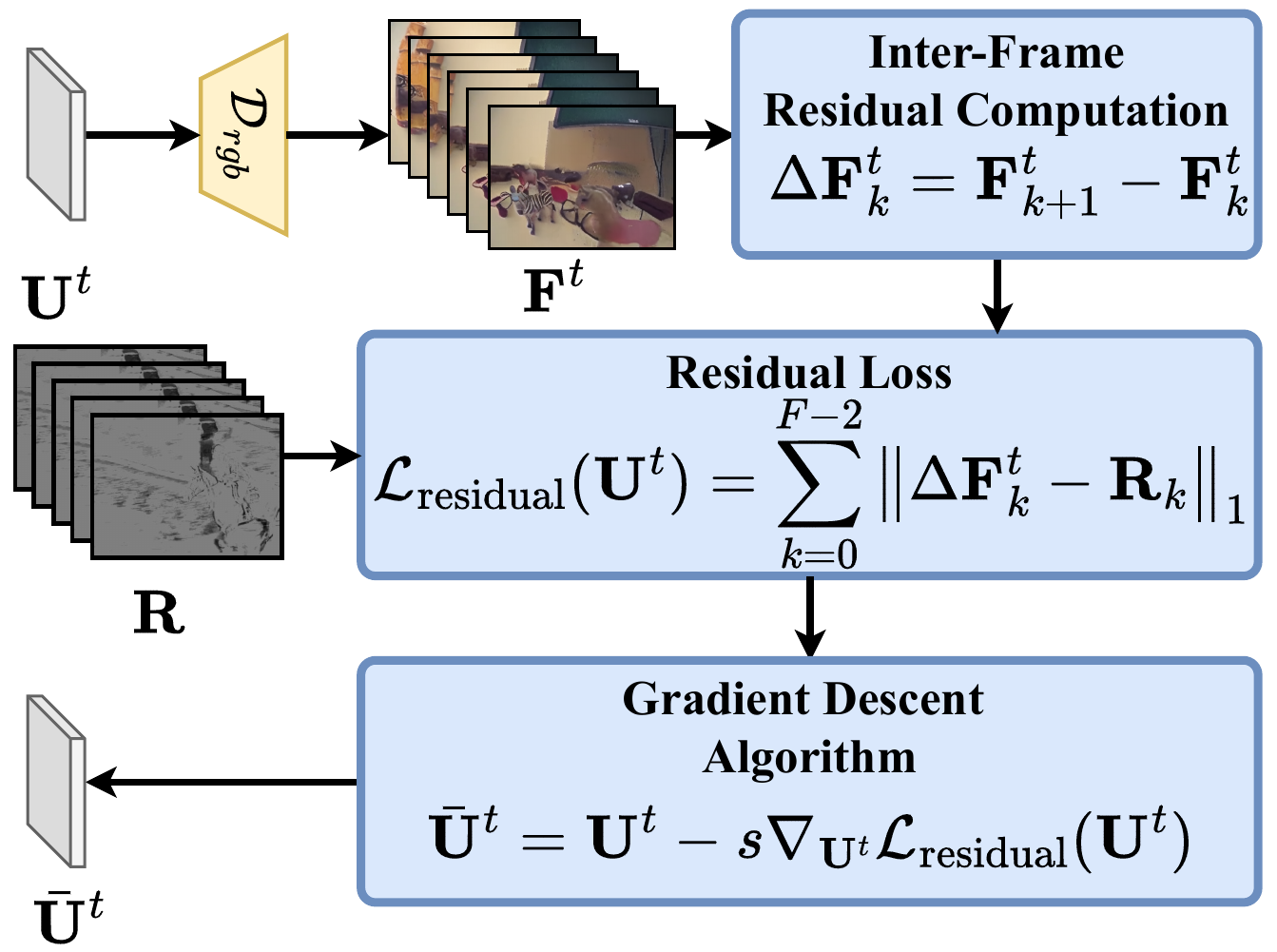}
\vspace{-5pt}
\caption{
\revision{Detailed illustration of the proposed event-based inter-frame residual guidance.
At reverse diffusion step $t$, the estimated clean latent $\mathbf{U}^{t}$ is decoded into the estimated clean frames $\mathbf{F}^{t}$, from which the inter-frame residuals $\Delta\mathbf{F}^{t}$ are computed.
The residual loss $\mathcal{L}_{\mathrm{residual}}$ is then formulated by measuring the $\ell_1$ discrepancy between $\Delta\mathbf{F}^{t}$ and the event-predicted inter-frame residuals $\mathbf{R}$.
Finally, the gradient of $\mathcal{L}_{\mathrm{residual}}$ with respect to $\mathbf{U}^{t}$ is used to obtain the refined clean latent $\bar{\mathbf{U}}^{t}$ for subsequent reverse sampling.}
}
\label{fig:ifrg}
\vspace{-15pt}
\end{figure}

\section{Proposed Method}
\label{sec:proposed method}

Leveraging the strong generative prior of the pre-trained SVD model, we first reconstruct video frames solely from the event representation $\mathbf{E}$. When the first and/or last frames (optionally)  of the ground-truth sequence $\mathbf{V}$ become available, we seamlessly adapt this event‐driven reconstruction to frame interpolation and prediction in a zero‐shot fashion, \revision{yielding a unified framework for event-to-frame reconstruction}.
Technically, we first introduce a baseline event-based video frame reconstruction model by fine-tuning the SVD model using event representations as conditional inputs, providing a foundation for the subsequent design (Sec. \ref{sec:proposed method:fine-tuning with event representation}). Then, we present the inter-frame residual guidance (Sec. \ref{sec:proposed method:inter-frame residual guidance}), \revision{as shown in~\figref{fig:method}}. Finally, we modulate the score function to enable the adaptation to video frame interpolation and prediction in a \textit{zero-shot} manner (Sec. \ref{sec:proposed method:adaptation to video frame interpolation and prediction}).

\subsection{Fine-tuning with Event Representation}
\label{sec:proposed method:fine-tuning with event representation}

To leverage the high temporal resolution property of event data to guide the pre-trained diffusion model with large generative priors for video frame reconstruction, we propose an event-conditioned fine-tuning strategy. Specifically, we first encode the event representation $\mathbf{E}$ with a dedicated encoder \revision{$\mathcal{E}_{event}$ to obtain the event feature $\mathcal{E}_{event}(\mathbf{E})$} that serves as the conditioning input.
\revision{
Subsequently, following the training scheme in previous works~\cite{song2021scorebased, karras2022elucidating, wu2024motion}, at each diffusion step $t$, we sample a noisy latent
$\mathbf{X}^t \in \mathbb{R}^{F \times 4 \times \frac{H}{8} \times \frac{W}{8}}$,
whose channel number and spatial resolution match those of the encoded clean latent and event feature, with the number of channels increased from 3 to 4 and the spatial resolution reduced from $H\times W$ by a factor of 8.
}
\revision{Conditioned on the event feature via concatenation along the channel dimension, the estimated clean latent $\mathbf{U}^t \in \mathbb{R}^{F\times 4\times \frac{H}{8}\times\frac{W}{8}}$ is obtained as $\mathbf{U}^t = \mathcal{X}_\theta\left(\mathbf{X}^t, \mathcal{E}_{event}(\mathbf{E}), t\right)$, where the operator $\mathcal{X}_\theta$ represents the composition of the raw U-Net output $\mathbf{F}^t = \mathcal{U}_\theta(\mathbf{X}^t, \mathcal{E}_{event}(\mathbf{E}), t)$ and the subsequent EDM-style preconditioning that converts U-Net's output into the clean-latent estimate via $\mathbf{U}^t = c_{\mathrm{skip}}(\sigma_t)\mathbf{X}^t + c_{\mathrm{out}}(\sigma_t)\mathbf{F}^t$, with $c_{\mathrm{skip}}(\sigma_t) = 1/(\sigma_t^2+1)$ and $c_{\mathrm{out}}(\sigma_t) = -\sigma_t/\sqrt{\sigma_t^2+1}$.}
Then the diffusion model is fine-tuned by minimizing the discrepancy between the estimated clean latent $\mathbf{U}^t$ and the ground-truth clean latent $\mathbf{X}^0 =\revision{\mathcal{E}_{rgb}(\mathbf{V})} \in \mathbb{R}^{F\times4\times\frac{H}{8}\times\frac{W}{8}}$, with the loss function $\mathcal{L}_\text{fine-tuning}$ formulated as
\begin{equation}
\begin{aligned}
\mathcal{L}_{\text{fine-tuning}} = \mathbb{E}_t \biggl\{\lambda(t) & \mathbb{E}_{\mathbf{X}^0} \mathbb{E}_{\mathbf{X}^t \mid \mathbf{X}^0} 
  \left[\left\|\mathbf{X}^0 - \mathbf{U}^t \right\|_2^2 \right] \biggr\}, \\
\lambda(t) =& \frac{\sigma_{t}^2 + \sigma_{\text{data}}^2}{\revision{(\sigma_{t} \cdot \sigma_{\text{data}})^2}},
\end{aligned}
\end{equation}
where $\lambda(t)$ serves as a weighting function.
Benefiting from this fine-tuning process, the pre-trained video diffusion model can effectively integrate event data, thereby enabling a high-fidelity synthesis of dynamic visual content.

\subsection{Inter-Frame Residual Guidance}
\label{sec:proposed method:inter-frame residual guidance}

\revision{Since each event is} triggered when pixel intensity changes reach a specific threshold, a notable correlation exists between the accumulated events and the corresponding inter-frame residual at the same pixel.
However, due to the differences in sensor sensitivities, gamma correction, or other ISP algorithms, it is intractable to manually solve the inverse process that directly calculates the frame residual.
\revision{Moreover, although event-conditioned temporal modeling enables the diffusion network to implicitly learn event-to-video dynamics, it cannot explicitly guarantee that the inter-frame variations generated during stochastic reverse diffusion remain consistent with the observed events.
Therefore, as illustrated in \figref{fig:method}, we propose to leverage event representations to predict inter-frame residuals and incorporate them into the reverse diffusion process as denoising guidance during the last $\tau$ sampling steps.}

Based on the reverse diffusion sampling process described in Eq. (\ref{x_t-1}), \revision{we refine the estimated clean latent $\mathbf{U}^{t}$ using event-derived inter-frame residuals during the last $\tau$ sampling steps}.
\revision{As illustrated in \figref{fig:method}, we initially utilize an off-the-shelf ResNet~\cite{he2016deep} model as the inter-frame residual estimator to map the event representation $\mathbf{E}$ to the inter-frame residuals $\mathbf{R}=\{\mathbf{R}_k\}_{k=0}^{F-2}$.}
\revision{The detailed procedure of the proposed inter-frame residual guidance is illustrated in \figref{fig:ifrg}. Specifically, at step $t$ ($t \leq \tau$) of the reverse diffusion process, the estimated clean latent $\mathbf{U}^{t}$---derived from $\mathbf{X}^{t}$---is fed into the autoencoder's decoder $\mathcal{D}_{rgb}$ to produce the estimated clean frame sequence $\mathbf{F}^{t}=\mathcal{D}_{rgb}(\mathbf{U}^{t})=\{\mathbf{F}^{t}_0,\ldots,\mathbf{F}^{t}_{F-1}\}$.}
\revision{The residual between each pair of consecutive frames is then calculated as $\Delta\mathbf{F}^{t}_k=\mathbf{F}^{t}_{k+1}-\mathbf{F}^{t}_k$, where $k\in\{0,\ldots,F-2\}$.} 
\revision{We then compute the inter-frame residual loss function, which is formulated as the sum of the L1 distances between the event-predicted inter-frame residuals $\mathbf{R}_k$ and the corresponding residuals $\Delta\mathbf{F}^{t}_k$ derived from the estimated clean frames in pixel space, i.e.,
\begin{equation}
\mathcal{L}_\text{residual}(\mathbf{U}^{t})
=
\sum_{k=0}^{F-2}
\left\|
\Delta\mathbf{F}^{t}_k-\mathbf{R}_k
\right\|_1.
\label{eq:residual_loss}
\end{equation}
}
This loss function is optimized via the gradient descent algorithm to update the estimated clean latent $\mathbf{U}^t$ at each sampling step $t$:
\begin{equation}
\label{Eq:reg}
\bar{\mathbf{U}}^{t}=\mathbf{U}^{t}-s \nabla_{\mathbf{U}^{t}} \mathcal{L}_\text{residual}(\mathbf{U}^{t}),
\end{equation}
\revision{where $s$ is the coefficient that controls the strength of the guidance. Moreover, the following Proposition establishes a residual-error propagation bound that explicitly relates the event-guided residual inconsistency to an upper bound on the frame reconstruction error, while accounting for both the initial-frame reconstruction error and the residual-estimation error.
\begin{prop}
Let $\mathbf{V}=\{\mathbf{V}_0,\ldots,\mathbf{V}_{F-1}\}$ and
$\mathbf{F}=\{\mathbf{F}_0,\ldots,\mathbf{F}_{F-1}\}$ denote the ground-truth and reconstructed frame sequences, respectively, with
$\Delta\mathbf{V}_k=\mathbf{V}_{k+1}-\mathbf{V}_k$ and
$\Delta\mathbf{F}_k=\mathbf{F}_{k+1}-\mathbf{F}_k$.
If the event-derived residual $\mathbf{R}_k$ satisfies
$\|\mathbf{R}_k-\Delta\mathbf{V}_k\|_1\leq\varepsilon_k$, then, for any $k\in\{0,\ldots,F-1\}$,
\begin{equation}
\left\|
\mathbf{F}_k-\mathbf{V}_k
\right\|_1
\leq
\left\|
\mathbf{F}_0-\mathbf{V}_0
\right\|_1
+
\mathcal{L}_{\mathrm{residual}}
+
\sum_{i=0}^{k-1}\varepsilon_i.
\label{eq:prop_error_bound}
\end{equation}
\end{prop}
}

\revision{The derived bound explicitly separates the frame reconstruction error into the initial-frame reconstruction error, the event-guided residual inconsistency, and the residual-estimation error.}
\revision{Therefore, for fixed initial-frame and residual-estimation errors, reducing $\mathcal{L}_{\mathrm{residual}}$ tightens the derived upper bound on the frame reconstruction error, although this does not imply that each gradient update necessarily decreases the actual reconstruction error.}

\begin{proof}
See Appendix \textcolor{blue}{A}.
\end{proof}
\revision{By substituting the original $\mathbf{U}^t$ in the reverse diffusion} with the refined clean latent expectation term $\bar{\mathbf{U}}^{t}$, the sampling process is guided with: 
\begin{equation}
\mathbf{X}^{t-1} =\revision{\mathcal{R}(\mathbf{X}^t,\bar{\mathbf{U}}^t)} =\mathbf{X}^{t}-\frac{\mathbf{X}^{t}-\bar{\mathbf{U}}^{t}}{\sigma_{t}} \left (\sigma_{t} - \sigma_{t-1} \right ).
\end{equation}

\subsection{Adaptation to Video Frame Interpolation and Prediction}
\label{sec:proposed method:adaptation to video frame interpolation and prediction}

Through the training on the event-based video frame reconstruction task, our approach is expected to acquire a powerful generative capability to map sparse, asynchronous event data to realistic and continuous video frames. Exploiting this ability, we extend our method to video frame interpolation and prediction without additional fine-tuning. Specifically, for video frame interpolation (or prediction), by leveraging the prior information from the first and last reference frames (or solely the first frame), we modulate the score function to theoretically reformulate the reverse diffusion sampling during the inference phase. This revised formulation then guides the video diffusion model to reconstruct the intermediate (or subsequent) frames with enhanced temporal consistency and visual fidelity.

% \begin{center}
\begin{algorithm}[t]
\caption{Reverse Diffusion Sampling for Video Frame Interpolation and Prediction}
\begin{algorithmic}[1]
\State \textbf{input:} clean latents $\{\revision{\mathcal{E}_{rgb}(\mathbf{V}_{0})}, \revision{\mathcal{E}_{rgb}(\mathbf{V}_{F-1})}\}$ (resp. $\{\mathcal{E}(\mathbf{V}_{0})\}$) for the task of video frame interpolation (resp. prediction), 
\revision{clean-latent estimation operator} $\mathcal{X}_{\theta}(\cdot)$, encoded event representation $\mathcal{E}(\mathbf{E})$.
\State initialize $\mathbf{X}^{T} \sim \mathcal{N}(\mathbf{0},\revision{\sigma_{T}^{2}}\mathbf{I})$.
\For{$t = T,\ldots,1$}
\State \revision{compute} $\mathbf{U}^{t}=\mathcal{X}_{\theta}(\mathbf{X}^{t},\revision{\mathcal{E}_{event}(\mathbf{E})},t)$.
\For{$i = 0,\ldots,F-1$}
\If{\underline{\textit{video frame interpolation}}}
\State calculate $\revision{\tilde{\mathbf{U}}^{t}_{i}}$ via Eq. (\ref{mu_t_alpha}),
\ElsIf{\underline{\textit{video frame prediction}}}
\State calculate $\revision{\tilde{\mathbf{U}}^{t}_{i}}$ via Eq. (\ref{mu_t_beta}),
\EndIf
\EndFor
\State \revision{form $\tilde{\mathbf{U}}^{t}=\{\tilde{\mathbf{U}}^{t}_{i}\}_{i=0}^{F-1}$ and reverse $\mathbf{X}^{t}$ to $\mathbf{X}^{t-1}$ via $\mathbf{X}^{t-1}=\mathcal{R}(\mathbf{X}^t,\tilde{\mathbf{U}}^t)$.}
\EndFor
\State \Return reconstructed latent $\mathbf{X}^0$.
\end{algorithmic}
\label{alg:reverse diffusion sampling}
\end{algorithm}
% \vspace{-10pt}
% \end{center}

Here, we take video frame interpolation as an example to illustrate our approach.
Given the clean latents \revision{$\mathcal{E}_{rgb}(\mathbf{V}_{0})$} and \revision{$\mathcal{E}_{rgb}(\mathbf{V}_{F-1})$}, obtained by feeding the first and last video frames into the autoencoder’s encoder, we first compute the deviations between them and the corresponding intermediate estimations $\mathbf{U}^{t}_{0}$ and $\mathbf{U}^{t}_{F-1}$:
% \vspace{-0.3cm}
\begin{equation}
\begin{aligned}
\mathbf{D}^{t}_{0} =& \revision{\mathcal{E}_{rgb}(\mathbf{V}_{0})} - \mathbf{U}^{t}_{0}, \\
\mathbf{D}^{t}_{F-1} =& \revision{\mathcal{E}_{rgb}(\mathbf{V}_{F-1})} - \mathbf{U}^{t}_{F-1}.
\end{aligned}
\end{equation}
Since these deviations are highly correlated with the discrepancy between the estimated clean latents and the provided prior information, an effective score function can be designed by utilizing $\mathbf{D}^{t}_{0}$ and $\mathbf{D}^{t}_{F-1}$ to modulate the estimated latent representation $\mathbf{U}^{t}_{i}$ for $0 \leq i \leq F-1$.
Thus, the score function of the reverse diffusion sampling is formulated:
\begin{equation}
\begin{aligned}
\tilde{\mathbf{U}}^{t}_{i} = \,\, \alpha(t) &\left[ \frac{(\mathbf{D}^{t}_{0} + \mathbf{U}^{t}_{i}) + (\mathbf{D}^{t}_{F-1} + \mathbf{U}^{t}_{i})}{2} \right] \\
&+ [1-\alpha(t)]\mathbf{U}^{t}_{i}
\end{aligned}
\label{mu_t_alpha}
\end{equation}
where the weighting coefficient $\alpha(t)=1-e^{-\sigma_{t}}$ balances the influence of the deviation correction.
To guide the reverse sampling of the video diffusion model, \revision{we form the modulated clean latent sequence $\tilde{\mathbf{U}}^{t}=\{\tilde{\mathbf{U}}^{t}_{i}\}_{i=0}^{F-1}$ and use it in place of the original estimated clean latent $\mathbf{U}^{t}$ in the reverse diffusion update, yielding}
$\mathbf{X}^{t-1}=\mathcal{R}(\mathbf{X}^t,\tilde{\mathbf{U}}^t)$.
It is important to note that for video frame prediction, only \revision{$\mathcal{E}_{rgb}(\mathbf{V}_{0})$} is available. 
In this case, the score function is modified into
\begin{equation}
\tilde{\mathbf{U}}^{t}_{i} = \alpha(t)(\mathbf{D}^{t}_{0} + \mathbf{U}^{t}_{i}) + [1-\alpha(t)]\mathbf{U}^{t}_{i}.
\label{mu_t_beta}
\end{equation}
Alg. \ref{alg:reverse diffusion sampling} summarizes the reverse diffusion sampling process.

\section{Experiment}
\label{sec:experiment}

\subsection{Experiment Settings}
\label{sec:experiment:implementation details}

\noindent
\textbf{Dataset.} 
The training set was generated by synthesizing event–frame pairs from real-world videos. Specifically, 1,800 sequences (400–500 frames each) were drawn from TrackingNet~\cite{muller2018trackingnet}, and an event stream between consecutive frames was simulated with DVS-Voltmeter~\cite{lin2022dvs}. For evaluation, we constructed a synthetic test set of 212 sequences from TrackingNet~\cite{muller2018trackingnet} and a real-world test set of 107 sequences from HS-ERGB~\cite{tulyakov2021time}, each containing 12 frames. More details of the dataset are provided in Appendix \textcolor{blue}{B}.

\begin{table*}[t]
\small
\centering
\caption{Quantitative comparison with state-of-the-art event-based video frame reconstruction methods on real-world and synthetic datasets.
MSE is reported in units of $10^{-2}$.
The best result in each column is highlighted in \textbf{bold}.
}
\vspace{-5pt}
\setlength{\tabcolsep}{19pt}
\begin{tabular}{l|c|c|c|c|c|c}
\toprule
\multirow{2}{*}{Method} & \multicolumn{3}{c|}{Real-World} & \multicolumn{3}{c}{Synthetic} \\
\cline{2-7}
& MSE~$\downarrow$ & SSIM~$\uparrow$ & LPIPS~$\downarrow$
& MSE~$\downarrow$ & SSIM~$\uparrow$ & LPIPS~$\downarrow$ \\
\hline
\rowcolor[rgb]{ .95,  .95,  .95}
E2VID~\cite{rebecq2019high} & 12.75 & 0.4200 & 0.6210 & 6.78 & 0.5040 & 0.5420 \\
FireNet~\cite{scheerlinck2020fast} & 12.10 & 0.4110 & 0.6300 & 6.20 & 0.5560 & 0.5220 \\
\rowcolor[rgb]{ .95,  .95,  .95}
E2VID+~\cite{stoffregen2020reducing} & 6.50 & 0.4390 & 0.6180 & 5.50 & 0.5130 & 0.5430 \\
FireNet+~\cite{stoffregen2020reducing} & 7.37 & 0.3870 & 0.6620 & 5.81 & 0.4540 & 0.5790 \\
\rowcolor[rgb]{ .95,  .95,  .95}
ETNet~\cite{weng2021event} & 8.49 & 0.4230 & 0.6440 & 5.22 & 0.5110 & 0.5480 \\
SSL-E2VID~\cite{paredes2021back} & 10.08 & 0.3660 & 0.6260 & 6.94 & 0.4980 & 0.6510 \\
\rowcolor[rgb]{ .95,  .95,  .95}
SPADE-E2VID~\cite{cadena2021spade} & 7.27 & 0.4330 & 0.5990 & 9.92 & 0.4500 & 0.6340 \\
CUBE~\cite{zhao2024controllable} & 8.51 & 0.3640 & 0.6900 & 14.37 & 0.1920 & 0.7970 \\
\rowcolor[rgb]{ .95,  .95,  .95}
HyperE2VID~\cite{ercan2024hypere2vid} & 6.32 & 0.4770 & \textbf{0.5620} & 7.27 & 0.3860 & 0.6320 \\
\textbf{\methodname~(Ours)} & \textbf{6.12} & \textbf{0.4990} & 0.6740 & \textbf{1.67} & \textbf{0.7100} & \textbf{0.3940} \\
\bottomrule
\end{tabular}
\label{table:vfr:sota:hs-ergb and synthetic}
\vspace{-10pt}
\end{table*}

\begin{figure*}[t]
\centering
\includegraphics[width=\linewidth]{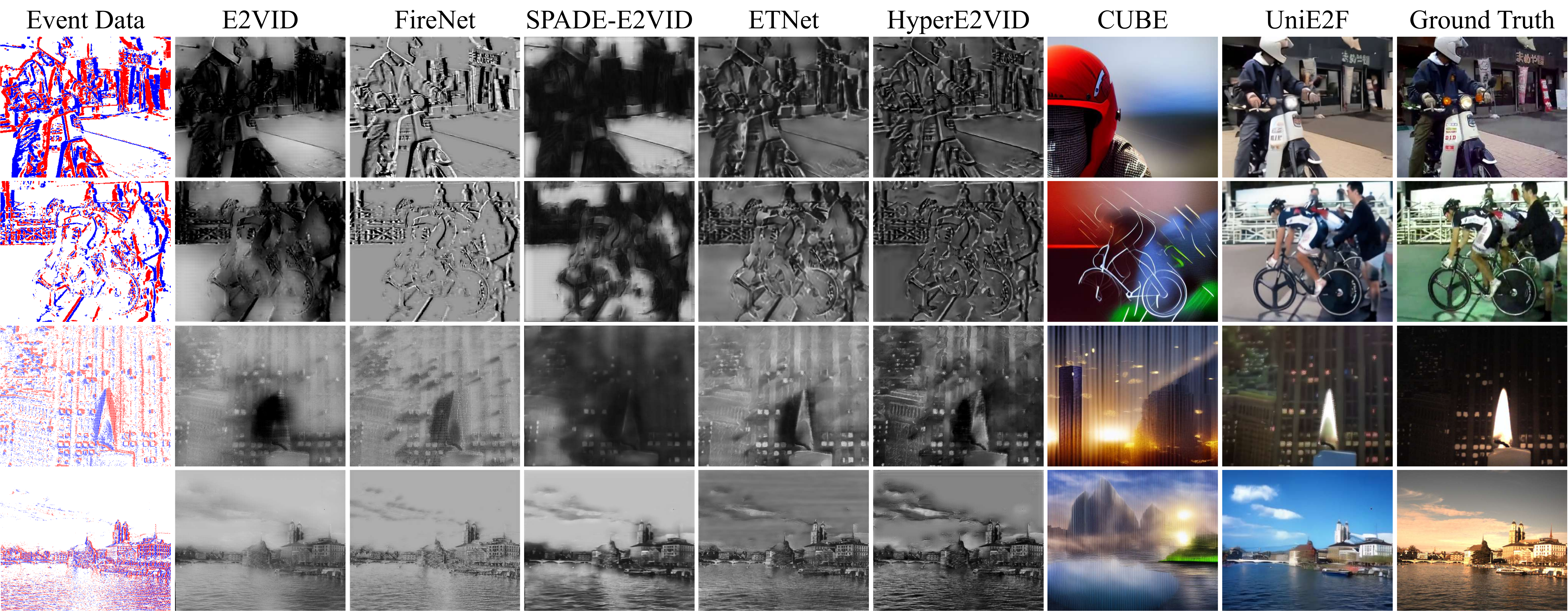}
\vspace{-15pt}
\caption{Visual comparison of event-based video frame reconstruction results on synthetic (1st and 2nd rows) and real-world (3rd and 4th rows) datasets. The event data is visualized as a polarity map.}
\label{fig:comparison with the state-of-the-art methods}
\vspace{-15pt}
\end{figure*}

\vspace{0.5em}
\noindent
\textbf{Implementation Details.} 
All the experiments were conducted on an NVIDIA RTX A6000 GPU.
\revisionx{We adopted a stage-wise training strategy. The event autoencoder was first pretrained independently, and the inter-frame residual estimator was trained separately. During subsequent event-conditioned diffusion-model fine-tuning, the denoising U-Net was initialized with the pretrained SVD weights, and the pretrained event encoder was kept frozen. Only the temporal transformer blocks in the denoising U-Net were fine-tuned; all other SVD parameters remained frozen. The diffusion model was optimized for 450,000 iterations using the AdamW optimizer~\cite{loshchilov2018decoupled} with a learning rate of $1\times10^{-5}$.} In both the training and inference stages, we set $F=12$ and employed 30 diffusion steps for the inference phase, with our event-based inter-frame residual guidance applied during the last 10 steps. Besides, we illustrate the training strategy for the inter-frame residual estimator in Appendix \textcolor{blue}{C}.
In terms of inference latency, reconstructing a sequence of 12 RGB frames with the resolution of 448$\times$320 takes about 48 seconds.
To quantitatively evaluate the model's performance, we employ three widely used metrics: MSE, SSIM~\cite{wang2004image}, and LPIPS~\cite{zhang2018unreasonable}.
A more detailed description of the implementation is provided in Appendix \textcolor{blue}{F}.

\begin{table*}[t]
\small
\centering
\caption{Quantitative comparison for video frame interpolation (VFI) and video frame prediction (VFP). MSE is reported in units of $10^{-2}$. The best result in each column is highlighted in \textbf{bold}. Here, ``$\dagger$'' denotes the model with pretrained weight and ``$^{*}$'' denotes the model retrained on our synthetic dataset.}
\setlength{\tabcolsep}{14.5pt}
\begin{tabular}{l|c|c|c|c|c|c|c}
\toprule
\multirow{2}{*}{Method} & \multirow{2}{*}{Mode}  & \multicolumn{3}{c|}{Synthetic} & \multicolumn{3}{c}{Real-World} \\
\cline{3-5} \cline{6-8}
 & & MSE $\downarrow$ & SSIM $\uparrow$ & LPIPS $\downarrow$ & MSE $\downarrow$ & SSIM $\uparrow$ & LPIPS $\downarrow$ \\
\hline
\rowcolor[rgb]{ .95,  .95,  .95}
CBMNet$^{\dagger}$~\cite{kim2023event} & VFI-4$\times$ & 2.50 & 0.5860 & \textbf{0.3030} & 0.32 & 0.8230 & \textbf{0.2540} \\
CBMNet$^{*}$~\cite{kim2023event} & VFI-4$\times$ & 1.74 & 0.6320 & 0.3550 & \textbf{0.23} & \textbf{0.8360} & 0.2810 \\
\rowcolor[rgb]{ .95,  .95,  .95}
TimeLens-XL$^{\dagger}$~\cite{ma2024timelens} & VFI-4$\times$ & 3.21 & 0.5350 & 0.3270 & 0.80 & 0.7270 & 0.2920 \\
TimeLens-XL$^{*}$~\cite{ma2024timelens} & VFI-4$\times$ & 2.91 & 0.5480 & 0.3160 & 0.78 & 0.7270 & 0.2900 \\
\rowcolor[rgb]{ .95,  .95,  .95}
\textbf{\methodname~(Ours)} & VFI-4$\times$ & \textbf{0.63} & \textbf{0.7340} & 0.3210 & 0.41 & 0.6770 & 0.4310 \\
\hline
CBMNet$^{\dagger}$~\cite{kim2023event} & VFI-11$\times$ & 4.91 & 0.4040 & 0.4580 & 8.92 & 0.5120 & 0.5800 \\
\rowcolor[rgb]{ .95,  .95,  .95}
CBMNet$^{*}$~\cite{kim2023event} & VFI-11$\times$ & 3.92 & 0.4530 & 0.5270 & 0.63 & \textbf{0.7500} & 0.4000 \\
RE-VDM$^{\dagger}$~\cite{chen2025repurposing} & VFI-11$\times$ & 5.03 & 0.4180 & 0.4130 & \textbf{0.57} & 0.7330 & \textbf{0.3480} \\
\rowcolor[rgb]{ .95,  .95,  .95}
\textbf{\methodname~(Ours)} & VFI-11$\times$ & \textbf{0.72} & \textbf{0.7400} & \textbf{0.3200} & 0.58 & 0.6500 & 0.4000 \\
\hline
Zhu et al.\cite{zhu2024videoframe}$^{\dagger}$ & VFP & 1.84 & 0.6140 & 0.3960 & \textbf{0.77} & \textbf{0.6620} & \textbf{0.3400} \\
\rowcolor[rgb]{ .95,  .95,  .95}
\textbf{\methodname~(Ours)} & VFP & \textbf{0.93} & \textbf{0.7100} & \textbf{0.3470} & 1.00 & 0.5940  & 0.4110 \\
\bottomrule
\end{tabular}
\vspace{-10pt}
\label{table:vfi and vfp:sota:hs-ergb and synthetic}
\end{table*}

\begin{figure*}[t]
\centering
\includegraphics[width=0.9\linewidth]{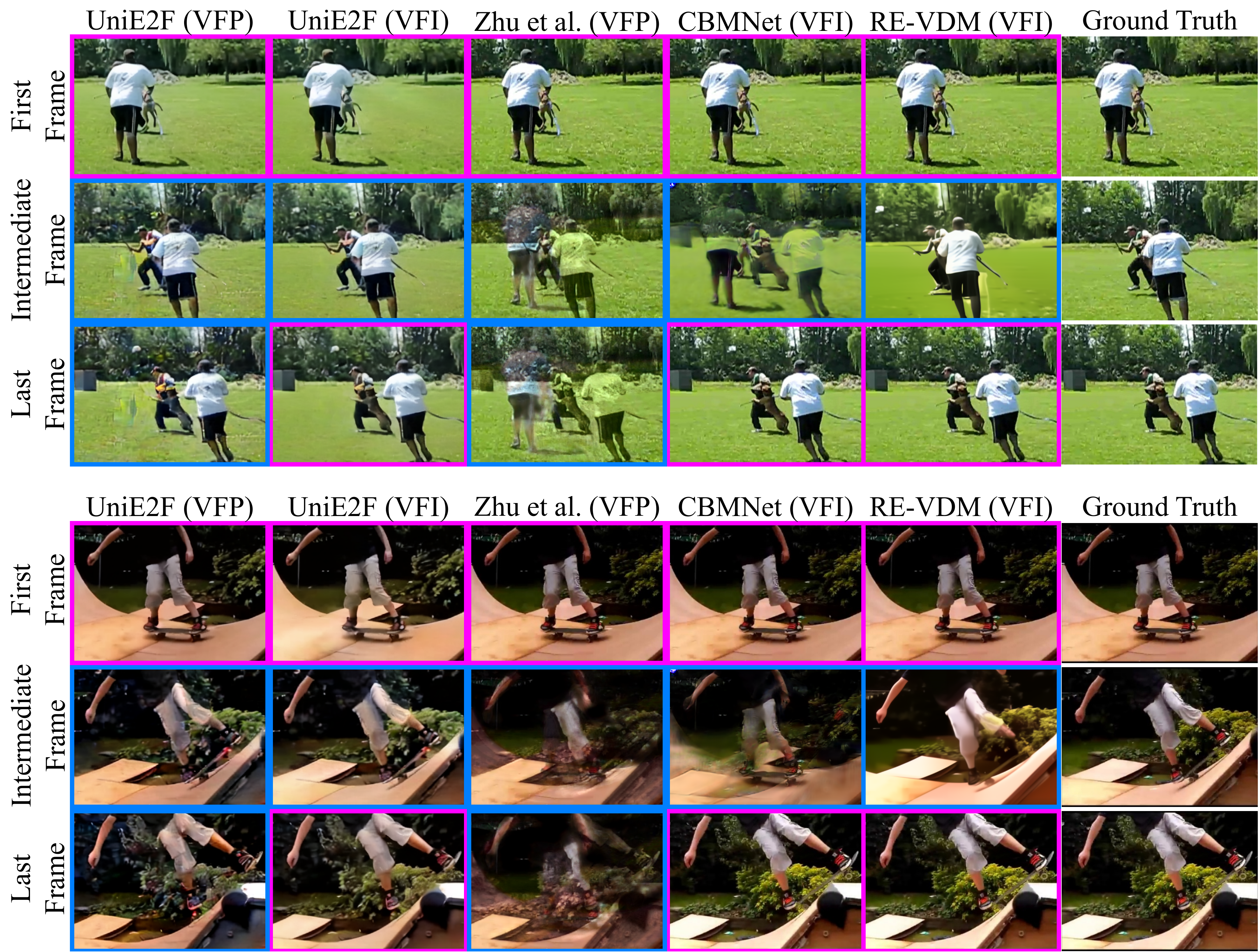}
\caption{Visual comparison on synthetic dataset: \methodname~is shown in both video frame interpolation (VFI) and video frame prediction (VFP) modes, while CBMNet is shown under the VFI setting. Note that the frames highlighted with purple borders denote the given frames, whereas frames highlighted with blue borders denote the predicted frames.}
\label{fig:vfi and vfp:sota:synthetic}
\vspace{-10pt}
\end{figure*}

\subsection{Results of Event-based Frame Reconstruction}
\label{sec:experiment:comparison with the state-of-the-art methods}

\noindent
\textbf{Comparison on Synthetic and Real-World Datasets.} 
We compared our method against various event-based reconstruction approaches, employing their official implementations and pre-trained parameters to ensure a fair evaluation. 
\revision{It should be noted that, because the compared methods differ in model scale and training data, the reported results should be understood as a comparison of different methods under their respective design paradigms, rather than as a strictly controlled comparison under identical pre-training conditions.
}
\revision{Besides}, since the methods in~\cite{liang2023event}, \cite{zhu2024temporal}, and \cite{chen2024lase} \textit{were not open-sourced}, they could not be included in our comparison.

The quantitative results, as summarized in~\tableref{table:vfr:sota:hs-ergb and synthetic}, reveal that our approach consistently outperforms the compared methods across multiple metrics.
\revision{Specifically, on the synthetic dataset, our approach yields significant performance gains across all the evaluation metrics. Moreover, on the real-world dataset, UniE2F achieves the best MSE of 0.0612 and SSIM of 0.4990, while its LPIPS score is relatively inferior to that of the best-performing baseline. We believe this discrepancy is mainly related to the intrinsic ambiguity of recovering RGB appearance from achromatic events. Unlike the compared methods, which mainly reconstruct grayscale intensity images, UniE2F leverages the generative prior of Stable Video Diffusion to recover RGB colors and textures. However, the synthetic-to-real domain gap further increases the uncertainty of this inherently ill-posed event-to-RGB mapping, such that the generated colors and textures may differ from the exact ground-truth appearance even when the overall intensity distribution and spatial structure are well preserved. Since LPIPS measures similarity in a deep perceptual feature space rather than directly evaluating pixel-wise or structural differences~\cite{zhang2018unreasonable}, such RGB appearance discrepancies can affect LPIPS differently from MSE and SSIM, resulting in the observed inconsistency among these metrics.}

Qualitative comparisons illustrated in~\figref{fig:comparison with the state-of-the-art methods} (with more visual comparison results provided in Appendix \textcolor{blue}{H}) reveal that our \methodname, leveraging a large-scale generative prior along with inter-frame residual guidance, achieves reconstruction with improved color fidelity and fewer artifacts.
In contrast, the compared methods, trained on grayscale images, tend to produce monochromatic outputs inconsistent with real-world color scenes, suffer from significant detail loss, and exhibit obvious artifacts.

\textit{It is worth noting} that although our method yields more natural and realistic reconstruction, a noticeable discrepancy in color tone remains between our result and the ground truth. This is primarily because event streams capture only intensity changes and inherently lack color information. Without the prior information about the colors of the scenes, 
achieving perfectly consistent color restoration remains an inherently challenging and ill-posed problem.

\vspace{0.5em}
\noindent
\textbf{Comparison on HQF, IJRR, and MVSEC Datasets.} 
We further compare UniE2F qualitatively with existing methods  on three widely used real-world datasets: HQF~\cite{stoffregen2020reducing}, IJRR~\cite{mueggler2017event}, and MVSEC~\cite{zhu2018multivehicle}, which provide single-channel intensity frames as ground truth. As shown in ~\figref{fig:vfr:hqf_ijrr_and_mvsec}, \revisionx{UniE2F generates RGB reconstructions with perceptually plausible colors and clear details}. In contrast, the other methods can only produce single-channel grayscale reconstructions, which lack color information and look less natural.

\begin{figure*}[p]
\centering
\includegraphics[width=0.9\linewidth]{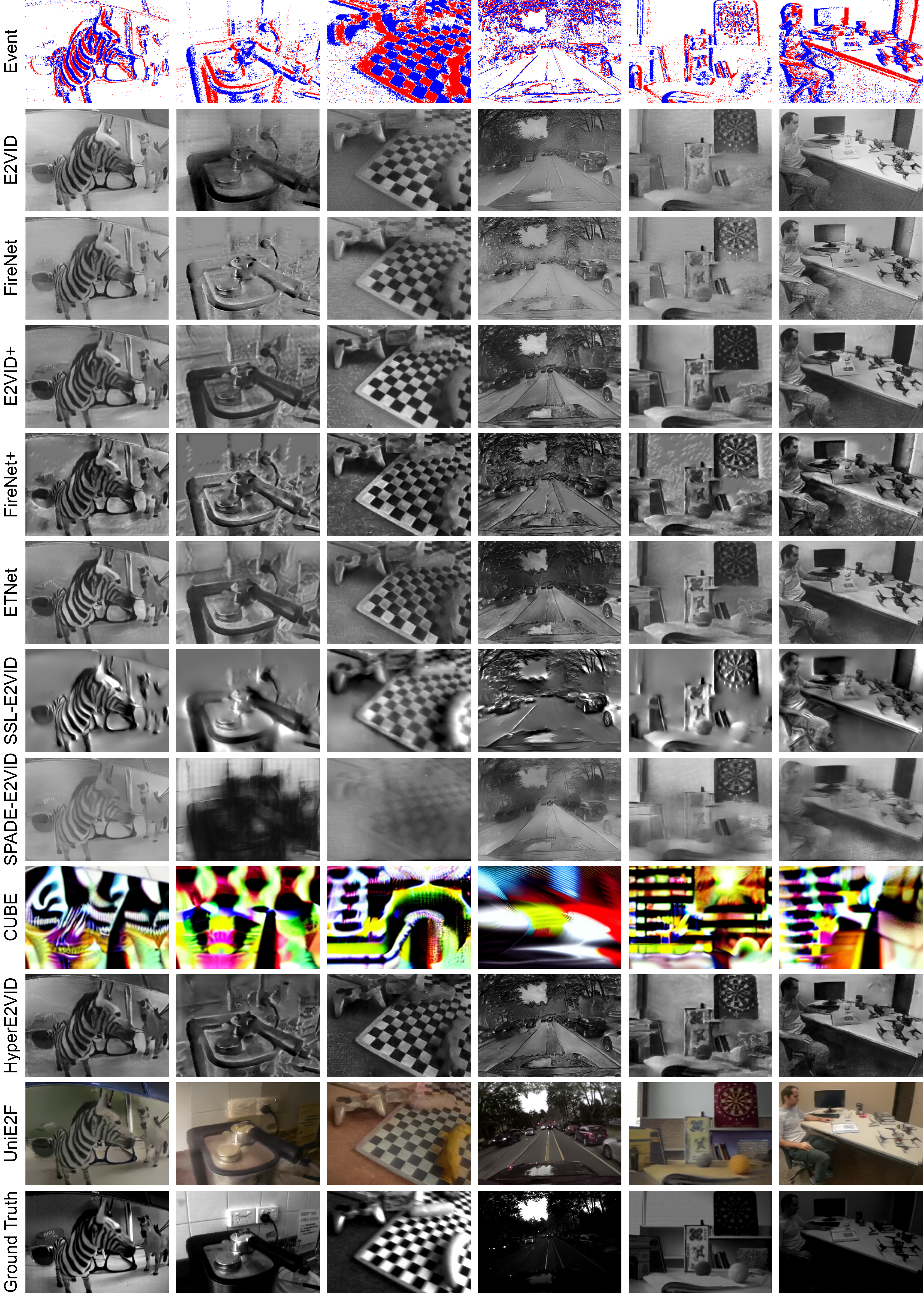}
\caption{Qualitative Comparison on HQF~\cite{stoffregen2020reducing}, IJRR~\cite{mueggler2017event}, and MVSEC~\cite{zhu2018multivehicle}.}
\label{fig:vfr:hqf_ijrr_and_mvsec}
\end{figure*}

\vspace{0.5em}
\noindent
\revision{\textbf{Comparison on EventAid Dataset.}}
\revision{
To further evaluate the generalization capability of our method on real-captured event data, we additionally conduct experiments on the EventAid benchmark~\cite{duan2025eventaid}.
For each scene, we sample multiple subsequences consisting of 12 consecutive frames for evaluation.
As shown in~\tableref{table:eventaid}, \methodname~achieves the best MSE and SSIM among the compared methods, reducing the MSE from the second-best result of 0.0803 to 0.0721 and improving the SSIM from 0.4220 to 0.4280.
Meanwhile, \methodname~achieves an LPIPS of 0.6350, which remains competitive with existing approaches.
These results demonstrate that \methodname~maintains favorable reconstruction performance when directly applied to unseen real-captured event data, further validating its cross-dataset generalization capability.
}

\begin{table}[t]
\small
\centering
\caption{\revision{Quantitative comparison of event-based frame reconstruction on the EventAid~\cite{duan2025eventaid} benchmark. MSE is reported in units of $10^{-2}$. The best result in each column is highlighted in \textbf{bold}.}}
\setlength{\tabcolsep}{12pt}
\begin{tabular}{l|c|c|c}
\toprule
\multirow{2}{*}{Method} 
& \multicolumn{3}{c}{EventAid~\cite{duan2025eventaid}} \\
\cline{2-4}
& MSE~$\downarrow$ 
& SSIM~$\uparrow$ 
& LPIPS~$\downarrow$ \\
\hline
\rowcolor[rgb]{ .95, .95, .95}
E2VID~\cite{rebecq2019high} & 16.58 & 0.3410 & 0.6600 \\
FireNet~\cite{scheerlinck2020fast} & 16.36 & 0.3450 & 0.6410 \\
\rowcolor[rgb]{ .95, .95, .95}
E2VID+~\cite{stoffregen2020reducing} & 8.03 & 0.3870 & \textbf{0.6080} \\
FireNet+~\cite{stoffregen2020reducing} & 8.81 & 0.3380 & 0.6410 \\
\rowcolor[rgb]{ .95, .95, .95}
ETNet~\cite{weng2021event} & 8.55 & 0.3800 & 0.6360 \\
SSL-E2VID~\cite{paredes2021back} & 8.61 & 0.4220 & 0.6220 \\
\rowcolor[rgb]{ .95, .95, .95}
SPADE-E2VID~\cite{cadena2021spade} & 16.51 & 0.3560 & 0.6450 \\
CUBE~\cite{zhao2024controllable} & 12.40 & 0.2770 & 0.7370 \\
\rowcolor[rgb]{ .95, .95, .95}
HyperE2VID~\cite{ercan2024hypere2vid} & 8.22 & 0.3730 & 0.6260 \\
\textbf{\methodname~(Ours)} & \textbf{7.21} & \textbf{0.4280} & 0.6350 \\
\bottomrule
\end{tabular}
\label{table:eventaid}
\vspace{-10pt}
\end{table}

\subsection{Results of Video Frame Interpolation and Prediction}
\label{sec:experiment:unified framework for video frame reconstruction}

Here, we conducted experiments to evaluate the \textit{zero-shot} capability of our \methodname~on the event‑based video frame interpolation and prediction tasks.
Specifically, we defined three tasks: VFI-4$\times$, VFI-11$\times$, and VFP, which respectively increased the frame rate to 4 times and 11 times the original and predicted subsequent frames given the first frame.
Following the above setup, we compared our method with several representative baselines for interpolation using their official pretrained weights or retraining them on our synthetic training set.
Since TimeTracker~\cite{liu2025timetracker} had \textit{not been open-sourced}, its weights and results could not be obtained and were therefore omitted.
More detailed descriptions of the experimental setting are provided in Appendix \textcolor{blue}{D}.

The quantitative and qualitative results are presented in~\tableref{table:vfi and vfp:sota:hs-ergb and synthetic} and~\figref{fig:vfi and vfp:sota:synthetic}, respectively.
\revision{On the synthetic dataset, UniE2F delivers a clear advantage in both short-range and long-range interpolation as well as prediction, achieving the best performance on most evaluation metrics. On the real-world datasets, UniE2F shows relatively weaker quantitative performance compared to the task-specific baselines, which is mainly attributed to the domain gap between synthetic training and real-world event distributions. Specifically, UniE2F is trained on simulated event--frame pairs, while discrepancies in the event formation process and event--frame mapping between simulated and real data inevitably affect its generalization to real-world scenarios. Nevertheless, considering that our model operates in a strict zero-shot setting without any fine-tuning for VFI/VFP or on real-world data, its ability to remain competitive on unseen real-world interpolation and prediction tasks is encouraging.}
Moreover, beyond numerical scores, \methodname~preserves motion dynamics while maintaining high-fidelity color, texture, and structural details, whereas outputs of other methods exhibit artifacts and degradation due to the lack of generative priors.
Remarkably, these zero-shot results are obtained \textbf{without any fine-tuning on interpolation or prediction datasets}, which confirms that \methodname~provides a unified and flexible framework that not only delivers excellent performance in event‑driven frame reconstruction but also delivers outstanding zero‑shot performance on both interpolation and prediction tasks, highlighting its \revision{cross-task generalization capability} and practical utility.

\section{Ablation Study\protect\footnote{More ablation studies are provided in the Appendix \textcolor{blue}{E}}}
\label{sec:ablation study}

\noindent
\textbf{Guidance Strength Strategy.}
We further investigated the optimal guidance strength strategy by evaluating four scheduling approaches under the experimental setup where the guidance was applied exclusively during the final 10 sampling steps.
Specifically, we designed four scheduling strategies over the designated steps: (\textbf{1}) a baseline configuration with no guidance, i.e., the guidance strength was set to 0; (\textbf{2}) maintaining a constant guidance strength of 0.1; (\textbf{3}) linearly decreasing the guidance strength from 0.1 to 0, and (\textbf{4}) linearly increasing the guidance strength from 0 to 0.1.
The comparison results \tableref{table:vfr:ablation:synthetic:guidance strength} show that the linearly decreasing schedule (0.1 $\rightarrow$ 0.0) achieves the best performance compared with other configurations.
The results indicate that applying stronger guidance early in the reverse diffusion process helps to enforce accurate inter-frame residual alignment, while gradually reducing the guidance allows the model’s generative prior to effectively refine finer details and prevent over-constraining the reconstruction. 
This progressive relaxation appears critical for balancing reconstruction fidelity with visual diversity.
\begin{table}[t]
\small
\centering
\caption{
Comparison of different guidance strength strategies.
MSE is reported in units of $10^{-2}$.
}
\setlength{\tabcolsep}{12pt}
\begin{tabular}{c|c|c|c}
\toprule
\multirow{2}{*}{Guidance Strength} & \multicolumn{3}{c}{Synthetic} \\
\cline{2-4}
& MSE~$\downarrow$ & SSIM~$\uparrow$ & LPIPS~$\downarrow$ \\
\hline
\rowcolor[rgb]{ .95,  .95,  .95}
0.0 $\rightarrow$ 0.0 & 1.91 & 0.6880 & 0.4020 \\
0.1 $\rightarrow$ 0.1 & 2.34 & 0.6610 & 0.4190 \\
\rowcolor[rgb]{ .95,  .95,  .95}
0.0 $\rightarrow$ 0.1 & 2.34 & 0.6540 & 0.4210 \\
0.1 $\rightarrow$ 0.0 & \textbf{1.67} & \textbf{0.7100} & \textbf{0.3940} \\
\bottomrule
\end{tabular}
\vspace{-15pt}
\label{table:vfr:ablation:synthetic:guidance strength}
\end{table}

\vspace{0.5em}
\noindent
\textbf{Guidance Mode.}
Since the SVD model conducts the denoising process in the latent space, a direct strategy to enhance inter-frame consistency is to impose explicit constraints on the residuals between adjacent frames within this space. 
To this end, we trained a ResNet with the aim of predicting inter-frame residuals within latent representation.
As shown in \figref{fig:vfr:ablation:synthetic:guidance-mode} and \tableref{table:vfr:ablation:synthetic:guidance mode}, the results of latent-level guidance are inferior to frame-level guidance results.
Moreover, the images generated with frame-level guidance exhibit higher fidelity and richer details, producing sharper structures and more natural textures.
In contrast, the outputs under latent-level guidance often suffer from noticeable distortions and artifacts, leading to degraded perceptual quality.
The performance gap can be attributed to the Gaussian distribution inherent in the latent inter-frame residuals, making it particularly challenging for networks trained with MAE loss to accurately model the precise variations.
\begin{table}[t]
\small
\centering
\setlength{\tabcolsep}{13pt}
\caption{
Quantitative comparison of different guidance modes.
MSE is reported in units of $10^{-2}$.
}
\vspace{-5pt}
\begin{tabular}{c|c|c|c}
\toprule
\multirow{2}{*}{Guidance Mode} & \multicolumn{3}{c}{Synthetic} \\
\cline{2-4}
& MSE~$\downarrow$ & SSIM~$\uparrow$ & LPIPS~$\downarrow$ \\
\hline
\rowcolor[rgb]{ .95,  .95,  .95}
Latent & 1.97 & 0.6920 & 0.4060 \\
Frame & \textbf{1.67} & \textbf{0.7100} & \textbf{0.3940} \\
\bottomrule
\end{tabular}
\label{table:vfr:ablation:synthetic:guidance mode}
\vspace{-15pt}
\end{table}
\begin{figure}[t]
% \vspace{-30pt}
\includegraphics[width=0.9\linewidth]{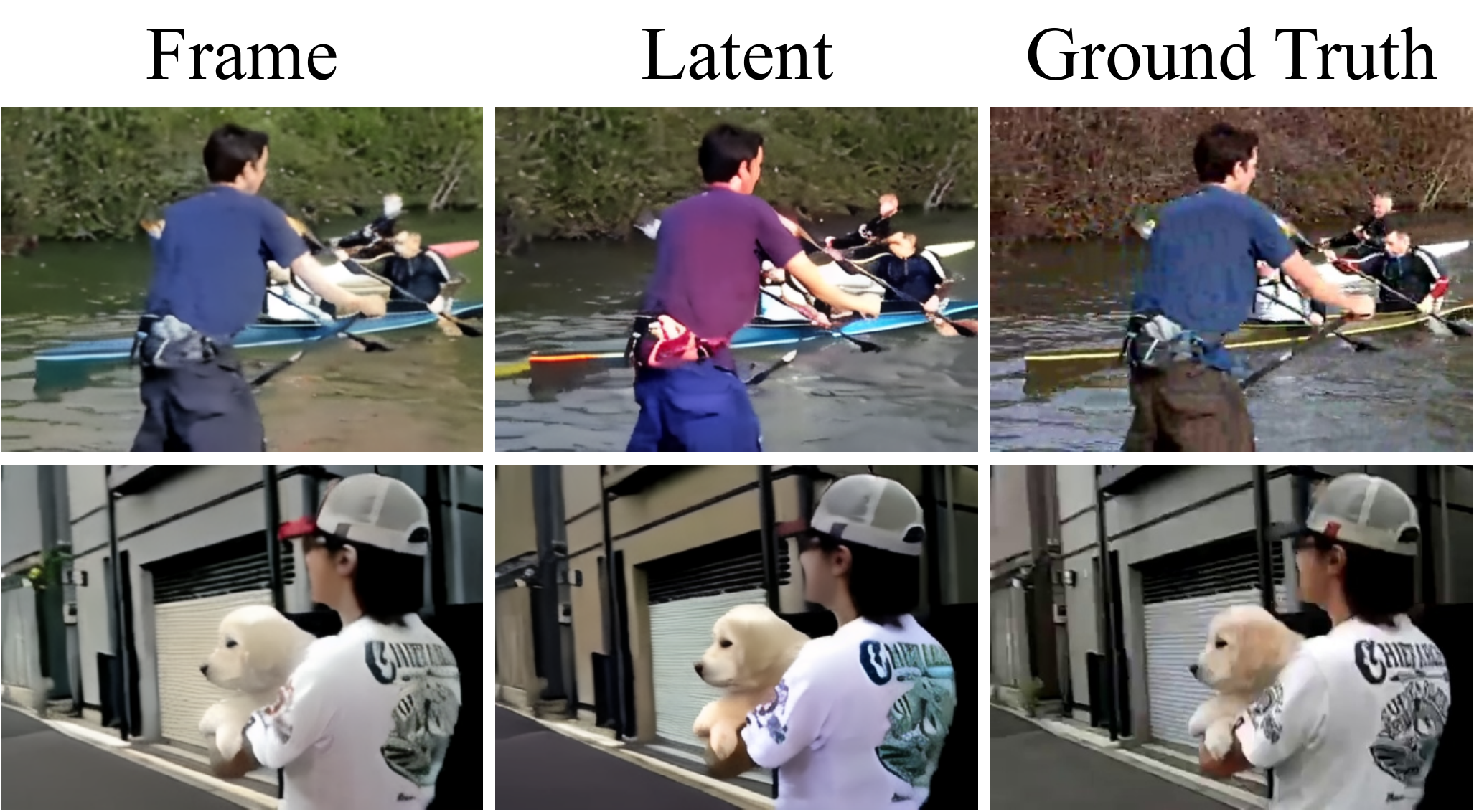}
\caption{Visual comparison of different guidance modes.}
\vspace{-5pt}
\label{fig:vfr:ablation:synthetic:guidance-mode}
\end{figure}

\vspace{0.5em}
\noindent
\textbf{Max Guidance Strength.}
The value of the weighting coefficient $s$ controls the gradient update magnitude for the estimated clean latent $U_t$ by weighting the contribution of the inter-frame residual loss. In the ablation experiments for the video frame reconstruction task, to evaluate the effect of the weighting coefficient $s$, we expanded the maximum guidance strength from 0.1 to larger values such as 0.5, 1.0, 5.0 and 10.0, and compared MSE, SSIM, and LPIPS (see ~\tableref{table:vfr:ablation:synthetic:maximum guidance strength}). The results show that as the coefficient grows, MSE increases, SSIM drops, and LPIPS rises—indicating that overly strong residual guidance causes the reconstruction to depend too heavily on the event-frame residual. Because the mapping from events to RGB involves nonlinear steps (e.g., gamma correction, ISP pipeline) and is not one‑to‑one, amplifying this loss instead introduces artifacts and degrades visual quality.
\begin{table}[t]
\small
\centering
\caption{
Quantitative comparison of different maximum guidance strength, evaluated using MSE, SSIM, and LPIPS.
MSE is reported in units of $10^{-2}$.
}
\setlength{\tabcolsep}{11pt}
\begin{tabular}{c|c|c|c}
\toprule
\multirow{2}{*}{\shortstack{Maximum Guidance\\Strength}} & \multicolumn{3}{c}{Synthetic} \\
\cline{2-4}
& MSE~$\downarrow$ & SSIM~$\uparrow$ & LPIPS~$\downarrow$ \\
\hline
\rowcolor[rgb]{ .95,  .95,  .95}
10.0 & 9.64 & 0.2950 & 0.7070 \\
5.0 & 5.73 & 0.4140 & 0.6510 \\
\rowcolor[rgb]{ .95,  .95,  .95}
1.0 & 2.28 & 0.6140 & 0.4770 \\
0.5 & 1.88 & 0.6710 & 0.4240 \\
\rowcolor[rgb]{ .95,  .95,  .95}
0.1 & \textbf{1.67} & \textbf{0.7100} & \textbf{0.3940} \\
\bottomrule
\end{tabular}
\label{table:vfr:ablation:synthetic:maximum guidance strength}
\vspace{-5pt}
\end{table}

\vspace{0.5em}
\noindent
\textbf{Inter-Frame Residual Guidance.}
To demonstrate the effectiveness of the Inter-Frame Residual Guidance (IFRG), we have included a visual comparison in \figref{fig:exp-ifrg}. As shown, without IFRG (top row), the reconstructions exhibit noticeable blurring and structural distortion: the bars of the fence are wavy and over-smoothed, and the ground textures are largely washed out. With IFRG (middle row), edges and fine structures align much better with the ground truth (bottom row): the fence bars are straighter and more regular, and the sand and background textures are more faithfully recovered. These improvements indicate that IFRG effectively constrains the reconstruction with inter-frame intensity changes, yielding frames that are both structurally more accurate to ground truth.
\begin{figure}[t]
% \vspace{-10pt}
\centering
\includegraphics[width=1.0\linewidth]{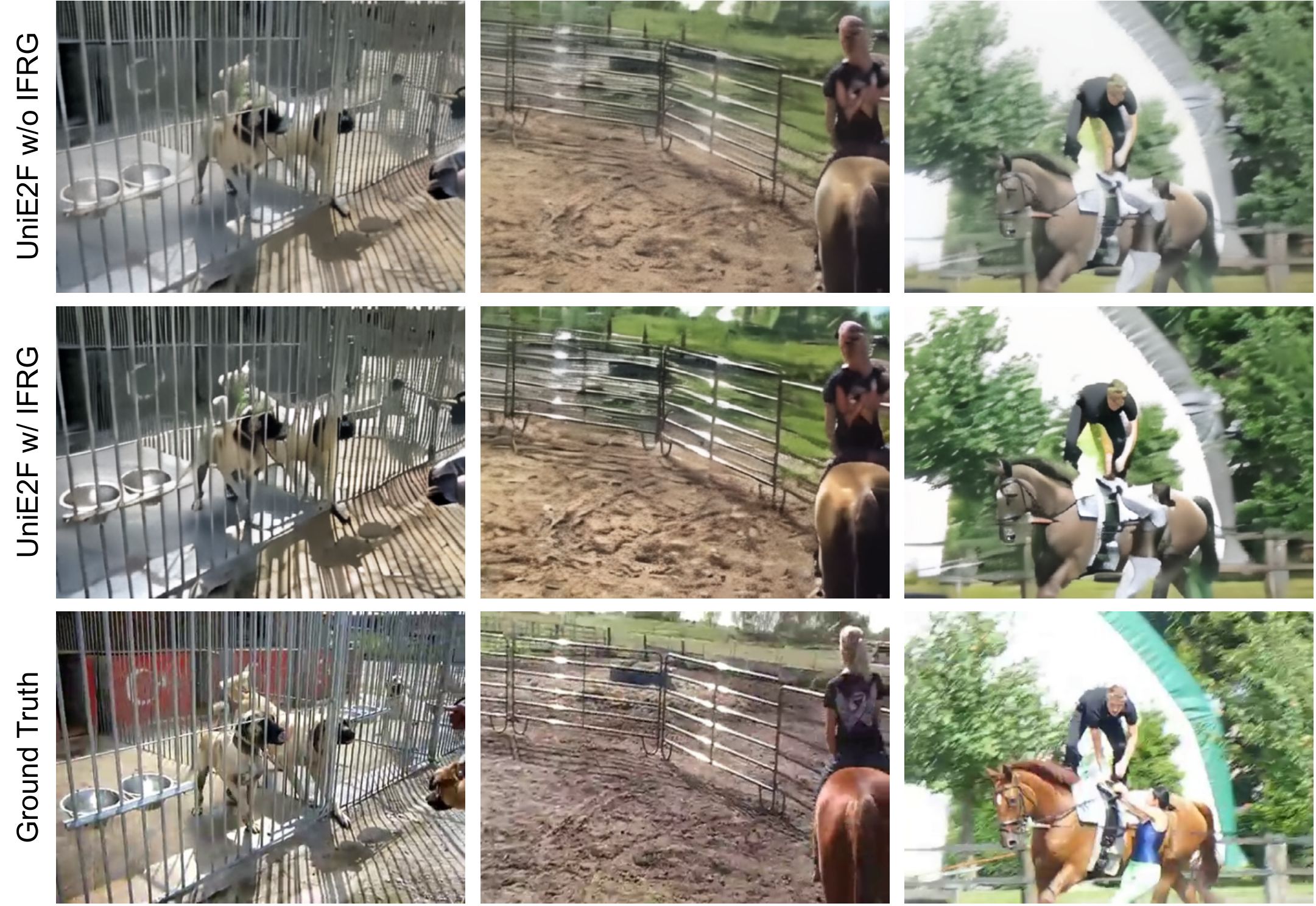}
\caption{Effect of Inter-Frame Residual Guidance (IFRG). Compared with the baseline (top), adding IFRG (middle) yields sharper edges and more faithful textures that better match the ground truth (bottom).}
\label{fig:exp-ifrg}
\vspace{-7pt}
\end{figure}

\vspace{0.5em}
\noindent
\textbf{Perceptual Quality Evaluation via Fréchet Inception Distance.}
The real-world datasets we used contained relatively few scenes, and the color distribution of HS-ERGB may not fully align with real-world conditions, leading to perceptual deviations from natural scenes. Since UniE2F was trained with RGB frames from real-world scenes in TrackingNet as supervisory targets, the color distribution of the reconstructed frames tends to be closer to real-world distributions, as seen in rows 3 and 4 of \figref{fig:comparison with the state-of-the-art methods}. 

As a result, perceptual differences in high-dimensional feature spaces, such as those captured by LPIPS, may appear larger. Thus, here, we evaluated the reconstruction performance of our method and the competing approaches using the Fréchet Inception Distance (FID) metric (calculated via the Inception v3 network) as shown in the~\tableref{table:vfr:fid}.
As can be seen, our method outperforms the competitors in terms of FID~\cite{heusel2017gans}, further supporting its advantages in generating perceptually high-quality reconstructions.
\begin{table}[t]
\small
\centering
\caption{
Quantitative comparison of perceptual quality using FID (lower is better) for video frame reconstruction. UniE2F achieves the best FID, demonstrating its advantage in generating perceptually natural frames.
}
\setlength{\tabcolsep}{16pt}
\begin{tabular}{l|c|c}
\toprule
\multirow{2}{*}{Method} & Real-World & Synthetic \\
\cline{2-3}
& FID~$\downarrow$ & FID~$\downarrow$ \\
\hline
\rowcolor[rgb]{ .95,  .95,  .95}
E2VID~\cite{rebecq2019high}        & 223.7926 & 179.1581 \\
FireNet~\cite{scheerlinck2020fast} & 241.5848 & 172.3014 \\
\rowcolor[rgb]{ .95,  .95,  .95}
E2VID+~\cite{stoffregen2020reducing} & 240.4527 & 203.2372 \\
FireNet+~\cite{stoffregen2020reducing} & 250.4045 & 245.4307 \\
\rowcolor[rgb]{ .95,  .95,  .95}
ETNet~\cite{weng2021event}         & 252.4636 & 207.1250 \\
SSL-E2VID~\cite{paredes2021back}   & 252.2566 & 241.1340 \\
\rowcolor[rgb]{ .95,  .95,  .95}
SPADE-E2VID~\cite{cadena2021spade} & 248.8114 & 193.7386 \\
CUBE~\cite{zhao2024controllable}   & 207.4970 & 192.1830 \\
\rowcolor[rgb]{ .95,  .95,  .95}
HyperE2VID~\cite{ercan2024hypere2vid} & 272.1271 & 224.4130 \\
\textbf{\methodname~(Ours)}         & \textbf{184.2509} & \textbf{57.1092} \\
\bottomrule
\end{tabular}
\vspace{-10pt}
\label{table:vfr:fid}
\end{table}

\begin{table}[t]
\small
\centering
\caption{
Comparison of computational overhead between our \methodname~and other methods, evaluated by reconstructing a sequence of 12 RGB frames at $448\times320$ resolution.
}
\setlength{\tabcolsep}{10pt}
\begin{tabular}{l|c|c}
\toprule
\multirow{2}{*}{Method} & \multirow{2}{*}{\shortstack{Computational\\Cost (TMACs)}}  &  \multirow{2}{*}{\shortstack{Peak Memory\\Footprint (MB)}}  \\
& & \\
\hline
\rowcolor[rgb]{ .95,  .95,  .95}
E2VID~\cite{rebecq2019high} & 0.583 & 622 \\
FireNet~\cite{scheerlinck2020fast} & 0.064 & 462 \\
\rowcolor[rgb]{ .95,  .95,  .95}
E2VID+~\cite{stoffregen2020reducing} & 0.782 & 816 \\
FireNet+~\cite{stoffregen2020reducing} & 0.064 & 436 \\
\rowcolor[rgb]{ .95,  .95,  .95}
ETNet~\cite{weng2021event} & 1.737 & 1074 \\
SSL-E2VID~\cite{paredes2021back} & 0.065 & 640 \\
\rowcolor[rgb]{ .95,  .95,  .95}
SPADE-E2VID~\cite{cadena2021spade} & 0.226 & 1096 \\
CUBE~\cite{zhao2024controllable} & 300.116 & 8215 \\
\rowcolor[rgb]{ .95,  .95,  .95}
HyperE2VID~\cite{ercan2024hypere2vid} & 0.060 & 1052 \\
\textbf{\methodname~(Ours)} & 344.34 & 46753 \\
\bottomrule
\end{tabular}
\vspace{0pt}
\label{table:computational overhead:sota:synthetic}
\end{table}

\begin{table}[t]
\small
\centering
\caption{Quantitative comparison of reconstruction performance and computational overhead of our method under different sampling steps. MSE is reported in units of $10^{-2}$.}
\setlength{\tabcolsep}{7.5pt}
\begin{tabular}{c|c|c|c|c}
\toprule
\multirow{2}{*}{\shortstack{Sampling\\Steps}} & \multirow{2}{*}{\shortstack{Computational\\Cost (TMACs)}} & \multicolumn{3}{c}{Synthetic} \\
\cline{3-5}
& & MSE~$\downarrow$ & SSIM~$\uparrow$ & LPIPS~$\downarrow$ \\
\hline
\rowcolor[rgb]{ .95,  .95,  .95}
1 & 37.14 & 1.97 & 0.6490 & 0.4380 \\
5 & 77.13 & 1.76 & 0.6760 & 0.4130 \\
\rowcolor[rgb]{ .95,  .95,  .95}
15 & 162.30 & 1.68 & 0.6980 & 0.4000 \\
30 & 344.34 & \textbf{1.67} & \textbf{0.7100} & \textbf{0.3940} \\
\bottomrule
\end{tabular}
\label{table:vfr:ablation:sampling_steps}
\vspace{-10pt}
\end{table}

\vspace{0.5em}
\noindent
\revision{
\textbf{Temporal Consistency.}
Here, we evaluated the optical-flow consistency of UniE2F, UniE2F w/o IFRG, and existing event-based frame reconstruction methods to provide a quantitative assessment of inter-frame temporal consistency.
\revisionx{Specifically, for each pair of consecutive frames, we estimate optical flow from the ground-truth frames and from the reconstructed frames using SEA-RAFT~\cite{wang2024sea} with the pretrained checkpoint \texttt{Tartan-C-T-TSKH-spring540x960-M.pth}. The same estimator and checkpoint were used for all compared methods. Flow estimation and the subsequent flow-consistency evaluation are performed at the spatial resolution of the corresponding ground-truth frames.} Let $\mathbf{V}_{f}$ and $\hat{\mathbf{V}}_{f}$ denote the $f$-th ground-truth and reconstructed frames, respectively. The corresponding optical flows are computed as
\begin{equation}
\mathbf{O}^{\mathrm{GT}}_{f}
=
\Phi(\mathbf{V}_{f},\mathbf{V}_{f+1}),
\qquad
\mathbf{O}^{\mathrm{Rec}}_{f}
=
\Phi(\hat{\mathbf{V}}_{f},\hat{\mathbf{V}}_{f+1}),
\end{equation}
\revisionx{where $\Phi(\cdot,\cdot)$ denotes the fixed pretrained SEA-RAFT estimator.}
}

\revision{
We then measure the discrepancy between the two optical-flow fields using the average endpoint error (EPE):
\begin{equation}
\begin{aligned}
\mathcal{O}_{\mathrm{flow}}
&=
\frac{1}{(F-1)HW}
\sum_{f=0}^{F-2}
\sum_{x=0}^{W-1}
\sum_{y=0}^{H-1} \\
&\quad \times
\left\|
\mathbf{O}^{\mathrm{Rec}}_{f}(x,y)
-
\mathbf{O}^{\mathrm{GT}}_{f}(x,y)
\right\|_{2},
\end{aligned}
\end{equation}
where a lower value indicates that the reconstructed video preserves motion between adjacent frames more consistently with the ground-truth video. 
}

\revision{
The comparison results on the real-world dataset are reported in~\tableref{table:4-4-1}. It can be seen that our method still achieves an optical-flow consistency error that is 35.79\% lower than that of the second-best method SSL-E2VID~\cite{paredes2021back}. More importantly, UniE2F with IFRG achieves better optical-flow consistency than the UniE2F w/o IFRG variant, demonstrating that the proposed IFRG effectively improves inter-frame temporal consistency rather than merely enhancing frame-wise reconstruction quality.
}

\begin{table}[t]
\centering
\small
\caption{
\revisionx{Quantitative comparison of optical-flow consistency on the real-world dataset. EPE measures the discrepancy between flow fields estimated from reconstructed and ground-truth frame pairs using the same pretrained SEA-RAFT model. The reference flow is estimated from ground-truth frames and is not independently measured ground-truth optical flow. Lower EPE indicates closer flow agreement. The best result is highlighted in \textbf{bold}.}}
\label{tab:optical_flow_consistency}
\setlength{\tabcolsep}{36pt}
\begin{tabular}{l|c}
\toprule
Method & EPE $\downarrow$ \\
\midrule
\rowcolor[rgb]{ .95, .95, .95}
E2VID~\cite{rebecq2019high}             & 2.5554  \\
FireNet~\cite{scheerlinck2020fast}      & 7.9732  \\
\rowcolor[rgb]{ .95, .95, .95}
E2VID+~\cite{stoffregen2020reducing}    & 8.2152  \\
FireNet+~\cite{stoffregen2020reducing}  & 8.1082  \\
\rowcolor[rgb]{ .95, .95, .95}
ETNet~\cite{weng2021event}              & 7.1900  \\
SSL-E2VID~\cite{paredes2021back}        & 2.5231  \\
\rowcolor[rgb]{ .95, .95, .95}
SPADE-E2VID~\cite{cadena2021spade}      & 3.2342 \\
CUBE~\cite{zhao2024controllable}        & 4.3118 \\
\rowcolor[rgb]{ .95, .95, .95}
HyperE2VID~\cite{ercan2024hypere2vid}   & 3.8875 \\
\midrule
\methodname~w/o IFRG                    & 2.1925 \\
\rowcolor[rgb]{ .95, .95, .95}
\textbf{\methodname~(Ours)}             & \textbf{1.6202} \\
\bottomrule
\end{tabular}
\label{table:4-4-1}
\vspace{-5pt}
\end{table}

\vspace{0.5em}
\noindent
\textbf{Computational Overhead.}
We evaluated the efficiency of our \methodname~and other methods on a single NVIDIA RTX A6000 GPU by reconstructing a sequence of 12 RGB frames at 448×320 resolution. The result is listed in~\tableref{table:computational overhead:sota:synthetic}. Although \methodname’s use of a pre-trained stable video diffusion model’s generative prior incurs higher computational cost and GPU memory usage than non-diffusion methods, it provides a high-quality baseline for event-to-frame reconstruction. Moreover, the Stable Diffusion–based CUBE requires both user‐provided text prompts and event data, and its reconstruction fails when only pure events are supplied. By contrast, our method still significantly outperforms CUBE in reconstruction fidelity and visual quality even without any textual prompts. 

It is important to note that computational efficiency is not the primary goal of our approach. 
\revision{Accordingly, UniE2F is primarily intended for offline, high-fidelity reconstruction scenarios where reconstruction quality and robustness are prioritized over real-time latency.}
Instead, since reconstructing video frames from events is a highly challenging information recovery task where significant visual details are missing, it inherently necessitates the integration of powerful generative priors.
The core objective of this work is to investigate how to effectively leverage existing large-scale pretrained video diffusion models and transfer their powerful generative priors to event-based vision tasks, rather than restricting the discussion to comparisons under a fixed small-model computational budget.
While utilizing such foundation models naturally raises concerns regarding computational complexity, we demonstrate that this can be alleviated by reducing the number of sampling steps to achieve a practical trade-off between efficiency and performance, as shown in \tableref{table:vfr:ablation:sampling_steps}.
By reducing the number of sampling steps to 15, 5 and 1, the computational cost of our method significantly drops from 344.34 to 162.30, 77.13 and 37.14 TMACs, while still maintaining superior reconstruction fidelity compared to all other state-of-the-art methods for event-based video frame reconstruction. 
In future work, building on the proposed \methodname, we will investigate diffusion-model distillation and network pruning to further speed up inference for real-time applications.

\section{Conclusion and Discussion}
\label{sec:conclusion}

We have presented a novel event-based video frame reconstruction approach by fine-tuning a pre-trained SVD model with event data as conditional inputs. 
Leveraging the powerful generative prior from the pre-trained video diffusion model, our method significantly enhances the fidelity and realism of reconstructed frames, especially in complex dynamic scenes. 
The introduction of event-based inter-frame residual guidance further improves the accuracy while maintaining diversity in the reconstruction.
Our unified framework is versatile, not only excelling in video frame reconstruction but also extending seamlessly to tasks such as interpolation and prediction. 
Experimental results on both synthetic and real-world datasets validate the effectiveness of our approach, demonstrating improved performance compared to existing methods across multiple evaluation metrics.

\vspace{0.5em}
\noindent
\textbf{Limitation and Future Work.}
\revisionx{UniE2F still has several limitations. Inaccurate residual estimates and inappropriate guidance may degrade reconstruction, while insufficient event observations can increase reliance on the generative prior. Moreover, the large diffusion backbone incurs considerable computational and memory overhead, favoring offline high-fidelity reconstruction over real-time applications. A more detailed discussion of these limitations and possible future directions is provided in Appendix \textcolor{blue}{G}.}

%\section*{Data Availability}

%All data used in our paper are publicly available. The details of the datasets and their acquisition are provided in Section \ref{sec:experiment:implementation details} and Appendix \ref{Sec:dataset_app}. Specifically, our synthetic data were generated using the DVS-Voltmeter~\cite{lin2022dvs} simulator based on the publicly accessible TrackingNet~\cite{muller2018trackingnet} dataset.

%\section*{Conflict of Interest}

%The authors affirm that there are no commercial or associative relationships that could be perceived as a conflict of interest related to the submitted work.

% \section*{Acknowledgments}

% This research was supported by the National Natural
% Science Foundation of China (No.62272433, No.62402468,
% No.U25A20390), and the Fundamental Research Funds for the
% Central Universities.

%\clearpage

\bibliographystyle{IEEEtran}
\bibliography{IEEE-bibliography}

\clearpage

\vfill

\end{document}

\title{\methodname: A Unified Diffusion Framework for Event-to-Frame Reconstruction with Video Foundation Models}

\author{Gang~Xu, Zhiyu~Zhu, and~Junhui~Hou,~\IEEEmembership{Senior Member,~IEEE}% <-this % stops a space
\thanks{Gang Xu and Zhiyu Zhu contributed equally to this work. \textit{(Corresponding author: Junhui Hou.)}}
\thanks{Gang Xu is with the Department of Computer Science, City University of Hong Kong, Hong Kong, China, and also with the Guangdong Laboratory of Artificial Intelligence and Digital Economy (SZ), Shenzhen, China (e-mail: xugang@gml.ac.cn).}
\thanks{Zhiyu Zhu is with the Department of Computer Science, City University of Hong Kong, Hong Kong, China, and also with the Department of Computer Science, City University of Hong Kong (Dongguan), Dongguan, China (e-mail: zhiyu.zhu@cityu-dg.edu.cn).}
\thanks{Junhui Hou is with the Department of Computer Science, City University of Hong Kong, Hong Kong, China (e-mail: jh.hou@cityu.edu.hk).}
\thanks{This work was supported in part by the National Natural Science Foundation of China under Grant 62422118, and in part by the Hong Kong Research Grants Council under Grant 11218121 and Grant N\_CityU1114/25.}
}

\maketitle

\appendices

\section*{Appendix Contents} % 这一行只是目录上面的标题，可改名

\noindent
\hyperref[Sec:Proof]{\texttt{A} Proof of Proposition 1: Residual-Error Propagation Bound}

\noindent
\hyperref[Sec:dataset_app]{\texttt{B} Details of Datasets}

\noindent
\hyperref[sec:ifre]{\texttt{C} Training Strategy for the Inter-Frame Residual Estimator}

\noindent
\hyperref[Sec:unified_setup]{\texttt{D} Detailed Experiment Settings for Video Frame Interpolation and Prediction}

\noindent
\hyperref[Sec:additional_ablation_studies]{\texttt{E} Additional Ablation Studies}

\noindent
\hyperref[Sec:additional_implementation_details]{\texttt{F} Additional Implementation Details}

\noindent
\hyperref[Sec:limitation_and_future_work]{\texttt{G} Limitation and Future Work}

\noindent
\hyperref[Sec:visual_result]{\texttt{H} More Visual Results}

\section{Proof of Proposition 1: Residual-Error Propagation Bound}
\label{Sec:Proof}

\revision{
We provide the proof of Proposition 1, which characterizes how the inconsistency between the reconstructed inter-frame residuals and the event-derived residuals contributes to the reconstruction error of individual video frames.
}

\revision{
At an arbitrary reverse diffusion step $t$ where the proposed inter-frame residual guidance is applied, let
\begin{equation}
\mathbf{F}^{t}
=
\mathcal{D}(\mathbf{U}^{t})
=
\{\mathbf{F}^{t}_0,\mathbf{F}^{t}_1,\ldots,\mathbf{F}^{t}_{F-1}\}
\end{equation}
denote the frame sequence decoded from the estimated clean latent $\mathbf{U}^{t}$.
Let
\begin{equation}
\mathbf{V}
=
\{\mathbf{V}_0,\mathbf{V}_1,\ldots,\mathbf{V}_{F-1}\}
\end{equation}
denote the corresponding ground-truth video frame sequence.
}

\revision{
For two consecutive frames, we define
\begin{equation}
\Delta\mathbf{V}_k
=
\mathbf{V}_{k+1}-\mathbf{V}_k,
\qquad
\Delta\mathbf{F}^{t}_k
=
\mathbf{F}^{t}_{k+1}-\mathbf{F}^{t}_k,
\end{equation}
where $k\in\{0,\ldots,F-2\}$.
}

\revision{
The proposed inter-frame residual guidance constrains the reconstructed inter-frame changes using the following sum-reduced L1 loss:
\begin{equation}
\mathcal{L}_{\mathrm{residual}}(\mathbf{U}^{t})
=
\sum_{k=0}^{F-2}
\left\|
\Delta\mathbf{F}^{t}_k-\mathbf{R}_k
\right\|_1,
\label{eq:supp_residual_loss}
\end{equation}
where $\mathbf{R}_k$ denotes the inter-frame residual estimated from the corresponding event representation, and the L1 norm sums the absolute differences over the spatial and channel dimensions.
}

\revision{
In our implementation, Eq.~(\ref{eq:supp_residual_loss}) is evaluated over temporal windows. For a window size of $w$, consecutive windows are sampled with a stride of $w-1$, such that adjacent windows share only their boundary frame and each consecutive-frame pair is included once. The gradients computed from the individual window losses are accumulated before updating the latent, which preserves the sum-reduced form of the residual objective in Eq.~(\ref{eq:supp_residual_loss}).
}

\revision{
For notational simplicity, we omit the diffusion-step superscript $t$ in the following derivation and write
$\mathbf{F}_k$, $\Delta\mathbf{F}_k$, and
$\mathcal{L}_{\mathrm{residual}}$ in place of
$\mathbf{F}^{t}_k$, $\Delta\mathbf{F}^{t}_k$, and
$\mathcal{L}_{\mathrm{residual}}(\mathbf{U}^{t})$, respectively.
The following derivation therefore applies to the decoded frame sequence at any guidance step.
}

\revision{
We define the residual-estimation error as
\begin{equation}
\boldsymbol{\epsilon}_k
=
\mathbf{R}_k-\Delta\mathbf{V}_k.
\end{equation}
This error accounts for discrepancies caused by imperfect event observations, sensor noise, contrast-threshold variations, and the learned event-to-residual mapping.
We assume that its magnitude is bounded as
\begin{equation}
\left\|
\boldsymbol{\epsilon}_k
\right\|_1
=
\left\|
\mathbf{R}_k-\Delta\mathbf{V}_k
\right\|_1
\leq
\varepsilon_k.
\label{eq:supp_residual_error}
\end{equation}
}

\revision{
For the $k$-th reconstructed frame, its reconstruction error can be written exactly as
\begin{equation}
\begin{aligned}
\mathbf{F}_k-\mathbf{V}_k
&=
\left(\mathbf{F}_0-\mathbf{V}_0\right)
+
\sum_{i=0}^{k-1}
\left[
\left(\mathbf{F}_{i+1}-\mathbf{F}_i\right)
-
\left(\mathbf{V}_{i+1}-\mathbf{V}_i\right)
\right]
\\
&=
\left(\mathbf{F}_0-\mathbf{V}_0\right)
+
\sum_{i=0}^{k-1}
\left(
\Delta\mathbf{F}_i-\Delta\mathbf{V}_i
\right).
\end{aligned}
\label{eq:supp_error_decomposition}
\end{equation}
}

\revision{
Applying the triangle inequality yields
\begin{equation}
\left\|
\mathbf{F}_k-\mathbf{V}_k
\right\|_1
\leq
\left\|
\mathbf{F}_0-\mathbf{V}_0
\right\|_1
+
\sum_{i=0}^{k-1}
\left\|
\Delta\mathbf{F}_i-\Delta\mathbf{V}_i
\right\|_1.
\label{eq:supp_error_1}
\end{equation}
}

\revision{
For each inter-frame residual term, we insert the event-derived residual $\mathbf{R}_i$:
\begin{equation}
\begin{aligned}
\Delta\mathbf{F}_i-\Delta\mathbf{V}_i
&=
\left(
\Delta\mathbf{F}_i-\mathbf{R}_i
\right)
+
\left(
\mathbf{R}_i-\Delta\mathbf{V}_i
\right).
\end{aligned}
\end{equation}
Therefore, by the triangle inequality,
\begin{equation}
\left\|
\Delta\mathbf{F}_i-\Delta\mathbf{V}_i
\right\|_1
\leq
\left\|
\Delta\mathbf{F}_i-\mathbf{R}_i
\right\|_1
+
\left\|
\mathbf{R}_i-\Delta\mathbf{V}_i
\right\|_1.
\label{eq:supp_error_2}
\end{equation}
}

\revision{
Using the residual-estimation error bound in Eq.~(\ref{eq:supp_residual_error}), we obtain
\begin{equation}
\left\|
\Delta\mathbf{F}_i-\Delta\mathbf{V}_i
\right\|_1
\leq
\left\|
\Delta\mathbf{F}_i-\mathbf{R}_i
\right\|_1
+
\varepsilon_i.
\label{eq:supp_error_3}
\end{equation}
}

\revision{
Substituting Eq.~(\ref{eq:supp_error_3}) into Eq.~(\ref{eq:supp_error_1}) gives
\begin{equation}
\boxed{
\left\|
\mathbf{F}_k-\mathbf{V}_k
\right\|_1
\leq
\left\|
\mathbf{F}_0-\mathbf{V}_0
\right\|_1
+
\sum_{i=0}^{k-1}
\left\|
\Delta\mathbf{F}_i-\mathbf{R}_i
\right\|_1
+
\sum_{i=0}^{k-1}
\varepsilon_i
}.
\label{eq:supp_frame_error_bound}
\end{equation}
}

\revision{
Since all terms in the residual loss are non-negative, for any
$k\in\{0,\ldots,F-1\}$, we have
\begin{equation}
\sum_{i=0}^{k-1}
\left\|
\Delta\mathbf{F}_i-\mathbf{R}_i
\right\|_1
\leq
\sum_{i=0}^{F-2}
\left\|
\Delta\mathbf{F}_i-\mathbf{R}_i
\right\|_1
=
\mathcal{L}_{\mathrm{residual}},
\label{eq:supp_partial_residual_bound}
\end{equation}
where the summation is understood to be zero when $k=0$.
Therefore, Eq.~(\ref{eq:supp_frame_error_bound}) further yields
\begin{equation}
\boxed{
\left\|
\mathbf{F}_k-\mathbf{V}_k
\right\|_1
\leq
\left\|
\mathbf{F}_0-\mathbf{V}_0
\right\|_1
+
\mathcal{L}_{\mathrm{residual}}
+
\sum_{i=0}^{k-1}
\varepsilon_i
}.
\label{eq:supp_global_error_bound}
\end{equation}
}

\revision{
If the residual-estimation error is uniformly bounded by
\begin{equation}
\varepsilon_i\leq\varepsilon,
\qquad
\forall i\in\{0,\ldots,F-2\},
\end{equation}
we further obtain
\begin{equation}
\boxed{
\left\|
\mathbf{F}_k-\mathbf{V}_k
\right\|_1
\leq
\left\|
\mathbf{F}_0-\mathbf{V}_0
\right\|_1
+
\mathcal{L}_{\mathrm{residual}}
+
k\varepsilon
}.
\label{eq:supp_uniform_error_bound}
\end{equation}
}

\revision{
The bound in Eq.~(\ref{eq:supp_frame_error_bound}) explicitly separates three sources of reconstruction error:
\begin{equation}
\underbrace{
\left\|
\mathbf{F}_0-\mathbf{V}_0
\right\|_1
}_{\text{initial-frame error}}
+
\underbrace{
\sum_{i=0}^{k-1}
\left\|
\Delta\mathbf{F}_i-\mathbf{R}_i
\right\|_1
}_{\text{event-guided residual inconsistency}}
+
\underbrace{
\sum_{i=0}^{k-1}
\varepsilon_i
}_{\text{residual-estimation error}}.
\end{equation}
}

\revision{
The first term is particularly important for event-only reconstruction.
Since event streams mainly describe relative intensity changes rather than absolute scene appearance, the initial frame cannot, in general, be uniquely recovered from event information alone.
Consequently, the proposed inter-frame residual guidance does not eliminate the ambiguity in the initial frame.
Instead, the pre-trained video diffusion model provides a generative prior for recovering plausible absolute appearance, while the event-derived residual guidance constrains the temporal changes between consecutive reconstructed frames.
}

\revision{
According to Eq.~(\ref{eq:supp_global_error_bound}), for fixed initial-frame error and residual-estimation error, reducing
$\mathcal{L}_{\mathrm{residual}}$
tightens the derived upper bound on the reconstruction error of the corresponding decoded frame sequence.
This analysis does not imply that each gradient update necessarily decreases the actual reconstruction error, nor does it guarantee a monotonic improvement in the final generation quality.
Instead, it shows that the event-guided residual inconsistency explicitly contributes to an upper bound on the frame reconstruction error, thereby providing an error-propagation motivation for the proposed inter-frame residual guidance.
}

\section{Details of Datasets}
\label{Sec:dataset_app}
The training of the network relies on a sufficiently large collection of event sequences paired with corresponding ground-truth video frames.
However, due to the limited availability of real-world datasets, we employ the event simulator to synthesize a large-scale training dataset from existing video data.

\revisionx{Specifically, we select 1,800 video sequences from TrackingNet~\cite{muller2018trackingnet} with varying levels of relative motion magnitude.}
Each sequence comprises approximately 400 to 500 frames.
Subsequently, we employ the DVS-Voltmeter~\cite{lin2022dvs} to simulate event streams between consecutive frames, thereby creating synthetic event data.
Unlike previous works that generate synthetic datasets from MS-COCO~\cite{lin2014microsoft}, which only contains globally homographic motion, our dataset is simulated from real-world videos and incorporates both globally homomorphic motion and locally independent motion.
This diversity can enable the network to generalize effectively to the various camera movements and scene variations encountered in real-world scenarios.

To ensure an accurate assessment of the algorithm’s performance in reconstructing video frames, we also construct a synthetic test set and a real-world test set. 
The synthetic test set is created by selecting 212 video frame sequences from TrackingNet~\cite{muller2018trackingnet}, with each sequence containing 12 frames. 
In contrast, the real-world test set, consisting of video sequences (each with 12 frames), is collected from the High-Speed Events and RGB (HS-ERGB)~\cite{tulyakov2021time} dataset, which contains synchronized event data and RGB frames recorded from a hybrid camera system.

% Notably, due to the architectural constraints of the diffusion model—where both the autoencoder and U-Net require the height and width of the inputs to be multiples of 8, and considering memory limitations, all event and video frames with $F=12$ are resized to a spatial resolution of $320\times448$ during training and inference phases of our method.

\section{Training Strategy for the Inter-Frame Residual Estimator}
\label{sec:ifre}

% First, we select event streams and their corresponding RGB video frames from the synthetic dataset’s training set, and convert each event stream into a three‑channel event representation; 
% % 
% then, this representation is fed into a ResNet backbone (identical to the standard ResNet except that its output layer is modified to predict pixel‑level residuals between consecutive frames). 

% During training, we supervise the network with the pixel‑wise difference between consecutive video frames, $\mathbf{V}_{t+1} - \mathbf{V}_t$, and optimize it by minimizing the $L_1$ loss between the predicted and ground‑truth residuals.
% % 
% The network is trained for 9,000 iterations with the Adam optimizer ($\beta_1=0.9,\ \beta_2=0.999$) at a fixed learning rate of $1\times10^{-4}$. 
% % 
% This setup ensures the estimator effectively learns the mapping from event representations to inter-frame residuals, providing accurate priors for the subsequent residual guidance.

\noindent
\revisionx{
\textbf{Architecture.}
The inter-frame residual estimator is a custom residual convolutional network with 12.928M parameters that preserves the input spatial resolution. The input is a three-channel event representation. An initial $3\times3$ convolution maps the input from 3 to 64 channels, followed by ReLU. This is followed by one residual block with main-path channel dimensions $64\rightarrow256\rightarrow256$ and ten residual blocks with main-path channel dimensions $256\rightarrow256\rightarrow256$. Each residual block contains two $3\times3$ convolutions, with ReLU applied after the first convolution and after the skip addition. The first residual block uses a $1\times1$ shortcut projection from 64 to 256 channels, while the remaining ten blocks use identity shortcuts.
}

\revisionx{
The output head consists of three $3\times3$ convolutions with channel dimensions $256\rightarrow128\rightarrow64\rightarrow3$, with ReLU after the first two convolutions and no activation after the final convolution. All convolutions use stride 1 and include bias terms; the $3\times3$ convolutions use padding 1. The network contains 11 residual blocks in total and uses neither batch normalization nor downsampling. Its output is a three-channel pixel-wise RGB residual.
}

\noindent
\revisionx{
\textbf{Training.}
We train the estimator independently using event streams and their corresponding RGB frames from the synthetic training set. Each event stream between consecutive frames is converted into a three-channel event representation, and the corresponding ground-truth residual is computed as $\mathbf{V}_{k+1}-\mathbf{V}_k$. We minimize the mean $L_1$ loss between the predicted and ground-truth residuals. The estimator is trained for 9,000 iterations using the Adam optimizer ($\beta_1=0.9$, $\beta_2=0.999$) with a fixed learning rate of $1\times10^{-4}$.
}

% \clearpage
% \newpage
\section{Detailed Experiment Settings for Video Frame Interpolation and Prediction\noindent}
% \vspace{-100pt}
\label{Sec:unified_setup}

Given a sequence of twelve consecutive frames from the test set, the network takes the reference frames in the case of video frame interpolation (VFI) to reconstruct the intermediate frames, or the first frame in the case of video frame prediction (VFP) to generate the subsequent frames.

For the interpolation task, to ensure compatibility with different interpolation factors supported by various pretrained baselines, we introduce two temporal up‑sampling settings during evaluation: VFI-4$\times$ and VFI-11$\times$.
Specifically, in the VFI-4$\times$ setting, we use frames at indices {0, 4, 8} as references and interpolate six intermediate frames (three intermediate frames between each pair. i.e., frames 1–3 for {0, 4} and 5–7 for {4, 8}). 
In the VFI-11$\times$ setting, we interpolate ten frames between frames at indices {0, 11}.
For the prediction task, we use the first frame at index 0 to generate the subsequent eleven frames at indices from 1 to 11.
After obtaining the intermediate or subsequent frames from the network output, we compare them with the corresponding ground-truth frames from the test set to compute MSE, SSIM, and LPIPS.

\section{Additional Ablation Studies}
\label{Sec:additional_ablation_studies}

\noindent
\textbf{Weighting Coefficient Scheduling.}
In video frame interpolation tasks, we compared two weight‐scheduling strategies: 1) Nonlinear schedule, as defined by the formula in the main text; 2) Linear schedule, i.e., linearly decreasing (increasing) from 1.0 (0.0) to 0.0 (1.0). The quantitative results are presented in~\tableref{table:vfi:ablation:synthetic:weighting coefficient}. 
As shown in the table, the nonlinear schedule generally achieves lower MSE and LPIPS scores while maintaining competitive SSIM, indicating a more effective fusion of reference‐frame
information compared to the linear schedule. In particular, the ``Nonlinear 1.00 $\rightarrow$ 0.00'' configuration yields the best overall performance , consistently reducing distortion and perceptual error. 
Therefore, comprehensively considering MSE, SSIM, and LPIPS, we adopt ``Nonlinear 1.00 $\rightarrow$ 0.00'' as the default setting.

\vspace{0.5em}
\noindent
\textbf{Long-Sequence Event-Based Video Reconstruction.}
To further evaluate our method on long sequences, we followed the same data collection setup in main paper. The long-sequence synthetic test set was built by sampling 200 video sequences from TrackingNet~\cite{muller2018trackingnet}, each containing 23 frames.
For UniE2F, we processed the sequences using a sliding window approach. 
We first performed event-to-frame reconstruction on the first 12 frames of each sequence, obtaining a 12-frame reconstructed sequence. Then, starting from the 12th frame, we used the following events and applied our prediction mode to generate the subsequent 11 frames.
We compared UniE2F with recent event-based video frame reconstruction methods on this setting. As shown in~\tableref{table:long_sequence}, our method achieves better performance across the reported metrics, indicating that UniE2F maintains stronger reconstruction quality and temporal stability even when the sequence length is increased.

\vspace{0.5em}
\noindent
\textbf{Event-Driven Interpolation at Arbitrary Timestamps.}
Here, we provided the experiment to evaluate the interpolation capability of our UniE2F at arbitrary timestamps. Concretely, we took a sequence of 12 frames. Concretely, we take a sequence of 12 frames $\{ V_0, V_1, \dots, V_{11} \}$ from the dataset and the 12 corresponding event groups $\{ \mathcal{G}_{0}, \mathcal{G}_{1}, \dots, \mathcal{G}_{11} \}$. Then, for the interpolation task, we \textbf{focused only on the frames between the first and last frames}, i.e., the 10 event groups $\{\mathcal{G}_{1}, \mathcal{G}_{2}, \dots, \mathcal{G}_{10} \}$ associated with $\{V_1, V_2, \dots, V_{10} \}$. These 10 intermediate event groups were re-partitioned into 7 new event groups, each covering a new temporal sub-interval between $V_0$ and $V_{10}$. Together with the original event groups ($\mathcal{G}_{0}$ and $\mathcal{G}_{11}$) aligned with the first and last frames ($V_0$ and $V_{11}$), thus we can obtained a total of 9 event groups.

Under this setting, we applied the interpolation scheme conditioned on the new sequence of 9 event groups. As shown in~\figref{fig:vfi:interpolation_at_arbitrary_timestamp}, UniE2F successfully reconstructs the video frame at the \emph{end timestamp} of the intermediate event group, even though the temporal duration of this event group has been changed. This demonstrates that, by flexibly adjusting the temporal duration of the event group at inference time, our UniE2F can generate frames at arbitrary user-specified timestamps, rather than being restricted to a fixed, pre-defined set of timestamps.

\vspace{0.5em}
\noindent
\textbf{Reconstruction under Sparse Event Streams.}
To further analyze the limitations of event-based video reconstruction, we investigated the behavior of UniE2F and competing methods in scenarios with extremely sparse events. \figref{fig:vfr:ablation:sparse event} presents qualitative comparisons of reconstructed frames under such sparse-event conditions. As can be seen from the figure, when the event stream becomes highly sparse, all event-based frame reconstruction methods—including ours—can only reliably recover structures in the vicinity of pixels where events are triggered. Regions without event activity provide little or no informative signal, and therefore cannot be faithfully reconstructed by any of these approaches. This illustrates a fundamental limitation shared by event-only reconstruction methods: in the absence of sufficient event observations, the model is unable to infer fine-grained details in event-free areas.

\vspace{0.5em}
\noindent
\textbf{Video Frame Interpolation/Prediction without Event.}
To illustrate the contribution of event data, we conducted experiments to evaluate the performance of UniE2F with only image input in interpolation and prediction tasks, i.e., without using events as conditional guidance. The qualitative results on synthetic datasets are shown in \tableref{table:exp-vfi_and_vfp_without_event} and \figref{fig:exp-vfi_and_vfp_without_event}. As can be seen from the table and figure, UniE2F has learned during training to map event representations to RGB frames. When we remove the event input and, following SVD, condition the model only on RGB frames that lie in a different domain, the network effectively applies this learned event-to-RGB mapping to RGB-to-RGB conditioning instead. This mismatch causes strong artifacts and disturbed appearance patterns. This demonstrates the critical role that event data plays in providing precise constraints, enabling the model to reconstruct intermediate or subsequent frames by adjusting the score function, even without explicit training for interpolation and prediction tasks.

\vspace{0.5em}
\noindent
\revision{\textbf{Analysis of the Synthetic-to-Real Performance Gap.}}
\revision{
To further investigate the source of the performance gap between synthetic and real-world data in the zero-shot video frame interpolation and prediction tasks, we introduce three additional test settings beyond the original evaluation setting. As summarized in~\tableref{table:domain-gap-settings}, these settings are designed to separately examine the effects of event threshold variation, background activity noise, and the discrepancy between simulated and real event streams.
}

\begin{table*}[t]
\small
\centering
\caption{\revision{Summary of the test settings used to analyze the synthetic-to-real performance gap.}}
\setlength{\tabcolsep}{18pt}
\begin{tabular}{
>{\centering\arraybackslash}m{2.5cm}|
>{\centering\arraybackslash}m{1.5cm}|
>{\arraybackslash}m{9.5cm}
}
\toprule
Setting & Data Source & Modification \\
\hline
\rowcolor[rgb]{.95,.95,.95}
Threshold Variation
& Synthetic
& Scale the original event-generation threshold by factors of $2$ and $3$ to simulate variations in the contrast threshold of real event cameras. \\
\hline
Noise Injection
& Synthetic
& Inject background activity noise with average densities of $0.5$ and $1.0$ events/pixel per inter-frame interval. The number of injected events in each inter-frame interval is sampled from a Poisson distribution with expectation $dHW$, where \(d\in\{0.5,1.0\}\) denotes the noise density and \(H\) and \(W\) denote the spatial resolution. \\
\hline
\rowcolor[rgb]{.95,.95,.95}
Synthetic Events on Real RGB
& HS-ERGB
& Retain the original real-world RGB frames while replacing the corresponding real event streams with synthetic events generated using the same event simulator adopted for our synthetic training data. \\
\bottomrule
\end{tabular}
\label{table:domain-gap-settings}
\vspace{-10pt}
\end{table*}

\revision{
For the noise-injection setting, since the synthetic RGB frames do not contain physical timestamps, we assign a fixed temporal interval between consecutive frames solely to construct the event-simulation timeline.
}

\revision{
As shown in~\tableref{table:domain-gap-threshold} and~\tableref{table:domain-gap-noise}, both event threshold variations and event noise degrade the performance of \methodname, indicating that changes in the event distribution can affect the event-to-frame mapping learned from synthetic training data. However, the influence of event noise is relatively limited compared with that of threshold variations, particularly in terms of MSE and SSIM. For example, increasing the event threshold scaling factor from $1$ to $3$ increases the MSE from $0.0072$ to $0.0245$ for VFI-$11\times$ and from $0.0093$ to $0.0314$ for VFP, accompanied by substantial decreases in SSIM. In comparison, injecting background activity noise with a density of $1.0$ event/pixel per inter-frame interval only increases the corresponding MSE values to $0.0078$ and $0.0101$, respectively. These results suggest that variations in the event-generation threshold, which directly alter the number and spatiotemporal distribution of generated events, constitute an important source of the synthetic-to-real performance gap.
}

\begin{table}[t]
\small
\centering
\caption{\revision{Quantitative comparison of UniE2F on the synthetic dataset with different event threshold scaling factors for VFI-$11\times$ and VFP. MSE is reported in units of $10^{-2}$.}}
\setlength{\tabcolsep}{7pt}
\begin{tabular}{c|c|c|c|c}
\toprule
\multirow{2}{*}{\makecell{Event Threshold\\Scaling Factor}} & \multirow{2}{*}{Mode} & \multicolumn{3}{c}{Synthetic} \\
\cline{3-5}
& & MSE~$\downarrow$ & SSIM~$\uparrow$ & LPIPS~$\downarrow$ \\
\hline
\rowcolor[rgb]{ .95,  .95,  .95}
1 & VFI-$11\times$ & \textbf{0.72} & \textbf{0.7400} & \textbf{0.3200} \\
2 & VFI-$11\times$ & 1.61 & 0.5850 & 0.4030 \\
\rowcolor[rgb]{ .95,  .95,  .95}
3 & VFI-$11\times$ & 2.45 & 0.4810 & 0.4680 \\
\hline
1 & VFP & \textbf{0.93} & \textbf{0.7100} & \textbf{0.3470} \\
\rowcolor[rgb]{ .95,  .95,  .95}
2 & VFP & 2.01 & 0.5560 & 0.4260 \\
3 & VFP & 3.14 & 0.4470 & 0.4800 \\
\bottomrule
\end{tabular}
\label{table:domain-gap-threshold}
\vspace{-10pt}
\end{table}

\begin{table}[t]
\small
\centering
\caption{\revision{Quantitative comparison of UniE2F on the synthetic dataset under different event-noise densities for VFI-$11\times$ and VFP. Noise density is measured in events/pixel per inter-frame interval. MSE values are reported in units of $10^{-2}$.}}
\setlength{\tabcolsep}{7pt}
\begin{tabular}{c|c|c|c|c}
\toprule
\multirow{2}{*}{Noise Density} & \multirow{2}{*}{Mode} & \multicolumn{3}{c}{Synthetic} \\
\cline{3-5}
& & MSE~$\downarrow$ & SSIM~$\uparrow$ & LPIPS~$\downarrow$ \\
\hline
\rowcolor[rgb]{ .95,  .95,  .95}
0 & VFI-$11\times$ & \textbf{0.72} & \textbf{0.7400} & \textbf{0.3200} \\
0.5 & VFI-$11\times$ & 0.76 & 0.7220 & 0.3300 \\
\rowcolor[rgb]{ .95,  .95,  .95}
1.0 & VFI-$11\times$ & 0.78 & 0.7130 & 0.3390 \\ 
\hline
0 & VFP & \textbf{0.93} & \textbf{0.7100} & \textbf{0.3470} \\ 
\rowcolor[rgb]{ .95,  .95,  .95}
0.5 & VFP & 0.98 & 0.6900 & 0.3590 \\
1.0 & VFP & 1.01 & 0.6780 & 0.3690 \\
\bottomrule
\end{tabular}
\label{table:domain-gap-noise}
\end{table}

\revision{
More importantly, we further investigate whether the performance gap is primarily associated with the RGB scene content itself or with the discrepancy between simulated and real event streams. To this end, we retain the original RGB video frames from the real-world HS-ERGB~\cite{tulyakov2021time} test set but replace the corresponding real event streams with synthetic events generated using the same event simulator as that used for our synthetic training data.
}

\revision{
As shown in~\tableref{table:domain-gap-real-synthetic-events}, replacing the real events with synthetic events substantially improves the performance for both VFI-$11\times$ and VFP. Specifically, the SSIM increases from $0.6500$ to $0.7750$ for VFI-$11\times$ and from $0.5940$ to $0.7320$ for VFP. The corresponding MSE also decreases from $0.0058$ to $0.0022$ for VFI-$11\times$ and from $0.0100$ to $0.0033$ for VFP. Since the underlying RGB content remains unchanged in this comparison, the improvement can be mainly attributed to the difference between real and simulated event streams rather than to the visual content itself.
}

\begin{table}[t]
\small
\centering
\caption{\revision{Quantitative comparison of UniE2F on different datasets for VFI-$11\times$ and VFP.
MSE is reported in units of $10^{-2}$.
Here, ``$\dagger$'' denotes the real-world dataset with synthetically generated event streams.}
}
\setlength{\tabcolsep}{8pt}
\begin{tabular}{c|c|c|c|c}
\toprule
Dataset & Mode & MSE~$\downarrow$ & SSIM~$\uparrow$ & LPIPS~$\downarrow$ \\
\hline
\rowcolor[rgb]{ .95,  .95,  .95}
Synthetic & VFI-$11\times$ & 0.72 & 0.7400 & \textbf{0.3200} \\
Real-World & VFI-$11\times$ & 0.58 & 0.6500 & 0.4000 \\
\rowcolor[rgb]{ .95,  .95,  .95}
Real-World$^{\dagger}$ & VFI-$11\times$ & \textbf{0.22} & \textbf{0.7750} & 0.3290 \\
\hline
Synthetic & VFP & 0.93 & 0.7100 & 0.3470 \\
\rowcolor[rgb]{ .95,  .95,  .95}
Real-World & VFP & 1.00 & 0.5940 & 0.4110 \\
Real-World$^{\dagger}$ & VFP & \textbf{0.33} & \textbf{0.7320} & \textbf{0.3440} \\
\bottomrule
\end{tabular}
\label{table:domain-gap-real-synthetic-events}
\vspace{-10pt}
\end{table}

\revision{
Taken together, these results indicate that the synthetic-to-real performance gap mainly arises from differences in the event-generation mechanism. In particular, event-threshold variation is an important contributing factor, whereas background activity noise has a relatively smaller influence under the tested settings. Differences between simulated and real event-generation processes can cause the same underlying intensity changes to produce different event counts and spatiotemporal distributions, resulting in a distribution shift in the event representations provided to the model and consequently degrading cross-domain generalization.
}

\vspace{0.5em}
\noindent
\revision{\textbf{Luminance-Space Evaluation.}}
\revision{
Since events encode only relative intensity changes and do not provide absolute chromatic information, the RGB appearance generated by UniE2F is partly inferred from the generative prior of the pretrained SVD model. Consequently, the reconstructed colors can be perceptually plausible and consistent with natural-video statistics, but they are not guaranteed to exactly match the ground-truth scene colors. This ambiguity may affect RGB-space metrics such as MSE, SSIM, LPIPS, and FID, because a reconstruction with correct structures and motion but different plausible colors can still receive a relatively large penalty. 
}

\revisionx{
Thus, to evaluate UniE2F and the competing methods in a common luminance space, we convert both the reconstructed RGB frames and the RGB ground-truth frames using $Y = 0.299R + 0.587G + 0.114B$. The resulting values are rounded and clipped to $[0,255]$, while the grayscale outputs of the competing methods are directly used as luminance images. We compute MSE, SSIM, and LPIPS on these luminance images, with LPIPS preprocessing specified in Appendix \textcolor{blue}{F}. The quantitative results are reported in~\tableref{table:4-3-1}.
}

\revision{
As shown, \revisionx{UniE2F achieves the best MSE and SSIM in the luminance-space evaluation on the real-world dataset, while its LPIPS score of 0.5720 does not outperform E2VID’s 0.5630}. These results demonstrate that the advantage of UniE2F is not solely due to the RGB generative prior; instead, it also achieves better reconstruction fidelity in terms of scene structure and intensity variation when color information is excluded from the evaluation.
}

\begin{table*}[t]
\small
\centering
\caption{\revision{Quantitative comparison with state-of-the-art event-based video frame reconstruction methods in the luminance space on real-world and synthetic datasets. MSE is reported in units of $10^{-2}$. The best result in each column is highlighted in \textbf{bold}.}}
\setlength{\tabcolsep}{19pt}
\begin{tabular}{l|c|c|c|c|c|c}
\toprule
\multirow{2}{*}{Method} 
& \multicolumn{3}{c|}{Real-World} 
& \multicolumn{3}{c}{Synthetic} \\
\cline{2-7}
& MSE~$\downarrow$ 
& SSIM~$\uparrow$ 
& LPIPS~$\downarrow$ 
& MSE~$\downarrow$ 
& SSIM~$\uparrow$ 
& LPIPS~$\downarrow$ \\
\hline
\rowcolor[rgb]{ .95, .95, .95}
E2VID~\cite{rebecq2019high} & 12.51 & 0.4400 & \textbf{0.5630} & 6.52 & 0.5220 & 0.4280 \\
FireNet~\cite{scheerlinck2020fast} & 11.84 & 0.4340 & 0.5870 & 5.69 & 0.5770 & 0.4040 \\
\rowcolor[rgb]{ .95,  .95,  .95}
E2VID+~\cite{stoffregen2020reducing} & 6.32 & 0.4620 & 0.5720 & 5.22 & 0.5320 & 0.4240 \\
FireNet+~\cite{stoffregen2020reducing} & 7.17 & 0.4090 & 0.6200 & 5.45 & 0.4680 & 0.4780 \\
\rowcolor[rgb]{ .95,  .95,  .95}
ETNet~\cite{weng2021event} & 8.30 & 0.4460 & 0.5990 & 4.88 & 0.5300 & 0.4330 \\
SSL-E2VID~\cite{paredes2021back} & 6.73 & 0.5210 & 0.6050 & 9.88 & 0.3790 & 0.5260 \\
\rowcolor[rgb]{ .95,  .95,  .95}
SPADE-E2VID~\cite{cadena2021spade} & 9.75 & 0.4700 & 0.5870 & 6.94 & 0.4510 & 0.4890 \\
CUBE~\cite{zhao2024controllable} & 8.17 & 0.3770 & 0.6200  & 12.60 & 0.2080 & 0.6870 \\
\rowcolor[rgb]{ .95,  .95,  .95}
HyperE2VID~\cite{ercan2024hypere2vid} & 7.13 & 0.4040 & 0.5890 & 6.11 & 0.4930 & 0.4480 \\
\textbf{\methodname~(Ours)} & \textbf{5.58} & \textbf{0.5280} & 0.5720 & \textbf{1.12} & \textbf{0.7490} & \textbf{0.2840} \\
\bottomrule
\end{tabular}
\label{table:4-3-1}
\vspace{-10pt}
\end{table*}

\begin{table}[t]
\small
\centering
\caption{Quantitative comparison of different weight-scheduling strategies for video frame interpolation (VFI), evaluated using MSE, SSIM, and LPIPS.
MSE is reported in units of $10^{-2}$.
}
\setlength{\tabcolsep}{7.5pt}
\begin{tabular}{c|c|c|c|c}
\toprule
\multirow{2}{*}{Mode} & \multirow{2}{*}{\shortstack{Weighting\\Coefficient}} & \multicolumn{3}{c}{Synthetic} \\
\cline{3-5}
& & MSE~$\downarrow$ & SSIM~$\uparrow$ & LPIPS~$\downarrow$ \\
\hline
\rowcolor[rgb]{ .95,  .95,  .95}
Linear & 0.50 $\rightarrow$ 0.50 & 0.80 & 0.7110 & 0.3460\\
Linear & 0.00 $\rightarrow$ 1.00 & 0.82 & 0.7030 & 0.3540\\
\rowcolor[rgb]{ .95,  .95,  .95}
Linear & 1.00 $\rightarrow$ 0.00 & 0.76 & \textbf{0.7450} & 0.3220\\
Nonlinear & 0.00 $\rightarrow$ 1.00 & 0.90 & 0.6960 & 0.3630\\
\rowcolor[rgb]{ .95,  .95,  .95}
Nonlinear & 1.00 $\rightarrow$ 0.00 & \textbf{0.72} & 0.7400 & \textbf{0.3200}\\
\bottomrule
\end{tabular}
\label{table:vfi:ablation:synthetic:weighting coefficient}
\vspace{-15pt}
\end{table}

\begin{table}[!htbp]
\small
\centering
\caption{
Quantitative results for long-sequence generation on synthetic dataset.
MSE is reported in units of $10^{-2}$.
}
\setlength{\tabcolsep}{11pt}
\begin{tabular}{l|c|c|c}
\toprule
Method & MSE~$\downarrow$ & SSIM~$\uparrow$ & LPIPS~$\downarrow$ \\
\hline
\rowcolor[rgb]{ .95,  .95,  .95}
SSL-E2VID~\cite{paredes2021back} & 9.06 & 0.4240 & 0.6310 \\
SPADE-E2VID~\cite{cadena2021spade} & 6.82 & 0.5200 & 0.5680 \\
\rowcolor[rgb]{ .95,  .95,  .95}
HyperE2VID~\cite{ercan2024hypere2vid} & 5.91 & 0.4960 & 0.5700 \\
\textbf{\methodname~(Ours)} & \textbf{2.71} & \textbf{0.6530} & \textbf{0.4470} \\
\bottomrule
\end{tabular}
\vspace{-5pt}
\label{table:long_sequence}
\end{table}

\begin{figure}[]
\centering
\includegraphics[width=1.0\linewidth]{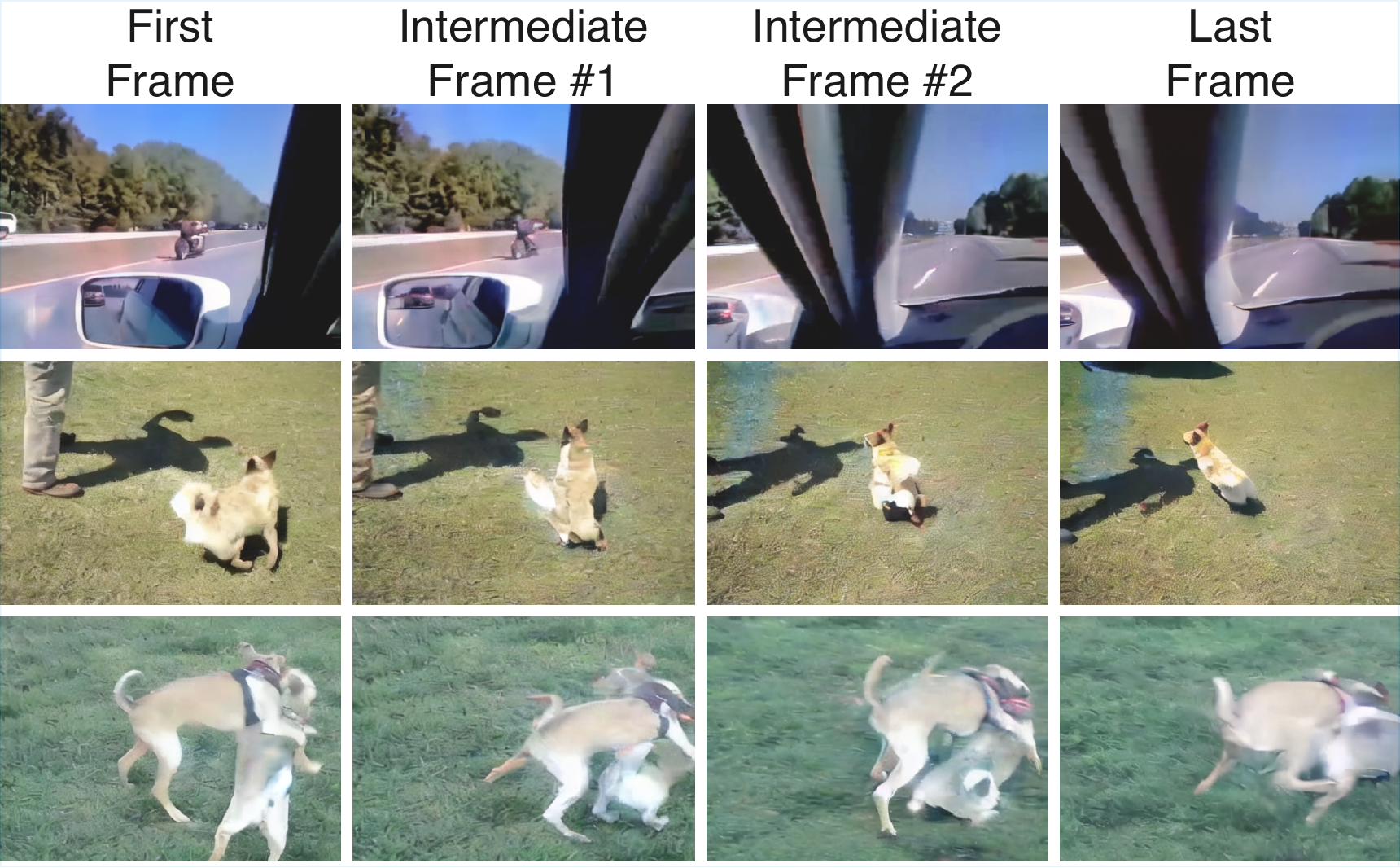}
\caption{Event-driven interpolation at arbitrary timestamps.}
\label{fig:vfi:interpolation_at_arbitrary_timestamp}
\end{figure}

\begin{figure*}[!htbp]
\centering
\includegraphics[width=1.0\linewidth]{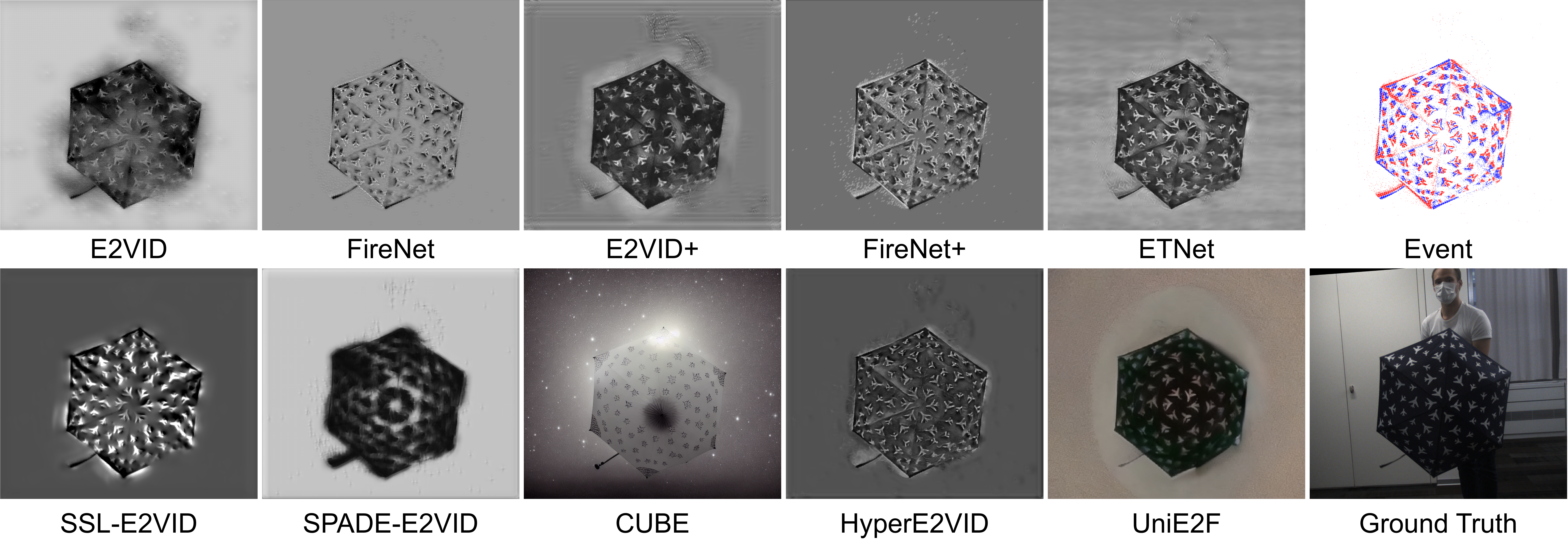}
\caption{Reconstruction under sparse event streams. All event-based methods recover structures near event pixels, but fail to reconstruct regions without events.}
\label{fig:vfr:ablation:sparse event}
\end{figure*}

\begin{table}[!htbp]
\small
\centering
\caption{Quantitative comparison of UniE2F with and without event guidance for prediction and interpolation.
MSE is reported in units of $10^{-2}$.
}
\setlength{\tabcolsep}{9pt}
\begin{tabular}{l|c|c|c}
\toprule
Task & MSE~$\downarrow$ & SSIM~$\uparrow$ & LPIPS~$\downarrow$ \\
\hline
\rowcolor[rgb]{ .95,  .95,  .95}
Prediction (w/ Event) & \textbf{0.93} & \textbf{0.7100} & \textbf{0.3470} \\
Prediction (w/o Event) & 15.63 & 0.107 & 0.658 \\
\hline
\rowcolor[rgb]{ .95,  .95,  .95}
Interpolation (w/ Event) & \textbf{0.72} & \textbf{0.7400} & \textbf{0.3200} \\
Interpolation (w/o Event) & 14.37 & 0.119 & 0.633 \\
\bottomrule
\end{tabular}
\vspace{-5pt}
\label{table:exp-vfi_and_vfp_without_event}
\end{table}

\begin{figure*}[htp]
\centering
\includegraphics[width=1.0\linewidth]{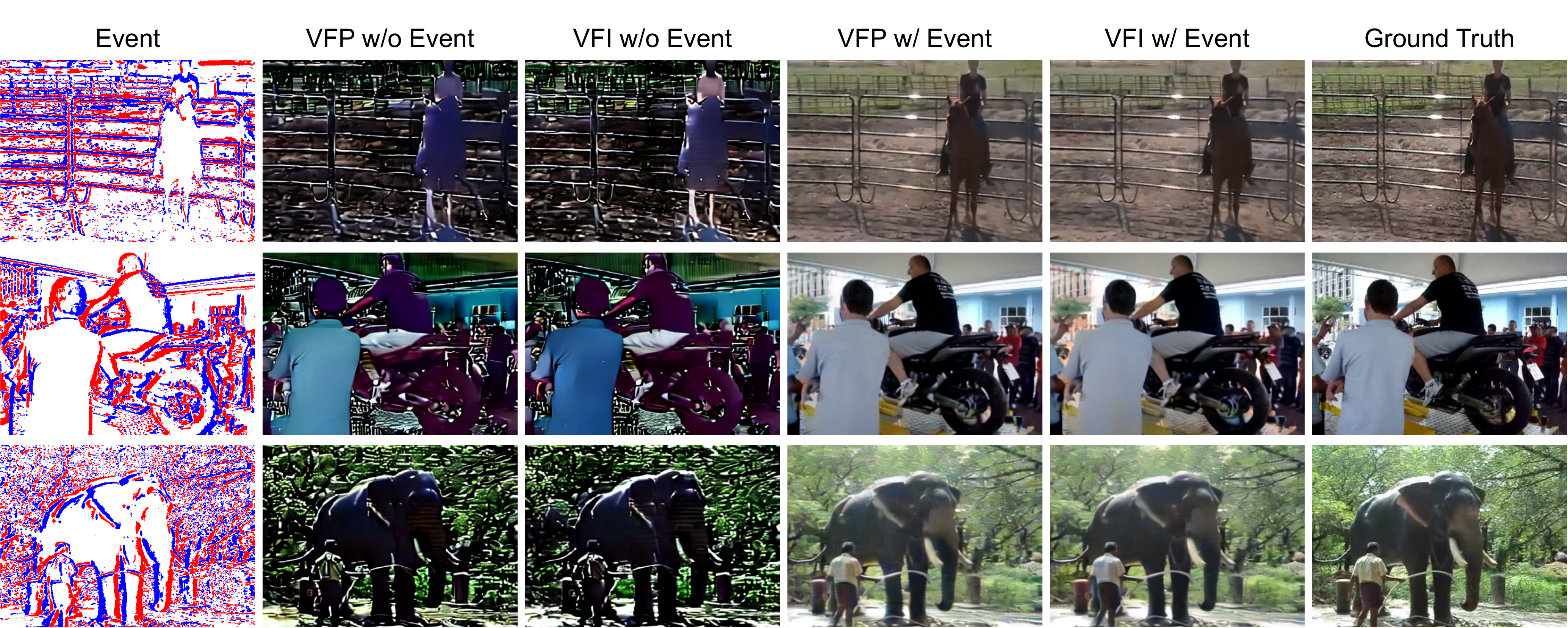}
\caption{Video frame interpolation/prediction without events.}
\label{fig:exp-vfi_and_vfp_without_event}
\vspace{-15pt}
\end{figure*}

\vspace{0.5em}
\noindent
\revision{
\textbf{Reference-Frame Guidance.}
To further verify the necessity of the proposed reference-frame guidance for enabling zero-shot VFI, we construct a variant with \revisionx{the direct reference constraint} for the VFI task. Different from our proposed method, this variant removes the diffusion-step-dependent weighting coefficient $\alpha(t)$ and directly applies the reference-frame deviations to the estimated clean latent.
Specifically, let $\mathbf{U}^{t}_{i}$ denote the estimated clean latent of the $i$-th frame at reverse diffusion step $t$. 
Following the definitions in Sec. \textcolor{blue}{IV-C} of main paper, the deviations between the encoded reference frames and their corresponding estimated clean latent are defined as
\begin{equation}
\begin{aligned}
\mathbf{D}^{t}_{0} &= \mathcal{E}_{rgb}(\mathbf{V}_{0})-\mathbf{U}^{t}_{0},\\
\mathbf{D}^{t}_{F-1} &= \mathcal{E}_{rgb}(\mathbf{V}_{F-1})-\mathbf{U}^{t}_{F-1},
\end{aligned}
\end{equation}
where $\mathcal{E}_{rgb}(\mathbf{V}_{0})$ and $\mathcal{E}_{rgb}(\mathbf{V}_{F-1})$ denote the clean latents encoded from the first and last reference frames, respectively.
The direct-constraint variant then directly sets the clean latent expectation of the $i$-th frame at each reverse diffusion step as
\begin{equation}
\tilde{\mathbf{U}}^{t}_{i}
=
\frac{
(\mathbf{D}^{t}_{0}+\mathbf{U}^{t}_{i})
+
(\mathbf{D}^{t}_{F-1}+\mathbf{U}^{t}_{i})
}{2}.
\end{equation}
That is, it directly adopts the mean of the latent estimates corrected by the deviations from the first and last reference frames, without weighted fusion with the original clean latent expectation $\mathbf{U}^{t}_{i}$. 
This variant can therefore be regarded as \revisionx{the direct reference-frame constraint}, allowing us to examine whether relying solely on reference-frame correction is sufficient to achieve stable and accurate zero-shot VFI.
}

\revision{
As shown in~\tableref{table:2-7-1}, compared with the proposed reference-frame guidance, this \revisionx{direct-constraint} variant exhibits noticeably degraded performance in terms of MSE, SSIM, and LPIPS. This is because, at the early stages of reverse diffusion, the original clean latent estimation $\mathbf{U}^{t}$ still contains considerable uncertainty. Directly modulating this clean latent expectation using the reference-frame correction at this stage may over-constrain the sampling trajectory, preventing the diffusion model from progressively recovering reasonable states for the entire frame sequence based on the event information and its generative prior. 
}

\begin{table}[t]
\small
\centering
\caption{\revisionx{Quantitative comparison of the proposed reference-frame guidance and the direct reference constraint for zero-shot video frame interpolation.
MSE is reported in units of $10^{-2}$.}
}
\setlength{\tabcolsep}{7pt}
\begin{tabular}{c|c|c|c}
\toprule
\multirow{2}{*}{Reference Modulation} & \multicolumn{3}{c}{Synthetic} \\
\cline{2-4}
& MSE~$\downarrow$ & SSIM~$\uparrow$ & LPIPS~$\downarrow$ \\
\hline
\rowcolor[rgb]{ .95,  .95,  .95}
\revisionx{Direct Reference Constraint} & 0.80 & 0.7000 & 0.3550 \\
Proposed Modulation (Ours) & \textbf{0.72} & \textbf{0.7400} & \textbf{0.3200} \\
\bottomrule
\end{tabular}
\label{table:2-7-1}
\vspace{-10pt}
\end{table}

\vspace{0.5em}
\noindent
\textbf{Robustness to Event Noise.}
We followed the setting in the SVD model, wherein the inference stage Gaussian noise with standard deviation 0.02 was injected into the 3-channel event representation.
As a reference, we regarded this original noise injection scheme as the \textit{Baseline}.
To further evaluate the robustness of our network to event noise, we introduced a noise-level coefficient $\alpha$ to scale the noise strength relative to the standard deviation of the event representation $E$: 
$E^\mathrm{noisy} = E_{f,c,y,x} + \mathcal{N}\!\left(0,\ (\alpha \times \sigma_E)^2\right), \quad \sigma_E = \mathrm{std}(E)$.
By varying $\alpha$, we systematically controlled the injected noise magnitude. 
From~\tableref{table:vfr:ablation:synthetic:noise‐level coefficient}, we observe that even under strong noise ($\alpha=1.0$), the reconstruction quality only slightly degrades: MSE rises from $0.0167$ to $0.0193$, SSIM drops from $0.7100$ to $0.6690$, and LPIPS increases from $0.3940$ to $0.4280$. 
These marginal changes demonstrate that our method maintains strong robustness to event noise.
\begin{table}[t]
\small
\centering
\caption{Quantitative comparison of different noise‐level coefficients, evaluated using MSE, SSIM, and LPIPS. MSE is reported in units of $10^{-2}$. The results demonstrate that our method exhibits strong robustness to event noise.
}
% \vspace{-5pt}
\setlength{\tabcolsep}{15pt}
\begin{tabular}{c|c|c|c}
\toprule
\multirow{2}{*}{\shortstack{Noise‐level\\Coefficient}} & \multicolumn{3}{c}{Synthetic} \\
\cline{2-4}
& MSE~$\downarrow$ & SSIM~$\uparrow$ & LPIPS~$\downarrow$ \\
\hline
\rowcolor[rgb]{ .95,  .95,  .95}
1.0 & 1.93 & 0.6690 & 0.4280 \\
0.5 & 1.77 & 0.6930 & 0.4050 \\
\rowcolor[rgb]{ .95,  .95,  .95}
0.1 & 1.68 & 0.7080 & 0.3950 \\
Baseline & \textbf{1.67} & \textbf{0.7100} & \textbf{0.3940} \\
\bottomrule
\end{tabular}
\label{table:vfr:ablation:synthetic:noise‐level coefficient}
% \vspace{-5pt}
\end{table}

\vspace{0.5em}
\noindent
\textbf{Guidance Strength Scheduling.}
Here, we investigated three different guidance strength strategies—linearly increasing, linearly decreasing, and constant guidance strengths—and the results show that the linearly decreasing schedule achieves the best performance. Here, to evaluate the robustness of this schedule, we introduced a non-linearly exponential decreasing strategy, where the guidance strength decreases from 0.1 to 0 following an exponential curve. The comparison results are presented in the \tableref{table:vfr:ablation:synthetic:guidance strength scheduling}. As shown, the linearly decreasing schedule still outperforms the exponentially decreasing one.
\begin{table}[t]
\small
\centering
\caption{
Quantitative comparison of different guidance strength strategies, evaluated using MSE, SSIM, and LPIPS.
MSE is reported in units of $10^{-2}$.
}
% \vspace{-10pt}
\setlength{\tabcolsep}{12pt}
\begin{tabular}{c|c|c|c}
\toprule
\multirow{2}{*}{\shortstack{Guidance Strength\\ Strategy}} & \multicolumn{3}{c}{Synthetic} \\
\cline{2-4}
& MSE~$\downarrow$ & SSIM~$\uparrow$ & LPIPS~$\downarrow$ \\
\hline
\rowcolor[rgb]{ .95,  .95,  .95}
Constant & 2.34 & 0.6610 & 0.4190 \\
Exponential & 1.81 & 0.6960 & 0.3990 \\
\rowcolor[rgb]{ .95,  .95,  .95}
Linear & \textbf{1.67} & \textbf{0.7100} & \textbf{0.3940} \\
\bottomrule
\end{tabular}
\label{table:vfr:ablation:synthetic:guidance strength scheduling}
\vspace{-10pt}
\end{table}

\vspace{0.5em}
\noindent
\revision{
\textbf{Influence of Residual Prediction Network Performance.}
To further analyze the influence of inter-frame residual estimator performance on the final reconstruction quality of UniE2F, we construct three residual estimators with different model capacities, namely IFRG-Tiny, IFRG-Medium, and IFRG-Large. The corresponding model sizes and residual estimation errors are reported in~\tableref{table:2-7-2}. Note that IFRG-Large corresponds to the standard residual estimator adopted in our final UniE2F model. While keeping all other settings of UniE2F unchanged, we use the inter-frame residuals predicted by each estimator to guide the reverse diffusion process and compare the resulting video reconstruction performance.
}

\revision{
As shown in~\tableref{table:2-7-2}, the performance of IFRG is strongly affected by the prediction capability of the inter-frame residual estimator. Compared with the baseline without IFRG, IFRG-Tiny leads to degraded reconstruction performance in terms of MSE, SSIM, and LPIPS. This indicates that inaccurate residual predictions may provide misleading constraints during reverse diffusion and consequently impair the reconstruction quality. Increasing the estimator capacity to IFRG-Medium substantially alleviates this degradation, although its performance still remains slightly inferior to the baseline without IFRG. In contrast, IFRG-Large, which achieves the lowest residual estimation error, consistently improves all three reconstruction metrics, reducing the MSE from 0.0191 to 0.0167, increasing the SSIM from 0.6880 to 0.7100, and reducing the LPIPS from 0.4020 to 0.3940. 
}

\revision{
These results demonstrate that the effectiveness of IFRG critically depends on the accuracy of inter-frame residual prediction: insufficiently accurate residual estimates may introduce incorrect guidance, whereas a sufficiently capable estimator can provide reliable inter-frame constraints and effectively improve the reconstruction quality. Therefore, we adopt IFRG-Large as the inter-frame residual estimator in the final implementation of UniE2F.
}

\begin{table}[t]
\small
\centering
\caption{
\revision{Effect of inter-frame residual estimator capacity on video reconstruction performance.
The reconstruction MSE is reported in units of $10^{-2}$.}
}
\setlength{\tabcolsep}{0pt}
\begin{tabular}{c|c|c|c|c|c}
\toprule
\multirow{2}{*}{Residual Estimator} 
& \multirow{2}{*}{Params.} 
& \multirow{2}{*}{Residual MSE~$\downarrow$}
& \multicolumn{3}{c}{Synthetic} \\
\cline{4-6}
& & 
& MSE~$\downarrow$ 
& SSIM~$\uparrow$ 
& LPIPS~$\downarrow$ \\
\hline
\rowcolor[rgb]{.95,.95,.95}
w/o IFRG 
& -- 
& -- 
& 1.91 
& 0.6880 
& 0.4020 \\
IFRG-Tiny 
& 0.074M 
& 0.16320 
& 2.25 
& 0.5810 
& 0.4630 \\
\rowcolor[rgb]{.95,.95,.95}
IFRG-Medium 
& 3.276M 
& 0.07527 
& 1.93 
& 0.6440 
& 0.4350 \\
IFRG-Large 
& 12.928M 
& \textbf{0.02594} 
& \textbf{1.67} 
& \textbf{0.7100} 
& \textbf{0.3940} \\
\bottomrule
\end{tabular}
\label{table:2-7-2}
\vspace{-5pt}
\end{table}

\section{Additional Implementation Details}
\label{Sec:additional_implementation_details}

\noindent
\revision{
\textbf{Evaluation Protocol for Cross-Modality Comparison.} When computing the RGB-space evaluation metrics, for competing methods that produce grayscale outputs, the single grayscale channel is replicated across the R, G, and B channels to form three-channel images, whereas the native RGB outputs of UniE2F are directly used for evaluation.
}
\revisionx{LPIPS is computed using a VGG backbone. Input pixel values $x$ are converted to 32-bit floating-point values and mapped to $[-1,1]$ through $2\times(x/255)-1$. For luminance-space LPIPS, each single-channel luminance image is replicated across three channels before being fed into the same LPIPS network.}

\vspace{0.5em}
\noindent
\revision{
\textbf{Event Encoder Architecture.} We adopt the same encoder architecture as the VAE encoder used in SVD. Specifically, the VAE encoder consists of an initial 3$\times$3 convolution followed by four hierarchical 2D downsampling blocks, each containing two ResNet layers, and a U-Net-style middle block with self-attention; the resulting features are further processed by GroupNorm, SiLU activation, and a 3$\times$3 convolution to parameterize the mean and log-variance of the latent distribution. Given an event-representation sequence $\mathbf{E}\in\mathbb{R}^{F \times 3\times H\times W}$, the event encoder produces the corresponding event feature
$\mathcal{E}_{event}(\mathbf{E})\in\mathbb{R}^{F\times4\times \frac{H}{8} \times \frac{W}{8}}$.
}

\vspace{0.5em}
\noindent
\revision{
\textbf{Event Representation and Preprocessing.} All methods construct their model inputs from event streams within the same temporal intervals, but we do not enforce exactly the same event representation across different methods. This is because the event representation itself is an integral part of the network design and training protocol of each method. Directly replacing the officially adopted event representation may introduce a significant distribution shift and therefore cannot faithfully reflect the performance of these methods under their original settings. Accordingly, for all baseline methods, we follow the event preprocessing and representation schemes specified in their official implementations. Specifically, \methodname~adopts the three-channel event representation described in Section III and encodes it as the conditional information for the video diffusion model. RE-VDM~\cite{chen2025repurposing} adopts the multi-stack event representation proposed in its original work. CUBE~\cite{zhao2024controllable} first divides the raw event streams into $V$ event slices along the temporal dimension and accumulates the number of events at each pixel location within each slice, followed by normalization, to construct event-edge representations for conditioning the diffusion model. Except for RE-VDM~\cite{chen2025repurposing} and CUBE~\cite{zhao2024controllable}, the other compared methods convert the raw event stream into voxel-grid representations as model inputs according to their official implementations.
}

\vspace{0.5em}
\noindent
\revision{
\textbf{Sequence Length.} For event-based video frame reconstruction, all methods are evaluated on the sequences with the same length of $12$ frames, ensuring that different methods process the same temporal intervals and the corresponding event streams. 
For VFI-$4\times$, VFI-$11\times$, and VFP, the corresponding reference frames and target frames are further selected from the same $12$-frame sequences according to the definition of each task, ensuring that all compared methods are evaluated over the same sequence length.
}

\vspace{0.5em}
\noindent
\revision{
\textbf{Input Resolution.} On the synthetic test set, all methods are evaluated using the same input resolution of $448\times320$. For the real-world data with higher spatial resolutions, since \methodname~is built upon a large-scale pretrained video diffusion model, its GPU memory consumption is substantially higher than that of conventional lightweight reconstruction networks. To avoid out-of-memory errors during inference, we first resize the input frames and their corresponding event data to $448\times320$ before feeding them into \methodname. After reconstruction, the outputs are resized back to the original spatial resolution, and the evaluation metrics are then computed against the corresponding ground-truth frames at the original spatial resolution. The other baseline methods are directly evaluated at the original spatial resolution.
}

\vspace{0.5em}
\noindent
\revision{
\textbf{Event-conditioning Mechanism.} In our method, we incorporate the encoded event information into SVD through channel-wise concatenation. Specifically, at diffusion step $t$, the event feature
$\mathcal{E}_{event}(\mathbf{E}) \in \mathbb{R}^{F \times 4 \times \frac{H}{8} \times \frac{W}{8}}$
is concatenated with the noisy video latent
$\mathbf{X}^{t} \in \mathbb{R}^{F \times 4 \times \frac{H}{8} \times \frac{W}{8}}$
along the channel dimension, resulting in an 8-channel latent input:
\begin{equation}
\operatorname{Concat}
\left[
\mathbf{X}^{t},
\mathcal{E}_{event}(\mathbf{E})
\right]
\in
\mathbb{R}^{F\times 8 \times \frac{H}{8} \times \frac{W}{8}}.
\end{equation}
This concatenated latent is directly fed into the denoising U-Net, allowing the encoded event information to condition the denoising process at each diffusion step.
}

\vspace{0.5em}
\noindent
\revision{
\textbf{Trainable Modules.} In the proposed method, all trainable parameters are restricted to three components: (i) the event encoder, (ii) the temporal transformer blocks in the SVD denoising U-Net, and (iii) the event-based inter-frame residual estimator. All other parameters of the pre-trained SVD model are kept frozen during training.
}

\vspace{0.5em}
\noindent
\revision{
\textbf{Video-Level Split of TrackingNet.} We selected 1,800 video sequences from TrackingNet to construct the training set. The selection was performed at the video level rather than by randomly sampling individual frames or clips, such that sequences used for training and testing do not overlap.
}

\revision{
To obtain a training subset covering a broad range of motion dynamics, we use the normalized average optical-flow magnitude as a measure of motion magnitude. Specifically, for each video, we first estimate the dense optical flow between every pair of adjacent frames using the optical-flow estimator. For an adjacent frame pair $(I_j,I_{j+1})$, given the estimated flow field
$\mathbf{f}_j(x,y)=(u_j(x,y),v_j(x,y))$, its average optical-flow magnitude is calculated as
\begin{equation}
m_j =
\frac{1}{HW}
\sum_{x=1}^{W}\sum_{y=1}^{H}
\sqrt{u_j(x,y)^2+v_j(x,y)^2}.
\end{equation}
We further normalize this value by the image diagonal length to make the motion measure independent of image resolution:
\begin{equation}
\hat{m}_j =
\frac{m_j}{\sqrt{H^2+W^2}}.
\end{equation}
Finally, the normalized average optical-flow magnitude of a video is defined as the median value over all adjacent frame pairs:
\begin{equation}
M_{\mathrm{video}}
=
\operatorname{median}_{j}\left(\hat{m}_j\right).
\end{equation}
\revisionx{Then, according to predefined thresholds ($t_1=0.002308$ and $t_2=0.005286$), the candidate videos are divided into three motion categories: slow, medium, and fast. We then select 600 videos from each motion category, resulting in 1,800 training videos in total.} 
We further construct the synthetic test set by selecting 212 unseen TrackingNet sequences from the remaining videos, while ensuring that the selected sequences span different levels of motion magnitude. None of these videos overlap with the training set.
}

\section{Limitation and Future Work}
\label{Sec:limitation_and_future_work}

\revisionx{UniE2F has several limitations related to residual prediction, guidance design, and its reliance on a generative prior, in addition to its computational overhead.}

\revisionx{First, the effectiveness of inter-frame residual guidance depends on the learned residual estimator. The guidance targets are predicted rather than directly measured inter-frame differences. The estimator-capacity ablation in Appendix \textcolor{blue}{E} shows that less accurate estimators can produce worse reconstructions than the variant without residual guidance, indicating that erroneous residual targets may introduce misleading constraints during reverse diffusion. Therefore, guidance should not be assumed to improve reconstruction when its residual targets are unreliable.}

\revisionx{Second, reconstruction performance is sensitive to the strength and schedule of inter-frame residual guidance. The guidance-strength and scheduling ablations in Sec.~VI of the main manuscript show that excessively strong residual guidance degrades reconstruction quality, while constant or linearly increasing guidance schedules perform worse than the adopted linearly decreasing schedule. These results suggest that excessive residual guidance may over-constrain reconstruction and hinder the refinement of fine details.}

\revisionx{Third, the model's performance can degrade when test events differ from the synthetic events used for training. The analysis in Appendix \textcolor{blue}{E} shows that changing event thresholds or adding background activity noise degrades performance in VFI and VFP. Replacing real events with simulated events while keeping the RGB frames unchanged also improves performance. These results suggest that differences between simulated and real events contribute to the performance gap, motivating further adaptation to real event data.}

\revisionx{Fourth, when events are extremely sparse or some regions generate few events over an extended period, UniE2F may lack sufficient information to recover the true scene appearance. The model may then fill in missing details using patterns learned by the pretrained video diffusion model. The generated colors, textures, or structures may look natural but differ from the actual scene.}

\revisionx{Finally, the large diffusion backbone incurs substantial computational and memory costs. The current framework is therefore more suitable for offline high-fidelity reconstruction than real-time or latency-critical applications. Potential future directions include uncertainty-aware residual estimation and adaptive guidance, adaptation to shifts in event-generation statistics, and model compression and sampling acceleration.}

\section{More Visual Results}
\label{Sec:visual_result}

From \figref{fig:x1} to \figref{fig:x8}, we provide additional visual comparisons to demonstrate that our method achieves superior performance.
% TBD TBD TBD TBD TBD TBD TBD TBD TBDTBD TBD TBDTBD TBD TBDTBD TBD TBDTBD TBD TBDTBD TBD TBDTBD TBD TBDTBD TBD TBDTBD TBD TBDTBD TBD TBDTBD TBD TBDTBD TBD TBDTBD TBD TBDTBD TBD TBDTBD TBD TBDTBD TBD TBDTBD TBD TBDTBD TBD TBDTBD TBD TBDTBD TBD TBDTBD TBD TBDTBD TBD TBDTBD TBD TBDTBD TBD TBDTBD TBD TBD

\newpage

\begin{figure*}[h!]
\centering
\includegraphics[width=\linewidth]{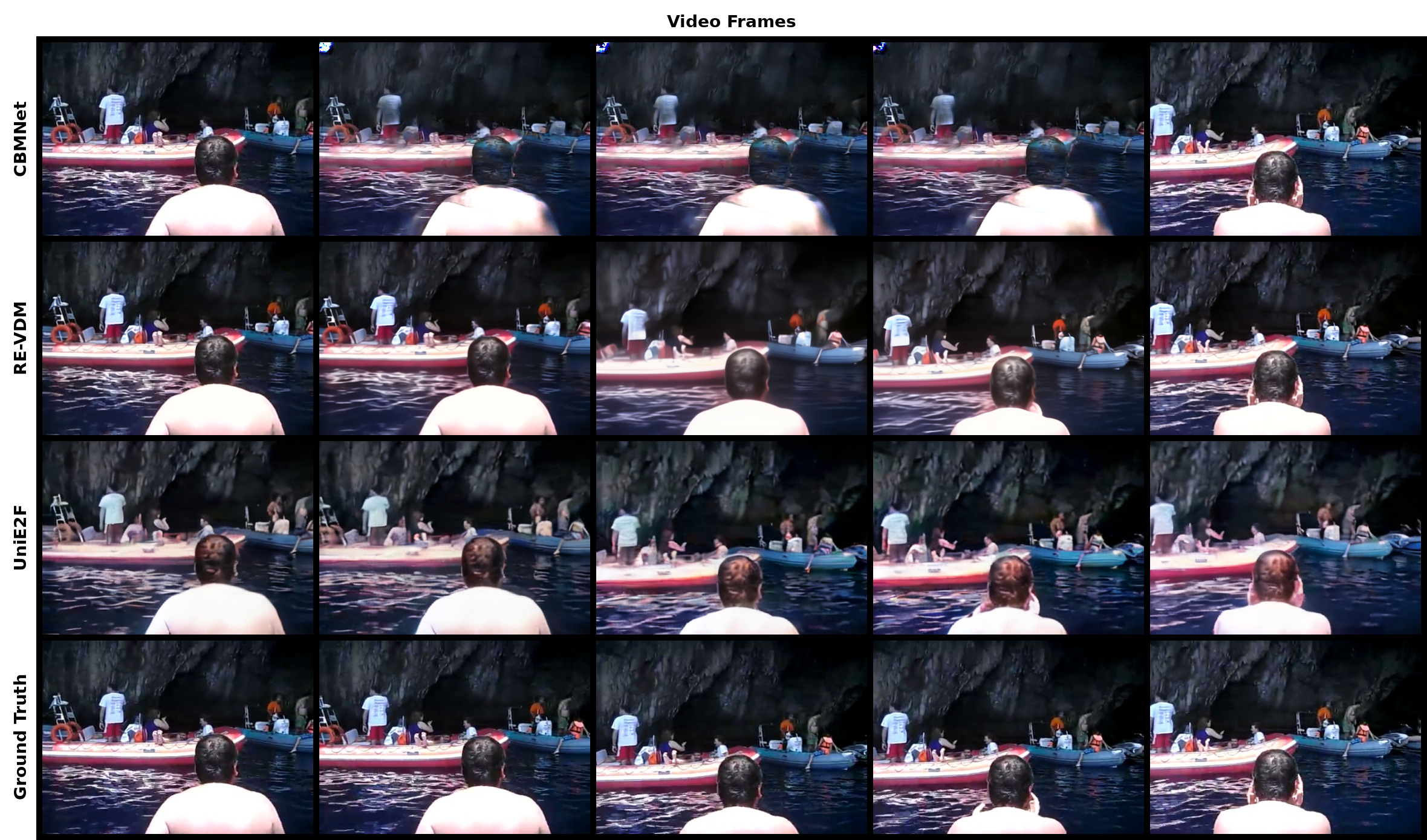}
% \vspace{-0.4cm}
\caption{Visual comparison of event-based video frame Interpolation results on synthetic dataset.}
\label{fig:x1}
% \vspace{-0.3cm}
\end{figure*}

\begin{figure*}[h!]
\centering
\includegraphics[width=\linewidth]{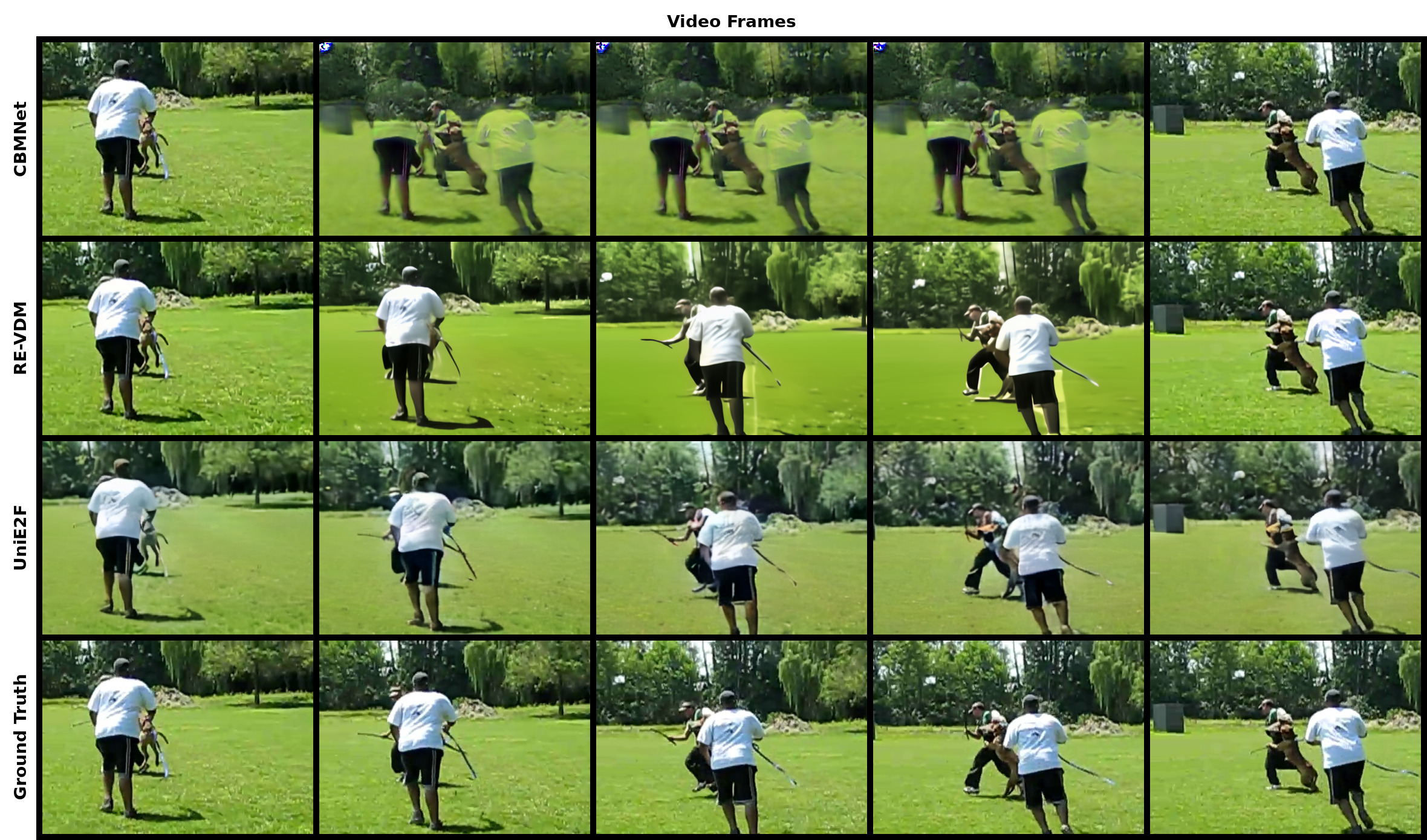}
% \vspace{-0.4cm}
\caption{Visual comparison of event-based video frame Interpolation results on synthetic dataset.}
% \vspace{-0.3cm}
\label{fig:x2}
\end{figure*}

\newpage
\begin{figure*}[h!]
\centering
\includegraphics[width=\linewidth]{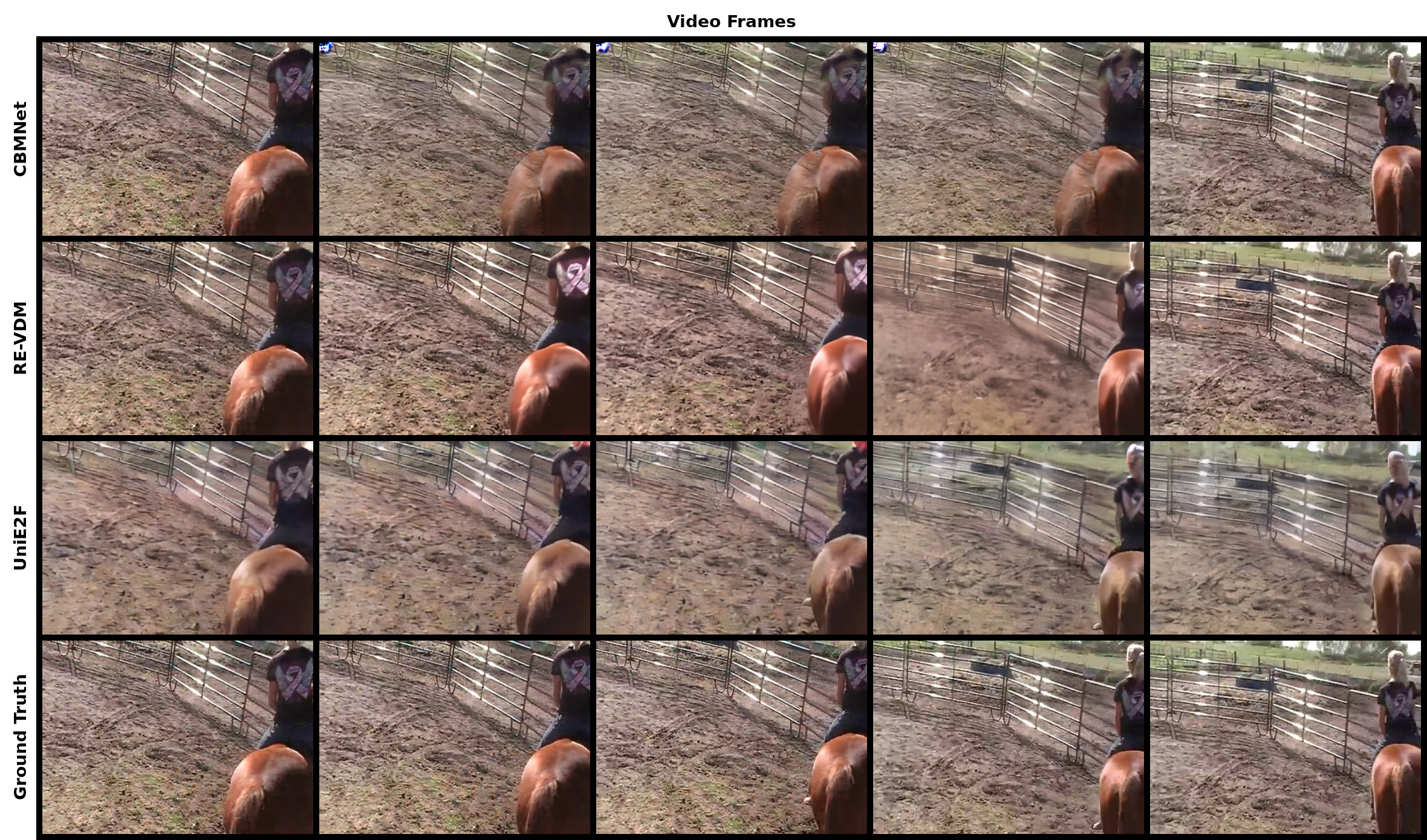}
% \vspace{-0.4cm}
\caption{Visual comparison of event-based video frame Interpolation results on synthetic dataset.}
% \vspace{-0.3cm}
\label{fig:x3}
\end{figure*}

\begin{figure*}[h!]
\centering
\includegraphics[width=\linewidth]{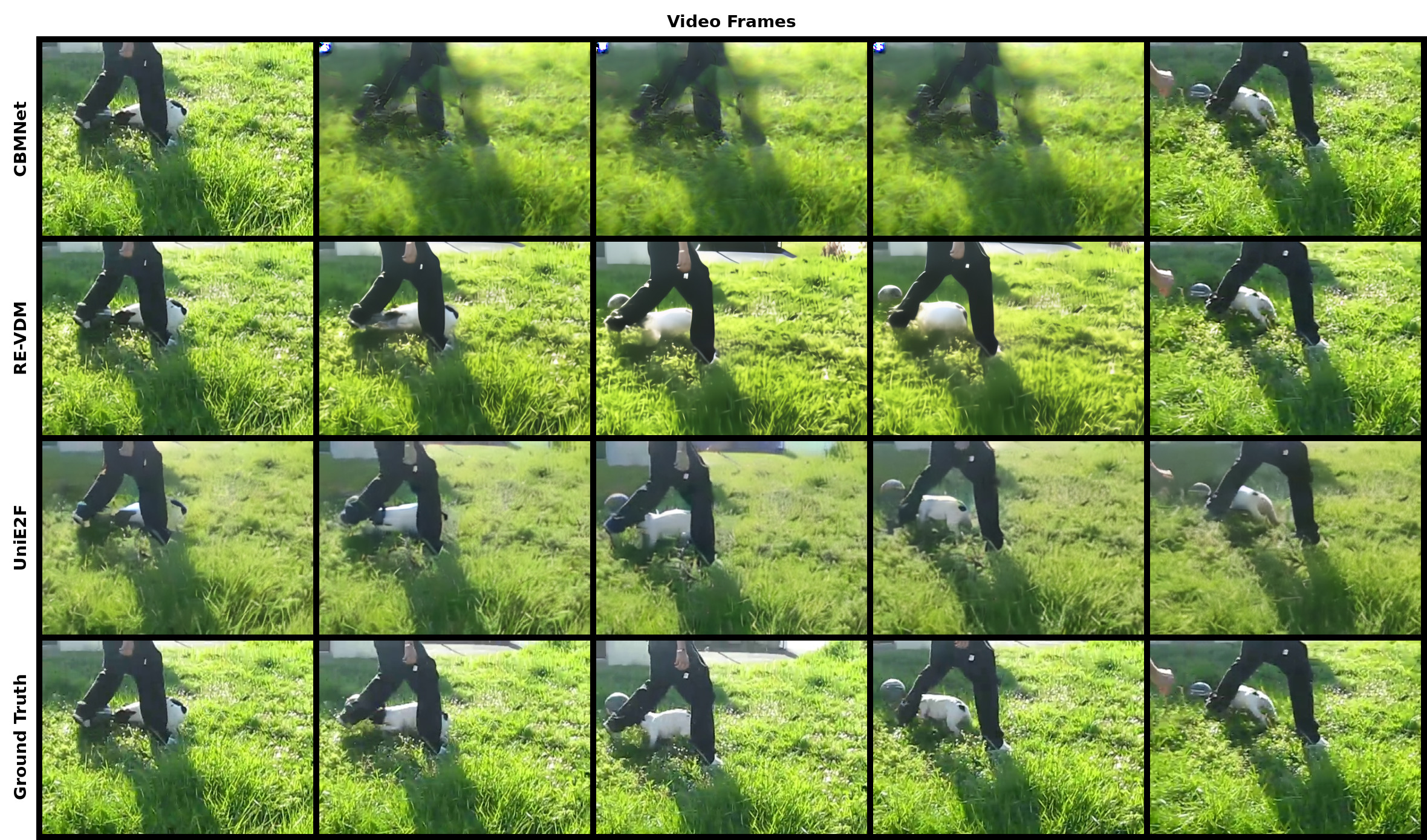}
% \vspace{-0.4cm}
\caption{Visual comparison of event-based video frame Interpolation results on synthetic dataset.}
% \vspace{-0.3cm}
\label{fig:x4}
\end{figure*}

\newpage
\begin{figure*}[h!]
\centering
\includegraphics[width=\linewidth]{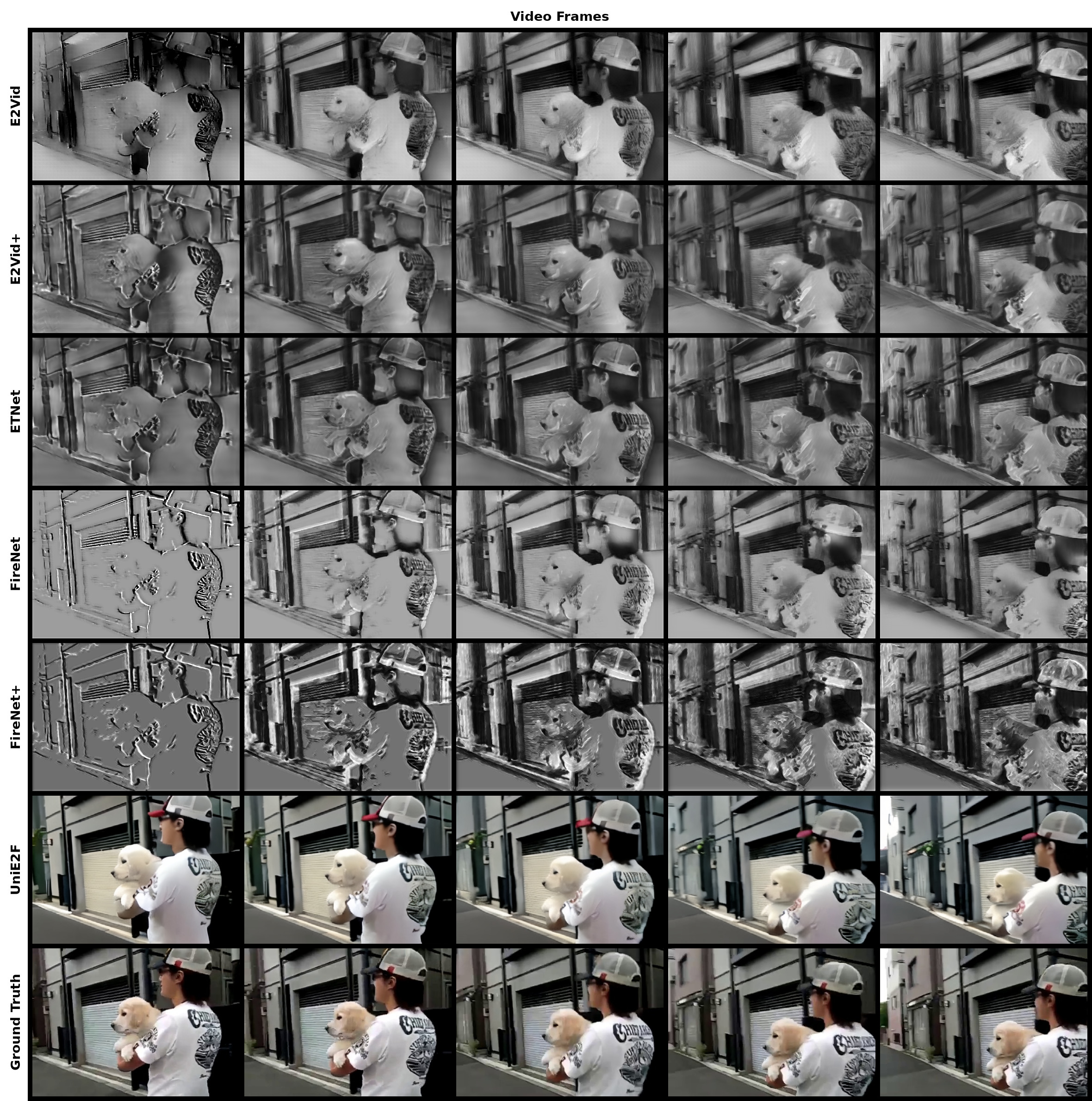}
% \vspace{-0.4cm}
\caption{Visual comparison of event-based video frame reconstruction results on synthetic dataset.}
% \vspace{-0.3cm}
\label{fig:x5}
\end{figure*}
\newpage
\begin{figure*}[h!]
\centering
\includegraphics[width=\linewidth]{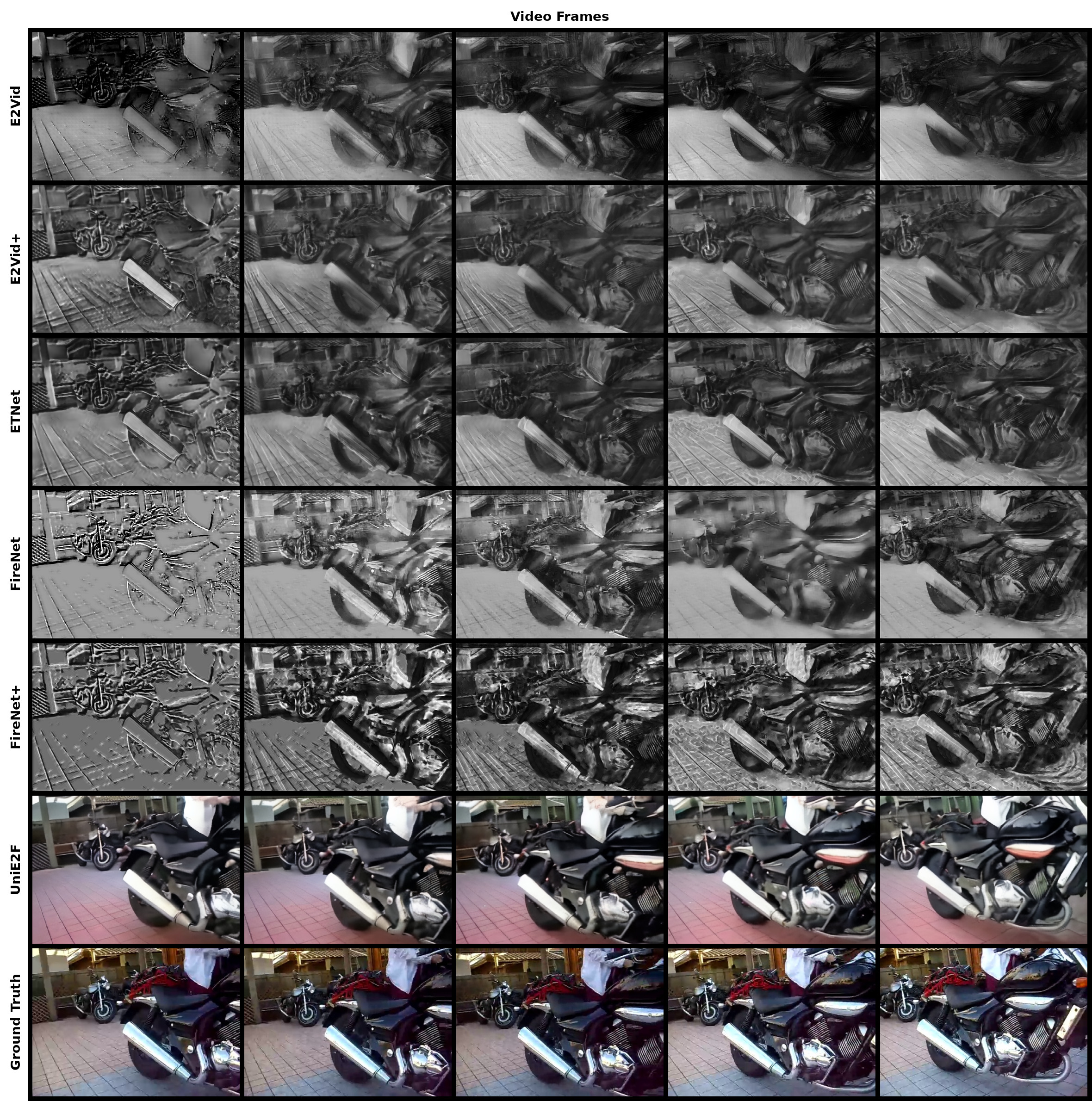}
% \vspace{-0.4cm}
\caption{Visual comparison of event-based video frame reconstruction results on synthetic dataset.}
% \vspace{-0.3cm}
\label{fig:x6}
\end{figure*}

\newpage

\begin{figure*}[h!]
\centering
\includegraphics[width=\linewidth]{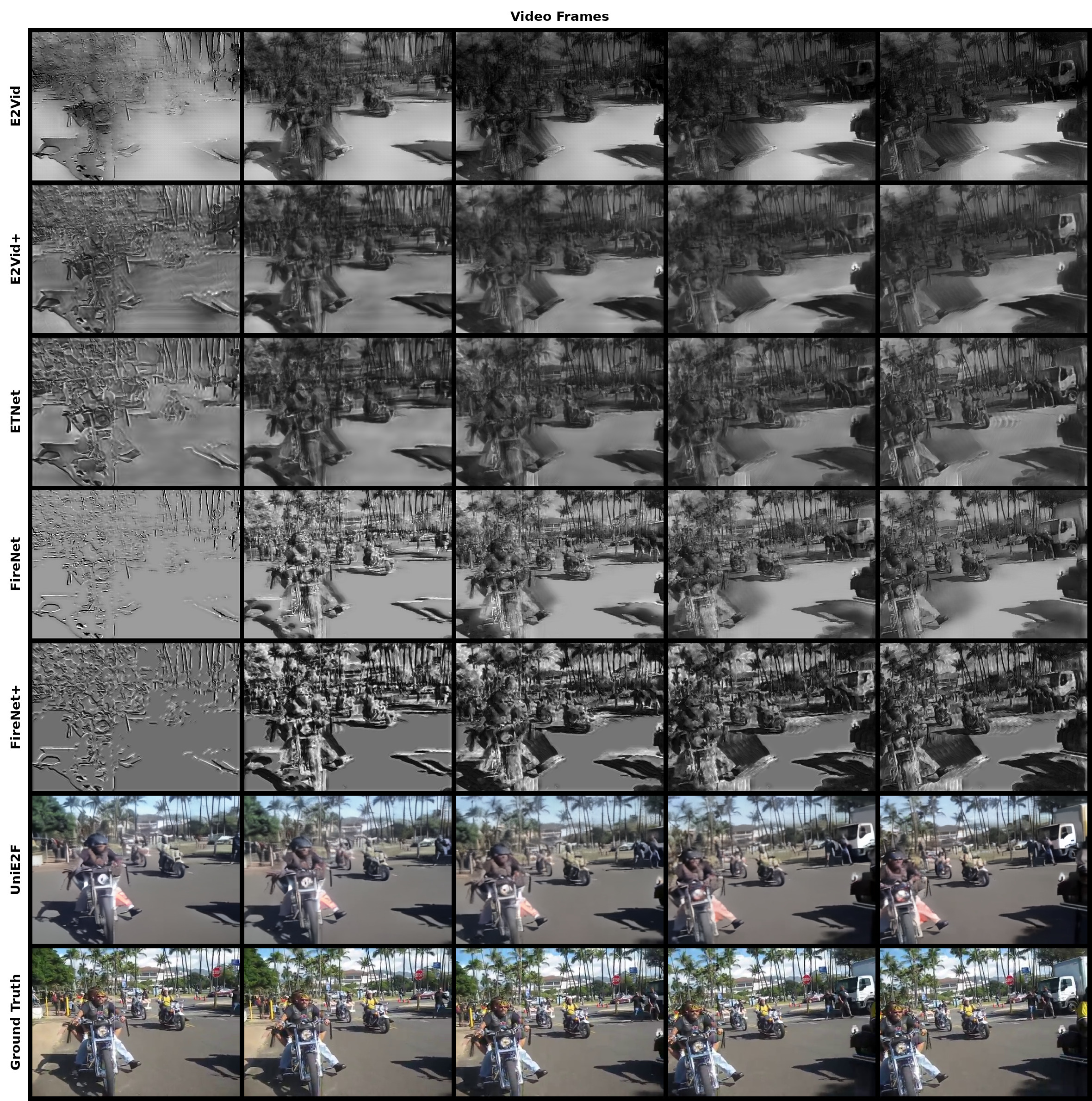}
% \vspace{-0.4cm}
\caption{Visual comparison of event-based video frame reconstruction results on synthetic dataset.}
% \vspace{-0.3cm}
\label{fig:x7}
\end{figure*}
\newpage
\begin{figure*}[h!]
\centering
\includegraphics[width=1.0\linewidth]{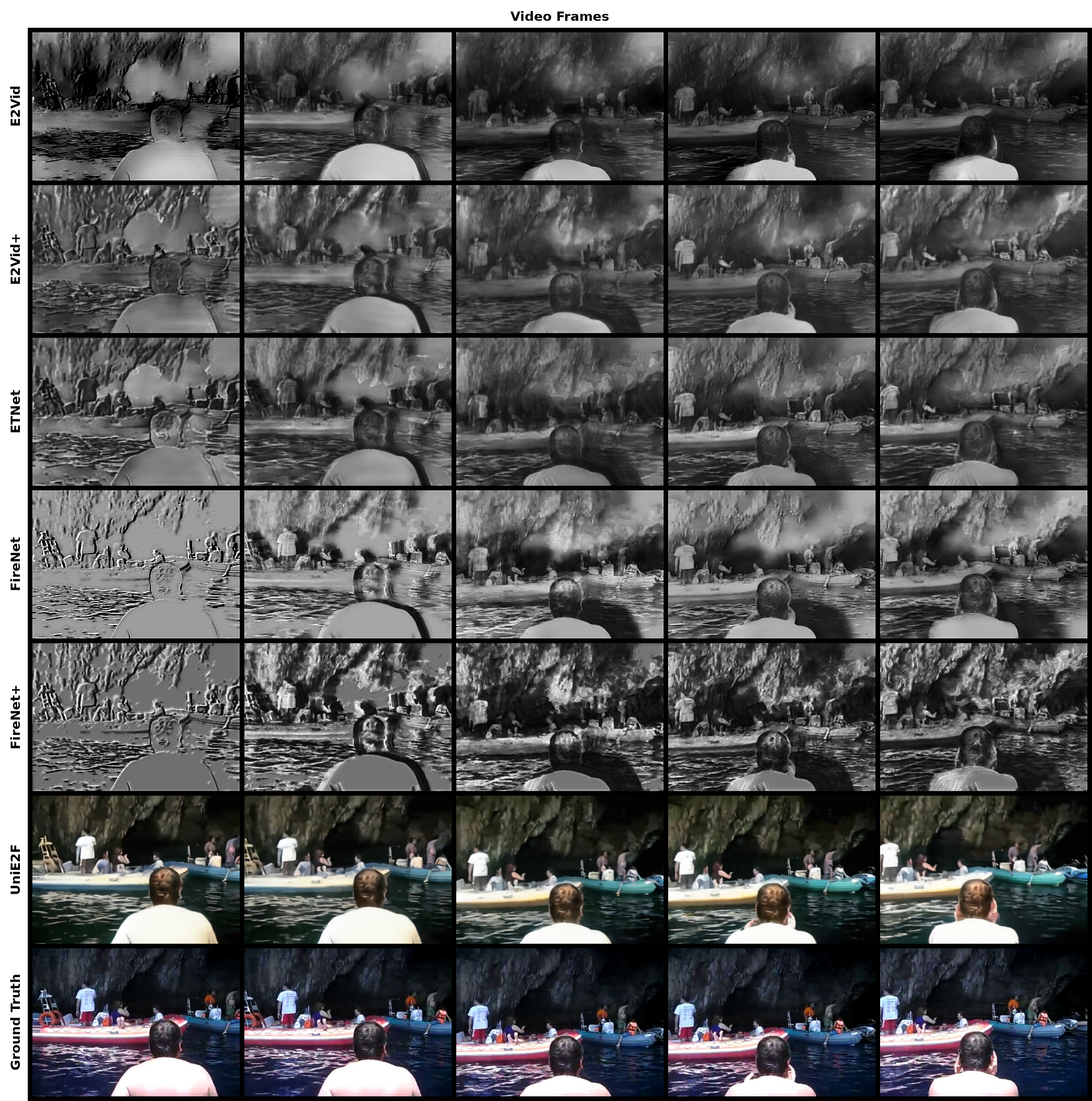}
% \vspace{-0.4cm}
\caption{Visual comparison of event-based video frame reconstruction results on synthetic dataset.}
% \vspace{-0.3cm}
\label{fig:x8}
\end{figure*}

\clearpage

\bibliographystyle{IEEEtran}
\bibliography{IEEE-bibliography}